\documentclass[11pt]{article}

\usepackage[final]{acl}

\usepackage{times}
\usepackage{latexsym}

\usepackage[T1]{fontenc}

\usepackage[utf8]{inputenc}

\usepackage{microtype}

\usepackage{xcolor}

\usepackage{inconsolata}

\usepackage{graphicx}
\usepackage{amsmath}
\usepackage{amssymb}
\usepackage{booktabs}
\usepackage{multirow}
\usepackage{enumitem}
\usepackage{amsthm}

\newtheorem{theorem}{Theorem}[section]
\newtheorem{lemma}[theorem]{Lemma}
\newtheorem{corollary}[theorem]{Corollary}
\newtheorem{assumption}{Assumption}[section]
\theoremstyle{remark}
\newtheorem{remark}{Remark}[section]

\title{What Does Layer-Importance Reveal About Transformers and State-Space Models?}

\author{
  Istabrak Abbes$^{1,2,3}$,
  Nizar Islah$^{2,3}$,
  Irina Rish$^{2,3,5}$,
  Sarath Chandar$^{1,2,4,5}$ \\[1ex]
  $^{1}$Chandar Research Lab  , $^{2}$Mila - Quebec AI Institute , $^{3}$Université de Montréal \\ $^{4}$Polytechnique Montréal, $^{5}$Canada CIFAR AI Chair \\[1ex]
  \texttt{istabrak.abbes@mila.quebec}
}

\newenvironment{revblock}{\begin{quote}\itshape}{\end{quote}}
\begin{document}
\maketitle
\begin{abstract}
Transformers and state-space models (SSMs) are the two dominant families of sequence models, and a central open question is how far the analytical knowledge built for transformers transfers to SSMs. We address this through the lens of layer importance which underpins compression, selective fine-tuning, and interpretability across both families. We decompose layer importance into two distinct notions. \emph{Necessity} captures how much the pretrained model depends on a layer's existing contribution, measured by the loss increase from bypassing it. \emph{Plasticity} captures where the model absorbs new information during fine-tuning, measured by the magnitude of task-specific weight updates. Our analysis reveals that the two families behave fundamentally differently: in every evaluated residual transformer up to $14$B parameters, Necessity and Plasticity anti-align across depth, whereas in the evaluated Mamba-style SSMs they point to overlapping regions. The sign of this alignment also predicts downstream adaptation behavior. In the evaluated transformers, concentrating updates in the most plastic layers increases catastrophic forgetting, while this tier-dependent effect disappears in the evaluated Mamba-style SSMs.
\end{abstract}
\section{Introduction}
\label{sec:introduction}
Modern language modeling is still dominated by residual transformers \citep{vaswani2017attention}, but selective state-space models have emerged as a major alternative architecture \citep{mamba,mamba2}. A large body of work on transformers has produced mature tools for pruning \citep{gromov2025unreasonableineffectivenessdeeperlayers,shortgpt}, parameter-efficient fine-tuning \citep{lora}, and interpretability \citep{rome,demystifyingroles}, many of which rely on estimating \emph{layer importance} \citep{zhang2024cornerstone,telltale,layercard,yao2024layerwise}. On the other hand, SSMs inherit the same vocabulary but not, obviously, the same internal structure. The natural question is whether the analytical tools calibrated on transformers transfer to SSMs, or whether they silently measure something else when selective recurrent state dynamics replace attention-based token mixing.
We study this question by decomposing layer importance into two distinct notions. \emph{Necessity} measures how much the pretrained model depends on a layer’s existing contribution, quantified through the loss increase caused by bypassing the layer. \emph{Plasticity} measures where the model absorbs new information during adaptation, quantified through the magnitude of task-specific weight updates during fine-tuning. Although these notions are often conflated under a single notion of ``importance,’’ they capture fundamentally different properties: pruning and interpretability primarily rely on Necessity \citep{shortgpt,zhang2024cornerstone}, while selective fine-tuning relies on Plasticity \citep{pan2024lisalayerwiseimportancesampling,adalora,alphalora,plop}. The key question is whether these two quantities identify the same regions of a network, and whether this relationship transfers across architectures.

\begin{figure}[t]
\centering
\begin{minipage}[c]{0.5\linewidth}
\centering
\includegraphics[width=\linewidth]{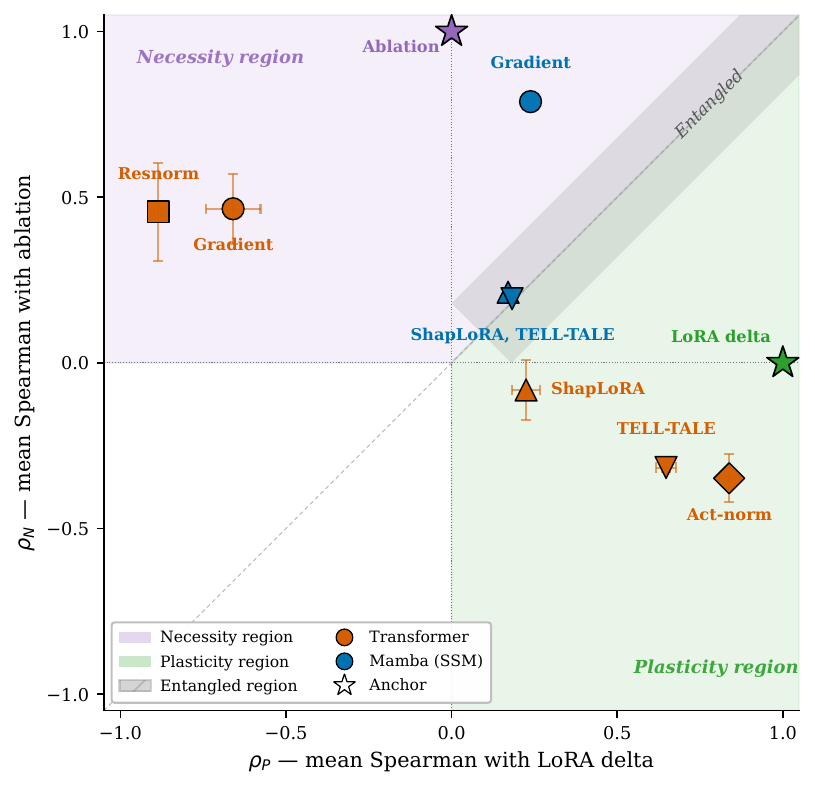}
\end{minipage}%
\hfill
\begin{minipage}[c]{0.49\linewidth}
\centering
\includegraphics[width=\linewidth]{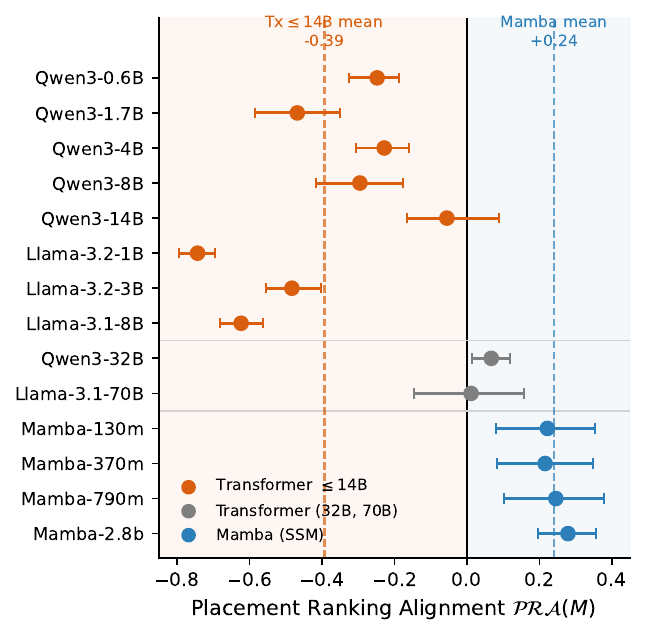}
\end{minipage}
\caption{\textbf{Layer importance is method-relative and architecture-conditional.}
\textbf{Left:} Estimators projected onto the Necessity--Plasticity plane (mean Spearman correlation with ablation, vertical; with LoRA-delta, horizontal). In transformers, gradient and resnorm methods align with Necessity; activation norm and TELL-TALE align with Plasticity. In Mamba, estimators cluster in the mixed-positive region. \textbf{Right:} $\mathcal{NPA}(M)$ across checkpoints (bars: bootstrap $95\%$ CI). Every evaluated transformer up to $14$B is negative; every evaluated Mamba-style SSM is positive; the two largest transformers collapse toward zero. Dashed lines mark family means.}   
\label{fig:overview}
\end{figure}

To make this distinction explicit, we introduce \emph{Necessity-Plasticity Alignment} ($\mathcal{NPA}$): the mean Spearman correlation, across tasks, between an ablation-based Necessity ranking and a LoRA-update-magnitude-based Plasticity ranking of the model's layers. Positive $\mathcal{NPA}$ means that the same layers are both necessary and plastic; negative $\mathcal{NPA}$ means that pretrained dependence and adaptation concentrate at opposite ends of the network; values near zero indicate that no stable depth ordering exists without specifying the estimator.

Across the evaluated transformer and Mamba-style checkpoints, $\mathcal{NPA}$ separates the two families cleanly. Every evaluated residual transformer up to $14$B parameters exhibits negative alignment: early layers are most necessary, while late layers are most plastic. In contrast, every evaluated Mamba-style SSM exhibits positive alignment, with Necessity and Plasticity concentrated in overlapping regions. This contrast persists across importance estimators, survives removal of the first layer, and extends beyond Mamba to RWKV and a hybrid architecture. It is also not an artifact of low-rank adaptation: an adapter-free control that recomputes Plasticity from full-parameter fine-tuning preserves the sign of $\mathcal{NPA}$ on every checkpoint we retrained (Appendix~\ref{app:fullft}). Larger transformers, including Qwen3-32B and Llama-3.1-70B, do not enter the SSM regime; instead, their alignment collapses toward zero. The same split predicts downstream adaptation behavior: transformers show tier-dependent forgetting when successive tasks adapt highly plastic layers, whereas this effect vanishes in Mamba-style SSMs \footnote{Code, per-layer importance scores, and analysis scripts: \url{https://github.com/chandar-lab/layer-importance-ssm-vs-transformers}.}

\paragraph{Contributions.}

\textbf{1)} We introduce \emph{Necessity-Plasticity Alignment} ($\mathcal{NPA}$), a diagnostic that measures the agreement between pretrained layer dependence and task-specific adaptation across depth.

\textbf{2)} We find a robust architecture-dependent split: every evaluated residual transformer up to $14$B exhibits negative $\mathcal{NPA}$, with early layers most necessary and late layers most plastic, whereas the evaluated Mamba-style SSMs exhibit positive $\mathcal{NPA}$, with Necessity and Plasticity concentrated in overlapping regions. The split survives an adapter-free full fine-tuning control and reappears \emph{within} a single hybrid checkpoint.

\textbf{3)} We show that $\mathcal{NPA}$ predicts selective fine-tuning behavior in a controlled two-task continual-learning probe: the evaluated transformers exhibit strong forgetting when both tasks adapt the same highly plastic layers, while this tier-dependent interference disappears in the evaluated Mamba-style SSMs.
\section{Related Work}
\label{sec:related_work}
\paragraph{Importance-guided PEFT and rank/placement allocation.}
A growing line of parameter-efficient fine-tuning (PEFT) work argues that adaptation should not be uniform across depth. AdaLoRA~\citep{adalora} reallocates rank by singular-value importance during training; AlphaLoRA~\citep{alphalora} assigns LoRA experts using heavy-tailed self-regularization statistics; Layer Card~\citep{layercard} introduces \emph{resnorm} as a reusable layer diagnostic for selective placement; ShapLoRA~\citep{shaplora} uses Shapley-style sensitivity scores; and PLoP~\citep{plop} proposes a precise placement signal driven by neural-feature norms. Each of these methods proposes a single importance estimator and studies it in isolation. Recent work has also begun to ask whether compression, PEFT, and interpretability tools developed for transformers transfer to selective state-space models. LoRA-style adaptation has been extended to Mamba-style architectures~\citep{mambapeft}, while layer pruning and relevance-propagation methods have been adapted to selective SSMs~\citep{mambashedder,mambalrp}. These works motivate cross-architecture comparisons, but they do not ask whether different task-sensitive importance estimators agree on the same checkpoints

\paragraph{Layer roles, redundancy, and depth-dependent task usage.}
A parallel literature studies which layers are useful, redundant, or task-specific in pretrained models. ShortGPT~\citep{shortgpt} shows that many transformer layers can be removed with limited degradation; TELL-TALE~\citep{telltale} performs task-aware layer elimination; \citet{demystifyingroles} argue that depth contributes differently to retrieval, knowledge, and reasoning; \citet{yao2024layerwise} study layer-wise importance for memory-efficient PEFT. \citet{zhang2024cornerstone} identify ``cornerstone'' layers via Shapley-style ablation, finding early-layer dominance, a phenomenon we recover under ablation but show dissolves under LoRA delta on the same checkpoint. \citet{nepal2025forged} report that the layers ablation identifies as important for mathematical reasoning are stable across pretraining and post-training. \citet{shen2024ila} learn binary masks for layer significance during alignment and observe high overlap across alignment datasets, all under a single estimator. These results document stability \emph{within} one estimator; we document the orthogonal fact that \emph{across} estimators, the very identity of ``important'' layers changes, and that whether it changes is architecture-conditional.

\paragraph{Method disagreement in attribution and importance.}
Disagreement between attribution methods is an established finding in the explainability literature. \citet{adebayo2018sanity} show that gradient-saliency maps survive parameter randomizations that destroy network behavior, evidence that gradient attribution does not measure the same thing as causal intervention. \citet{ancona2018towards} unify gradient-based methods as varying-fidelity linearizations of true ablation; \citet{integratedgradients} prove that no attribution method satisfies sensitivity, implementation invariance, and completeness simultaneously. \citet{krishna2022disagreement} measure the magnitude of feature-level attribution disagreement and characterize how practitioner interpretations shift with the method chosen. Our contribution shifts the disagreement from feature attributions to layer rankings, from a single architecture to a transformer/SSM contrast, and from ``expected at high resolution'' to a structurally predictable function of residual-stream gradient flow.

\section{Layer-Importance Diagnostic Protocol}

Our goal is to determine whether commonly used layer-importance estimators measure the same underlying property across architectures, or whether their behavior depends on the distinction between pretrained reliance and adaptation dynamics. We therefore study layer importance through two complementary quantities: \emph{Necessity} and \emph{Plasticity}.

\paragraph{Necessity.}
Necessity measures how much the pretrained model depends on a layer's existing contribution at inference time. A layer is considered necessary if bypassing it substantially degrades performance. For a model $M$ with layers $\{\ell_i\}_{i=1}^L$, we define the Necessity score of layer $\ell_i$ on task $t$ as:
\[
N(\ell_i; M,t)
=
\mathcal{L}\!\left(M_{\setminus \ell_i}, t\right)
-
\mathcal{L}(M,t),
\]
where $\mathcal{L}$ denotes the task loss and $M_{\setminus \ell_i}$ is the model with layer $\ell_i$ bypassed. Larger values indicate that the pretrained computation relies more strongly on that layer.

\paragraph{Plasticity.}
Plasticity measures where the model absorbs new information during adaptation. Rather than quantifying reliance on the pretrained computation, it measures how strongly each layer changes during task-specific fine-tuning. Using LoRA adaptation, we define the Plasticity score of layer $\ell_i$ as:
\[
P(\ell_i; M,t)
=
\left\|
\Delta W_i^{(t)}
\right\|_F,
\]
where $\Delta W_i^{(t)}$ is the learned LoRA update for layer $\ell_i$ on task $t$, and $\|\cdot\|_F$ denotes the Frobenius norm. Larger values indicate that adaptation concentrates more strongly in that layer.

Although both quantities are often grouped under a single notion of ``importance,'' they capture fundamentally different properties. Necessity measures dependence of the pretrained model on an existing computation, whereas Plasticity measures receptivity to new information during adaptation. The relationship between these two quantities is therefore an empirical and architectural question.

To study this relationship, we introduce \emph{Necessity-Plasticity Alignment} ($\mathcal{NPA}$), defined as the Spearman correlation between layer rankings induced by Necessity and Plasticity:
\begin{equation*}
\resizebox{0.95\linewidth}{!}{$%
\mathcal{NPA}(M,t)
=
\rho_{\mathrm{Spearman}}
\bigl(
R_{\mathrm{Nec}}(M,t),
R_{\mathrm{Plast}}(M,t)
\bigr),
$}
\end{equation*}
where $R_{\mathrm{Nec}}$ and $R_{\mathrm{Plast}}$ denote the corresponding layer rankings for model $M$ on task $t$.

Positive $\mathcal{NPA}$ indicates that the same layers are both necessary and plastic, while negative $\mathcal{NPA}$ indicates that pretrained reliance and adaptation concentrate on opposite ends of the network. Values near zero indicate that no stable agreement exists between the two notions of importance.

This protocol allows us to compare layer-importance structure across architectures independently of any single estimator. Rather than asking whether one importance metric is universally correct, we ask whether different estimators consistently align with Necessity or Plasticity, and whether the relationship between these quantities changes across model families.
\section{Transformers and SSMs Disagree on Layer Importance}
We now apply the diagnostic protocol to the transfer question directly. If transformer-derived layer-importance concepts transferred uniformly to SSMs, Necessity and Plasticity rankings should relate similarly across both families. Instead, their relationship changes sign across architectures.

\subsection{Necessity--Plasticity Alignment changes sign across architectures}
\label{sec:consistency_results}

\begin{figure*}[t]
    \centering
    \includegraphics[width=\textwidth]{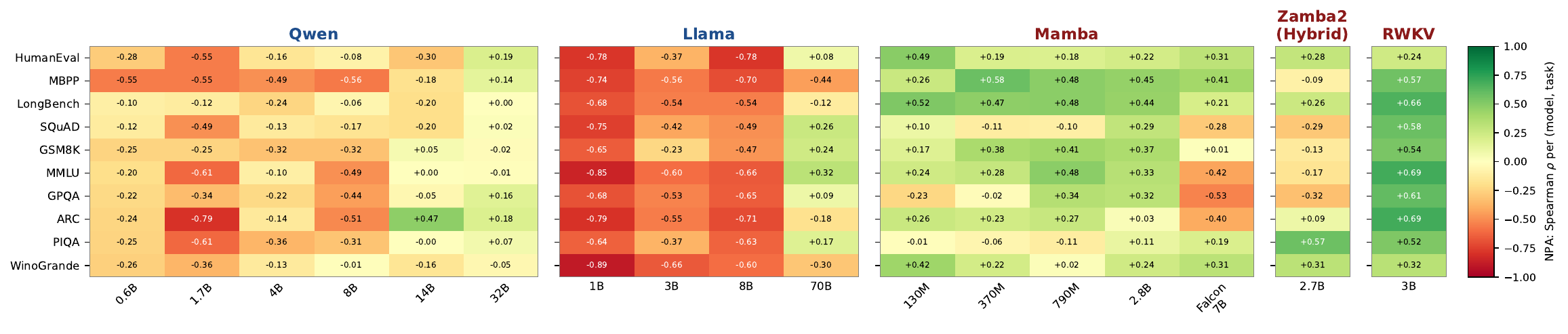}
    \caption{\textbf{Per-task $\mathcal{NPA}$ across architectures and scales.}
    Spearman~$\rho$ between Necessity and Plasticity rankings per (model, task) pair. Transformers are predominantly negative (red) up to $14$B, fading toward zero at larger scale. Mamba-family and RWKV6-3B are positive (green). Hybrid and boundary models sit near zero.}
    \label{fig:taskwise_npa}
\end{figure*}

Figure~\ref{fig:overview} reports $\mathcal{NPA}(M)$ between ablation-based Necessity rankings and LoRA-delta-based Plasticity rankings. The sign of $\mathcal{NPA}$ cleanly separates the two families. All evaluated Mamba-style SSMs show positive alignment: layers receiving larger LoRA updates are also those whose ablation causes larger loss increases. In contrast, every evaluated residual transformer up to $14$B shows negative alignment: the layers most necessary for preserving pretrained computation are not the layers where fine-tuning writes the largest task-specific updates. Figure~\ref{fig:taskwise_npa} resolves the same measurement per (model, task) pair and shows that the split is not carried by a handful of outlier tasks: within each family the sign is consistent across the task suite.

The largest transformers weaken this pattern without becoming SSM-like. Qwen3-32B and Llama-3.1-70B move toward near-zero alignment, suggesting a partial reorganization of Plasticity rather than recovery of a universal layer ordering. This disagreement is not estimator noise: LoRA-delta rankings remain self-consistent under matched scopes while continuing to disagree with ablation. Strong negative $\mathcal{NPA}$ requires a concentrated high-Necessity region together with an adaptation profile that avoids it, and both ingredients weaken with scale; since $\mathcal{NPA}$ is a rank correlation, flattening either profile drives it toward zero. We therefore state the negative-transformer result for the evaluated $\le 14$B regime and do not extrapolate beyond the measured checkpoints.

\paragraph{The split is not an artifact of low-rank adaptation.}
Because Plasticity is read off a LoRA update, the split could in principle reflect where LoRA places updates rather than a property of the architecture. An adapter-free control on $13$ checkpoints rules this out (Appendix~\ref{app:fullft}): full-FT and LoRA update profiles are positively rank-correlated on every one, and the sign of the alignment is preserved throughout, negative for every evaluated transformer with ablation baselines and positive or near zero for the evaluated Mamba, RWKV, and hybrid checkpoints. Normalizing updates by pretrained weight norm leaves this unchanged. The conclusion therefore does not depend on LoRA rank, scaling, initialization, or target-module selection.

\paragraph{Generalization beyond Mamba.}
The positive-alignment regime is not specific to Mamba. RWKV ~\citep{RWKV} exhibits strongly positive alignment ($\mathcal{NPA}=+0.54$, with all per-task correlations positive). Zamba2 ~\citep{zamba2}, a hybrid SSM+attention model, lies near zero but slightly positive ($\mathcal{NPA}=+0.05$), while Falcon-Mamba-7B~\citep{falconmamba} lies slightly below zero ($\mathcal{NPA}=-0.02$). Together, these results suggest the ordering
\begin{equation*}
\resizebox{0.95\linewidth}{!}{$%
\mathcal{NPA}_{\mathrm{SSM}}
>
\mathcal{NPA}_{\mathrm{Hybr.}}
\approx 0
>
\mathcal{NPA}_{\mathrm{Transf.} \leq 14\mathrm{B}}.
$}
\end{equation*}
The transformer-SSM contrast is strongest at small and medium scale, while very large transformers and boundary SSMs move toward a shared near-zero regime.

\paragraph{Necessity and Plasticity also dissociate \emph{within} a hybrid model.}
Zamba2-2.7B interleaves $9$ shared-attention blocks among $45$ Mamba blocks, which lets us test the dissociation without changing checkpoints. Attention blocks are only $16.7\%$ of the layers but absorb $36.3\%$ of LoRA and $31.6\%$ of full-FT update mass, while carrying just $10.4\%$ of ablation Necessity mass, which concentrates in the Mamba blocks (Appendix~\ref{app:zamba-components}). Zamba2's near-zero whole-model $\mathcal{NPA}$ is therefore the average of a transformer-like and an SSM-like component, not the absence of the effect.

\subsection{Depth profiles explain the architecture split}

\begin{figure}[t]
\centering
\includegraphics[width=\linewidth]{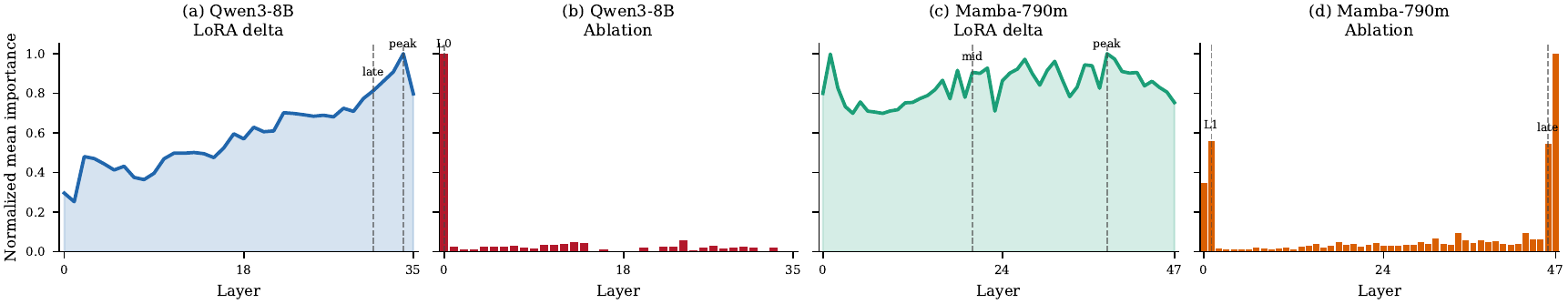}
\caption{\textbf{Representative Necessity and Plasticity depth profiles.}
Qwen3-8B: LoRA delta concentrates near the final layers (a); ablation is dominated by an early layer-0 spike (b). Mamba-790m: LoRA delta peaks mid-network (c); ablation emphasizes boundary layers with partial overlap (d).
Transformers select opposite ends of depth; Mamba profiles differ but are not spatially opposed.}
\label{fig:money}
\end{figure}

Figure~\ref{fig:money} shows the spatial origin of the sign change. In transformers, Plasticity concentrates near the output side of the network, whereas Necessity concentrates near the input side, often dominated by a strong layer-$0$ component. The two quantities therefore select opposite ends of depth, producing negative alignment. In Mamba-style SSMs, the profiles are different but not spatially opposed. Plasticity peaks in middle layers, while Necessity emphasizes boundary layers with substantial overlap. Because the two quantities co-occupy partially overlapping regions, the resulting alignment remains positive.

This depth organization also explains why large transformers move toward zero alignment. Their Necessity profiles remain strongly early-layer dominated, but Plasticity becomes less purely terminal and develops additional early-layer mass. The resulting partial overlap weakens anti-alignment without producing the positive agreement observed in SSMs.

\subsection{Layer-\texorpdfstring{$0$}{0} and metric controls}

\begin{figure}[t]
\centering
\includegraphics[width=\linewidth]{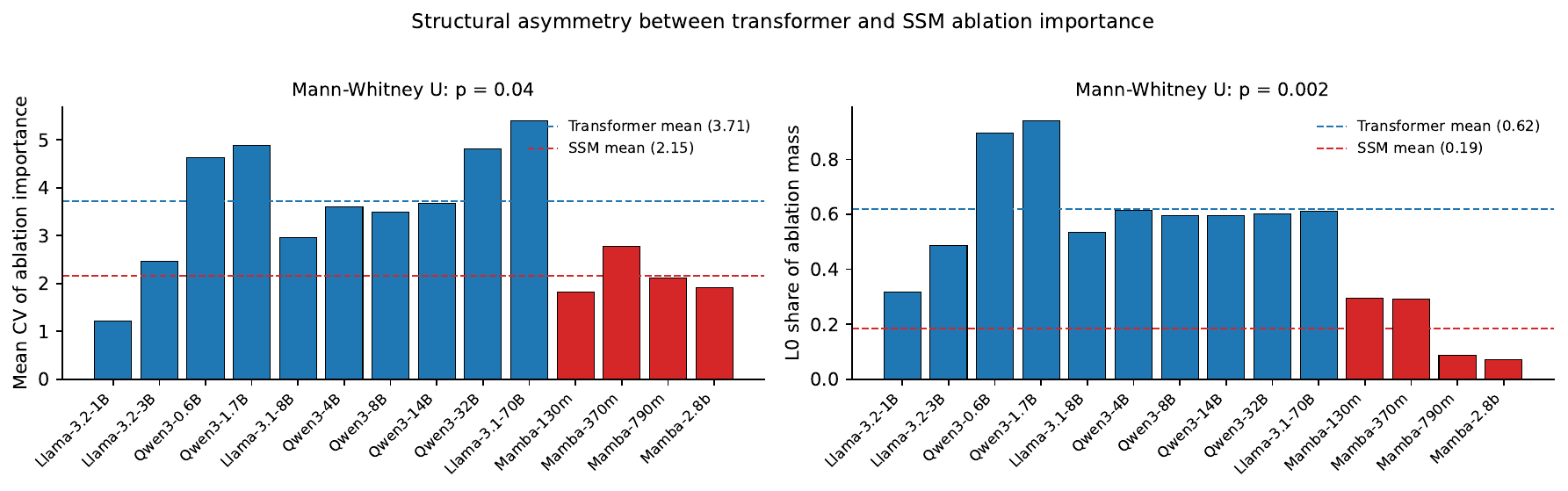}
\caption{%
  Leave-out-L0 and L0-share controls per checkpoint.
  Transformers concentrate more ablation mass on layer~$0$ than Mamba models,
  yet removing layer~$0$ shifts transformer alignment toward zero without reversing its sign,
  ruling out the embedding bottleneck as the sole driver of the architecture split.
}
\label{fig:robustness-l0}
\end{figure}

A natural alternative explanation is that the transformer--SSM contrast is driven entirely by layer~$0$. Figure~\ref{fig:robustness-l0} rules out this strong form. Removing layer~$0$ shifts transformer checkpoints toward zero, but every evaluated transformer up to $14$B remains negatively aligned, while Mamba checkpoints change minimally. Transformers do allocate substantially more Necessity mass to layer~$0$ than Mamba models, but the anti-alignment persists even after removing it, indicating a broader depth-wise separation between Necessity and Plasticity.

We also test whether ranking-based placement conclusions transfer uniformly across validation metrics. They do not. On multiple-choice evaluation, the relative advantage of top-$k$ versus bottom-$k$ placement depends on both the estimator and the model family. Even after rankings are computed under a shared proxy protocol, their downstream interpretation remains metric-relative.

\subsection{Estimator spectrum and implications}

The Necessity--Plasticity distinction also organizes estimators beyond the two anchors. In transformers, gradient- and resnorm-based estimators align more strongly with Necessity, while activation norm and TELL-TALE align more strongly with Plasticity. ShapLoRA occupies an intermediate position. In Mamba, where Necessity and Plasticity are already positively aligned, the estimators cluster within the same mixed-positive region.

These results show that layer importance is not a universal ordering of depth. Ablation-based, adaptation-based, and proxy estimators are not noisy measurements of a single latent ranking; they measure different properties of the model. A layer can be important because removing it disrupts pretrained computation, or because fine-tuning preferentially writes into it.

Importance-guided placement should therefore be treated as a measurement problem rather than a model-independent fact.

\begin{figure}[t]
\centering
\includegraphics[width=\linewidth]{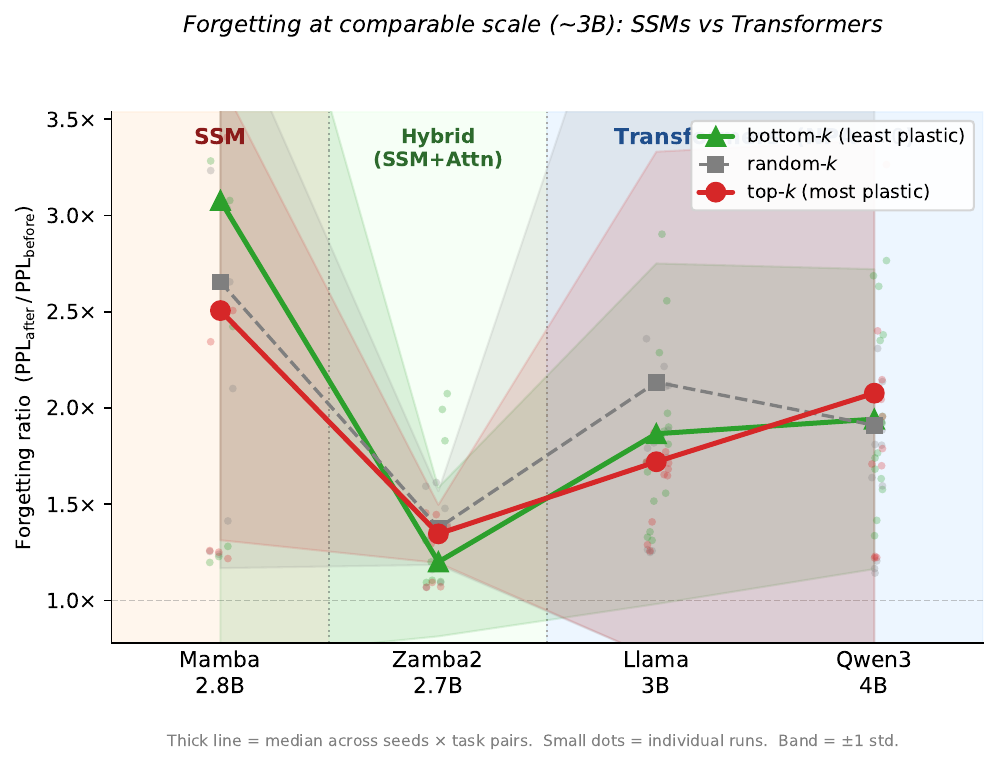}
\caption{\textbf{Tier-dependent forgetting is a transformer effect.}
Forgetting ratio ($\text{PPL}_\text{after}/\text{PPL}_\text{before}$) for
top-$k$, random-$k$, and bottom-$k$ LoRA placement ($k{=}6$).
In transformers (right, NPA~$<0$), top-$k$ placement causes significantly more forgetting than bottom-$k$ ($p{<}0.001$, $66$ pairs).
In SSMs (left, NPA~$>0$), no significant difference is observed ($p{=}0.45$, $54$ pairs).}
\label{fig:cl_forgetting}
\end{figure}

\subsection{A two-task probe of the forgetting mechanism}
\label{sec:cl_forgetting}

If $\mathcal{NPA}$ captures the spatial relationship between Necessity and Plasticity, it should predict when selective adaptation creates stability failures. We test this in a controlled two-task setup where only the adapted layers change. This experiment is a \emph{mechanism probe}, not a continual-learning benchmark: two tasks is the minimal setting in which forgetting is well defined, and restricting to it isolates the Necessity--Plasticity overlap from confounds longer sequences introduce (interference accumulation, capacity saturation, implicit replay through task similarity). We scope its conclusions to the evaluated two-task setting and the fixed LoRA configuration described here. For each model, we compute LoRA-delta scores on a source task~A, select either the top-$k$ most plastic or bottom-$k$ least plastic layers, fine-tune on task~A, then fine-tune on task~B using the same layer subset. Forgetting is measured on task~A as
\[
\mathrm{FR}
=
\frac{\mathrm{PPL}_{\mathrm{after}}}{\mathrm{PPL}_{\mathrm{before}}}.
\]

Figure~\ref{fig:cl_forgetting} shows that the plasticity-stability tradeoff is architecture-dependent. In transformers ($\mathcal{NPA}<0$), adapting through the most plastic layers causes substantially more forgetting than adapting through the least plastic layers. The top-$k$ condition consistently lies above bottom-$k$, indicating that highly plastic late layers form a shared adaptation bottleneck: successive tasks write into the same narrow region and partially overwrite one another. In Mamba-style SSMs ($\mathcal{NPA}>0$), this tier effect disappears. Top-$k$ and bottom-$k$ placement produce statistically indistinguishable forgetting ratios ($p=0.45$), consistent with Plasticity being distributed rather than concentrated in a single vulnerable tier. Falcon-Mamba-7B lies between these regimes: its near-zero alignment ($\mathcal{NPA}=-0.02$) is matched by a weak, marginally significant forgetting effect ($p=0.055$). Thus, $\mathcal{NPA}$ acts as a continuous predictor of bottleneck risk: negative alignment produces localized interference, near-zero alignment produces unstable effects, and positive alignment produces no reliable tier-dependent bottleneck. The same qualitative pattern is stable across scale; we report the full scale-wise plot in Figure~\ref{fig:cl_forgetting_scale} of Appendix~\ref{app:forgetting}.

\paragraph{EWC regularization baseline.}
We also compare placement against elastic weight consolidation (EWC)~\citep{kirkpatrick2017overcoming}, which regularizes parameter movement according to Fisher information rather than restricting which layers are adapted. Figure~\ref{fig:ewc_comparison} shows that the plasticity--stability tradeoff has different geometry across architectures. In both Mamba checkpoints, EWC achieves the lowest median forgetting. For Mamba-790M, EWC reaches a median forgetting ratio of $1.45{\times}$, compared with $2.21{\times}$ for top-$k$ and $1.68{\times}$ for bottom-$k$ ($p=0.004$); Mamba-2.8B shows the same qualitative ordering. This supports the view that when Plasticity is distributed across depth, stability is better recovered by globally constraining update magnitude than by selecting a small subset of layers.

In transformers, the tradeoff is more localized. For Qwen3-0.6B, EWC reduces forgetting relative to top-$k$ but remains worse than bottom-$k$. For Llama-3.2-3B, EWC is again worse than bottom-$k$ ($p=0.004$). Thus, transformers benefit more from avoiding the high-plasticity bottleneck, whereas SSMs benefit more from global regularization.

\paragraph{Recommendation for adapter placement.}
In the evaluated transformers ($\mathcal{NPA}<0$), the high-Necessity and high-Plasticity regions are disjoint, so placement is consequential: place adapters away from the high-Necessity early block, or protect that block when tasks arrive in sequence. In the evaluated Mamba-style SSMs ($\mathcal{NPA}>0$), Plasticity is distributed rather than concentrated in one vulnerable tier, so globally constraining update magnitude (EWC) helps more than placement. Both rules are read off the $\le 14$B regime and the two-task probe.

\begin{figure}[t]
\centering
\includegraphics[width=\linewidth]{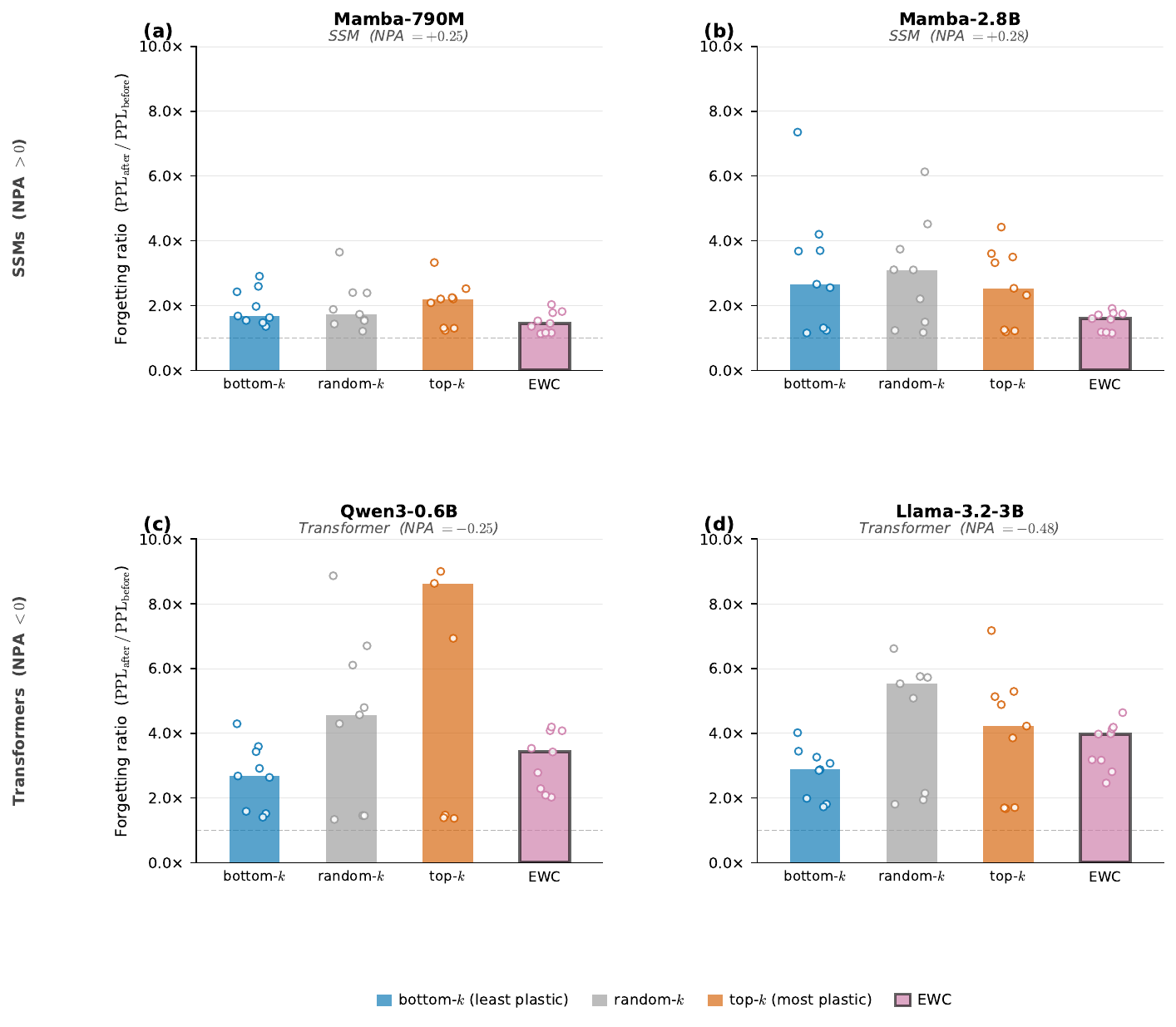}
\caption{\textbf{EWC versus placement strategies across architectures.}
Forgetting ratio for bottom-$k$, random-$k$, top-$k$, and EWC (Fisher-weighted, all layers). In Mamba-790M (NPA~$>0$, left), EWC achieves the lowest forgetting. In both transformers (NPA~$<0$, centre and right), bottom-$k$ placement matches or beats EWC, consistent with the shared-bottleneck mechanism.}
\label{fig:ewc_comparison}
\end{figure}

\section{Mechanistic Account}
\label{sec:mechanism}
The previous section showed that Necessity and Plasticity anti-align in residual transformers but align in Mamba-style SSMs. We now give a mechanistic account of this split. The central idea is that transformer anti-alignment arises from a depth-wise separation between pretrained dependence and low-cost adaptation: early layers are costly to remove, whereas later layers provide favorable sites for task-specific updates. Mamba-style SSMs do not exhibit the same systematic opposition across depth, explaining why transformer-calibrated layer-importance tools do not transfer uniformly across architectures.

Necessity and Plasticity coincide only under restrictive assumptions, such as locally isotropic curvature, infinitesimal interventions, and unconstrained optimization. Large pretrained models do not satisfy these assumptions in practice. The relevant question is therefore not whether ablation and LoRA-delta estimate the same latent quantity, but when architecture makes their rankings align or diverge.

We approach this question from two complementary perspectives. The first is a \emph{curvature} view: ablation importance measures the cost of removing a layer's pretrained contribution, whereas LoRA-delta importance measures where optimization can write a large, low-cost task-specific update.

The second is a \emph{residual-Jacobian} view. In residual networks, removing layer $l$ induces a perturbation that propagates through the downstream Jacobian product
\[
\prod_{j>l}(I+J_{f_j}),
\]
creating a directional asymmetry across depth. Both perspectives predict the same empirical pattern: early-layer concentration of Necessity and late-layer concentration of Plasticity in transformers.

\subsection{Residual geometry biases Necessity toward early layers in transformers}

Consider a pre-LN residual block
\[
h_l = h_{l-1} + f_l(h_{l-1}),
\]
with the loss attached to the final representation $h_L$. Bypassing the layer removes the residual computation $f_l(h_{l-1})$ while preserving the skip path. Linearizing the resulting perturbation gives
\begin{multline*}
I_{\mathrm{abl}}(l,t)
\approx
\mathbb{E}_{x}\!\Big[
\Big\langle
\tfrac{\partial L_t}{\partial h_L},\\
\Big(
\prod_{j=l+1}^{L}(I + J_{f_j})
\Big)
f_l(h_{l-1})
\Big\rangle
\Big]\\
+
O(\|f_l(h_{l-1})\|^2),
\end{multline*}
where $J_{f_j}$ denotes the Jacobian of block $j$.

This expression does not imply that early layers must always dominate ablation. The downstream Jacobian product can amplify, suppress, rotate, or cancel perturbations. However, earlier residual contributions propagate through more downstream blocks and participate in constructing the representations on which all later layers operate. This creates a systematic structural bias toward the input side of the network.

Figure~\ref{fig:money} shows that this bias dominates empirically. In transformers, Necessity mass concentrates strongly in early layers, often with a pronounced layer-$0$ spike, whereas Plasticity concentrates near the output side of the network. The result is a spatial separation between pretrained dependence and task adaptation.

Under the curvature interpretation, the same pattern emerges for a complementary reason. Late transformer layers can absorb task-specific updates with more direct influence on the final output and less disruption to earlier computations, making them favorable locations for adaptation. Plasticity therefore accumulates late even when Necessity remains concentrated early.

Together, these effects predict negative alignment between Necessity and Plasticity in residual transformers.

\textbf{Prediction 1:}
In residual transformers, Necessity and Plasticity should be negatively correlated across depth.

\subsection{Mamba-style SSMs weaken this depth-wise separation}

The transformer mechanism relies not only on residual depth, but on a particular depth organization: pretrained dependence concentrates early, while low-cost adaptation concentrates late. Mamba-style SSMs retain residual depth, but their selective recurrent dynamics do not induce the same empirical separation.

First, every Mamba layer directly reads the current token through a learned selective gate. Input information is therefore injected continuously across depth rather than once at the beginning of a residual stream. No single layer plays the privileged initialization role associated with transformer layer~$0$, so there is no structural reason for Necessity to collapse toward the input side. Second, the selective recurrence propagates information through an impulse-response operator $\Phi_{j,l}(x)$ whose eigenvalues remain bounded inside the unit disk~\citep{mamba}. The expected propagation norm therefore depends primarily on distance $|j-l|$ rather than direction from input to output. Gradient propagation through the recurrence does not preferentially amplify late-layer adaptation.

Together, these properties remove the architectural mechanism forcing separation between Necessity and Plasticity. Figure~\ref{fig:money} shows the resulting behavior empirically: in Mamba, Necessity and Plasticity occupy partially overlapping regions rather than opposite ends of depth.

Appendix~\ref{app:thm-mamba} formalizes this intuition. Under assumptions B1$'$--B2, the mass-displacement coefficient $\delta$ governing separation between Necessity and Plasticity vanishes asymptotically:
\[
|\delta|
\le
\frac{2K}{L}
\to 0
\qquad
\text{as}
\qquad
L/K \to \infty.
\]

Thus, under the stated assumptions, the mechanism producing systematic rank anti-correlation in transformers vanishes asymptotically.

\textbf{Prediction 2:}
Mamba-style SSMs should not exhibit systematic negative Necessity--Plasticity Alignment; their alignment should be non-negative or near zero.

\subsection{Layer-$0$ controls rule out a pure embedding-bottleneck explanation}
\label{sec:l0-asymmetry}

The residual account further predicts that transformers should allocate unusually large Necessity mass to layer~$0$. Figure~\ref{fig:robustness-l0} confirms this prediction: transformers assign substantially more ablation mass to layer~$0$ than Mamba models.

However, layer-$0$ dominance alone does not explain the full architecture split. If the effect were purely an embedding bottleneck artifact, removing layer~$0$ from the rankings should eliminate the transformer--SSM contrast. Instead, Figure~\ref{fig:robustness-l0} shows that the contrast survives. Transformer alignment becomes less negative but remains negative, whereas Mamba checkpoints remain positive.

The phenomenon therefore reflects a broader separation between early-side Necessity and late-side Plasticity rather than a single anomalous layer.

\textbf{Prediction 3:}
Removing layer~$0$ should weaken negative alignment in transformers.


\subsection{Curvature measurements support the decomposition}
\label{sec:scale-failure}
The curvature account predicts that layer-wise curvature should co-vary with Necessity and anti-correlate with Plasticity in transformers, but not in SSMs. We test this using diagonal Fisher information:
\[
F_l
=
\mathbb{E}_x
\left[
\|
\nabla_{\theta_l}\mathcal{L}
\|^2
\right].
\]

\begin{figure}[t]
\centering
\includegraphics[width=\linewidth]{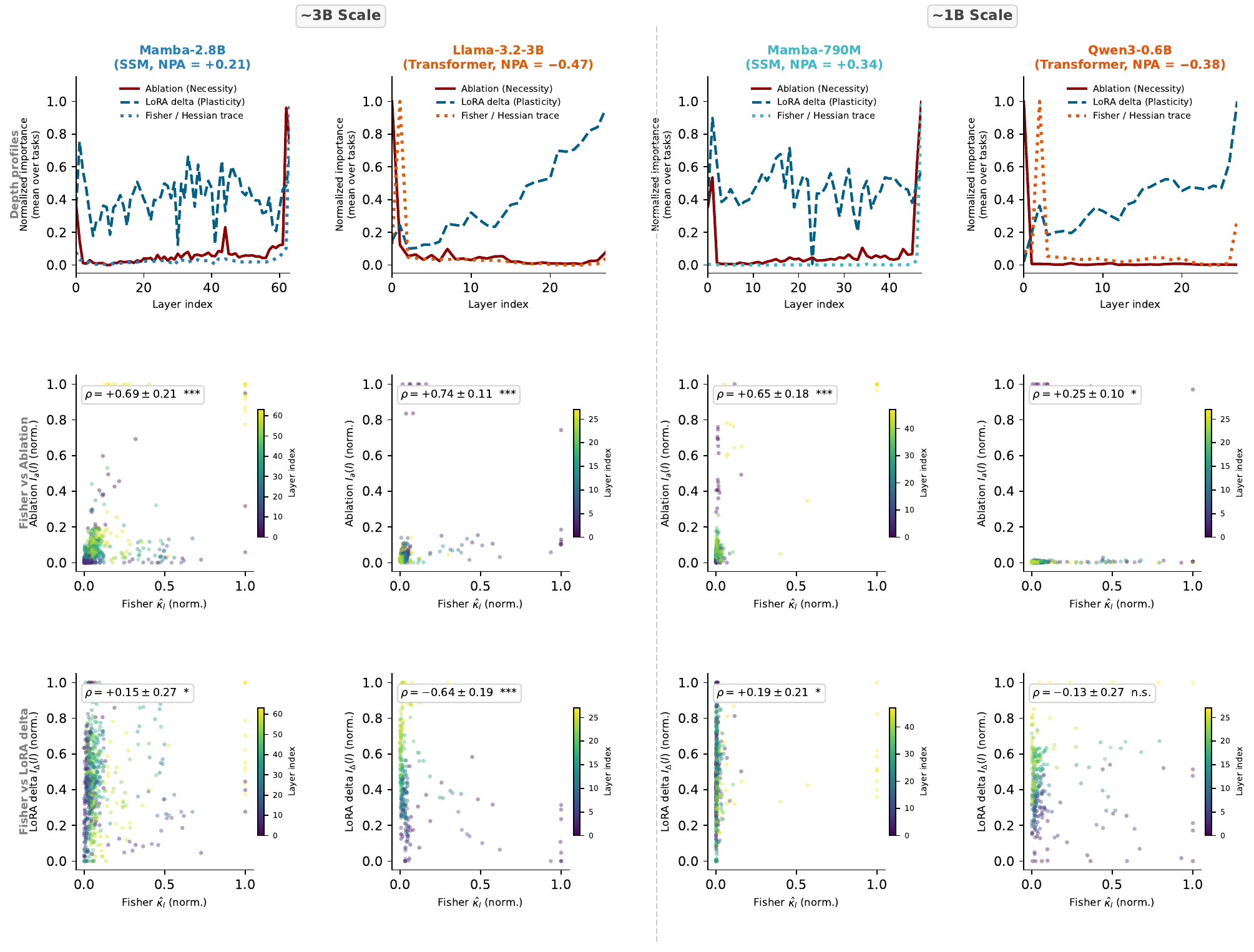}
\caption{%
  \textbf{Fisher curvature supports the Necessity--Plasticity split.}
  Two size-matched SSM--Transformer pairs ($\sim$3B: Mamba-2.8B vs.\ Llama-3.2-3B; $\sim$1B: Mamba-790M vs.\ Qwen3-0.6B).
  \textbf{Top}: depth profiles of Fisher curvature (dotted), Necessity (solid), and Plasticity (dashed).
  \textbf{Middle/Bottom}: scatter of Fisher vs.\ ablation and vs.\ LoRA delta.
  In transformers, Fisher curvature co-varies with Necessity and anti-correlates with Plasticity.
  In SSMs, both correlations are non-negative.
}
\label{fig:hessian-curvature}
\end{figure}

Figure~\ref{fig:hessian-curvature} supports this prediction. In transformers, Fisher curvature co-varies with Necessity and anti-correlates with Plasticity: high-Fisher layers are costly to remove but receive smaller LoRA updates. In SSMs, Fisher remains positively associated with Necessity and no longer anti-correlates with Plasticity. Thus, curvature separates pretrained dependence from adaptation in transformers but not in the same way in Mamba-style SSMs.

\subsection{Scale and estimator structure}

The same account explains why transformer anti-alignment weakens at larger scale. Negative alignment requires early Necessity together with late-concentrated Plasticity. Qwen3-32B and Llama-3.1-70B retain early Necessity, but their Plasticity profiles develop an additional early-layer mode (Figure~\ref{fig:bimodality}, and greater depth and redundancy flatten both profiles. This partial overlap collapses $\mathcal{NPA}$ toward zero without producing the positive alignment observed in SSMs. The mechanism predicts the direction of the trend but not where it terminates, so we do not extrapolate the sign beyond the measured checkpoints.

Finally, the decomposition also organizes proxy estimators. Figure~\ref{fig:overview} shows that, in transformers, gradient- and resnorm-based methods align more strongly with Necessity, while activation norm and TELL-TALE align more strongly with Plasticity; ShapLoRA lies between them. In Mamba-style SSMs, where the anchors are already positively aligned, estimators cluster in the same mixed-positive region. Layer importance is therefore not a universal ordering of depth, but a measurement whose meaning depends on both architecture and intervention.

\section{Conclusion}

We showed that layer importance is not a model-independent ordering of depth. By separating importance into \emph{Necessity}, the dependence of the pretrained model on a layer's existing computation, and \emph{Plasticity}, the tendency of adaptation to write task-specific updates into that layer, we find a systematic architecture-dependent split. Every evaluated residual transformer up to $14$B exhibits negative Necessity--Plasticity Alignment: early layers are most necessary, while later layers are most plastic. The evaluated Mamba-style SSMs instead show positive or near-zero alignment, with the two quantities occupying overlapping regions of depth. The split survives an adapter-free control and reappears within a hybrid checkpoint.

This split is not only diagnostic; it predicts the stability cost of selective adaptation. In a controlled two-task probe, adapting the most plastic layers of the evaluated transformers creates a shared late-layer bottleneck and increases forgetting, whereas in the evaluated Mamba-style SSMs this tier-dependent effect largely disappears. Layer-$0$ controls, and corroborating curvature measurements, support the same interpretation: transformer anti-alignment reflects a broader separation between pretrained dependence and adaptation, not a single anomalous layer.

Overall, layer-importance methods should be treated as measurements of specific properties rather than interchangeable estimates of one universal ranking. Tools calibrated on transformers may therefore fail silently on recurrent or hybrid architectures, where the geometry of stored computation and new adaptation can differ fundamentally.

\section*{Acknowledgements}
Sarath Chandar is supported by the Canada CIFAR AI Chairs program, the Canada Research Chair in Lifelong Machine Learning, and the NSERC Discovery Grant. Irina Rish is supported by the Canada CIFAR AI Chairs program and the Canada Excellence Research Chair in Autonomous AI. Experiments were conducted using computational resources provided by Mila Quebec AI Institute.
\section{Limitations}

Our analysis studies layer importance through two operational anchors: ablation-induced loss increase for Necessity and update magnitude for Plasticity. These definitions are well matched to pruning and parameter-efficient adaptation, but they do not cover all possible notions of importance. Plasticity in particular is defined through the outcome of an adaptation procedure rather than as an intrinsic layer property. We reduce, but do not eliminate, the resulting dependence: Appendix~\ref{app:fullft} shows that replacing LoRA with unconstrained full-parameter fine-tuning, and normalizing updates by pretrained weight norm, both preserve the sign of $\mathcal{NPA}$ on every retrained checkpoint. Other adaptation procedures (prefix tuning, other PEFT families, different optimizers or budgets) were not tested. On the Necessity side, ablation measures coarse total functional reliance at the layer level; activation patching, causal tracing, and optimizer-state analyses could decompose that reliance into finer mechanisms and may expose additional structure. Our model coverage spans transformers, Mamba-style SSMs, RWKV, and hybrid checkpoints, but remains limited relative to the full architecture space, and all architecture-level statements in this paper should be read as scoped to the evaluated checkpoints and scales. In particular, hybrid models and very large transformers move toward a near-zero alignment regime, so $\mathcal{NPA}$ should be interpreted as a continuous diagnostic rather than a strict family label. The negative-transformer result is established at $\le 14$B; at $32$B and $70$B the effect is near zero, and we have no evidence about scales beyond $70$B. The hybrid component analysis rests on a single checkpoint with one attention/SSM ratio. The mechanistic account is also partial. The residual-Jacobian and curvature views explain the observed depth-wise patterns and are supported by layer-$0$ controls and Fisher measurements, but the theoretical argument relies on simplifying assumptions. Diagonal Fisher is a coarse proxy for curvature: it discards off-diagonal terms and is sensitive to parameter scale, so we use it as corroborating evidence only. None of the paper's main claims depend on it. Finally, our continual-learning experiments are a two-task mechanism probe rather than a benchmark. They isolate layer placement under fixed LoRA configurations, task pairs, and optimization budgets. Longer or more varied task sequences, different ranks, longer training, replay, or full-model updates may change the magnitude and even the ordering of forgetting effects. Our conclusion is therefore not that one placement rule is universally optimal, but that the meaning and risk of layer selection depend on architecture and on what the estimator measures.

\bibliography{custom}

@misc{vaswani2017attention,
      title={Attention Is All You Need}, 
      author={Ashish Vaswani and Noam Shazeer and Niki Parmar and Jakob Uszkoreit and Llion Jones and Aidan N. Gomez and Lukasz Kaiser and Illia Polosukhin},
      year={2023},
      eprint={1706.03762},
      archivePrefix={arXiv},
      primaryClass={cs.CL},
      url={https://arxiv.org/abs/1706.03762}, 
}

@misc{qwen3,
      title={Qwen3 Technical Report}, 
      author={An Yang and Anfeng Li and Baosong Yang and Beichen Zhang and Binyuan Hui and Bo Zheng and Bowen Yu and Chang Gao and Chengen Huang and Chenxu Lv and Chujie Zheng and Dayiheng Liu and Fan Zhou and Fei Huang and Feng Hu and Hao Ge and Haoran Wei and Huan Lin and Jialong Tang and Jian Yang and Jianhong Tu and Jianwei Zhang and Jianxin Yang and Jiaxi Yang and Jing Zhou and Jingren Zhou and Junyang Lin and Kai Dang and Keqin Bao and Kexin Yang and Le Yu and Lianghao Deng and Mei Li and Mingfeng Xue and Mingze Li and Pei Zhang and Peng Wang and Qin Zhu and Rui Men and Ruize Gao and Shixuan Liu and Shuang Luo and Tianhao Li and Tianyi Tang and Wenbiao Yin and Xingzhang Ren and Xinyu Wang and Xinyu Zhang and Xuancheng Ren and Yang Fan and Yang Su and Yichang Zhang and Yinger Zhang and Yu Wan and Yuqiong Liu and Zekun Wang and Zeyu Cui and Zhenru Zhang and Zhipeng Zhou and Zihan Qiu},
      year={2025},
      eprint={2505.09388},
      archivePrefix={arXiv},
      primaryClass={cs.CL},
      url={https://arxiv.org/abs/2505.09388}, 
}

@misc{llama3,
      title={The Llama 3 Herd of Models}, 
      author={Aaron Grattafiori and Abhimanyu Dubey and Abhinav Jauhri and Abhinav Pandey and Abhishek Kadian and Ahmad Al-Dahle and Aiesha Letman and Akhil Mathur and Alan Schelten and Alex Vaughan and Amy Yang and Angela Fan and Anirudh Goyal and Anthony Hartshorn and Aobo Yang and Archi Mitra and Archie Sravankumar and Artem Korenev and Arthur Hinsvark and Arun Rao and Aston Zhang and Aurelien Rodriguez and Austen Gregerson and Ava Spataru and Baptiste Roziere and Bethany Biron and Binh Tang and Bobbie Chern and Charlotte Caucheteux and Chaya Nayak and Chloe Bi and Chris Marra and Chris McConnell and Christian Keller and Christophe Touret and Chunyang Wu and Corinne Wong and Cristian Canton Ferrer and Cyrus Nikolaidis and Damien Allonsius and Daniel Song and Danielle Pintz and Danny Livshits and Danny Wyatt and David Esiobu and Dhruv Choudhary and Dhruv Mahajan and Diego Garcia-Olano and Diego Perino and Dieuwke Hupkes and Egor Lakomkin and Ehab AlBadawy and Elina Lobanova and Emily Dinan and Eric Michael Smith and Filip Radenovic and Francisco Guzmán and Frank Zhang and Gabriel Synnaeve and Gabrielle Lee and Georgia Lewis Anderson and Govind Thattai and Graeme Nail and Gregoire Mialon and Guan Pang and Guillem Cucurell and Hailey Nguyen and Hannah Korevaar and Hu Xu and Hugo Touvron and Iliyan Zarov and Imanol Arrieta Ibarra and Isabel Kloumann and Ishan Misra and Ivan Evtimov and Jack Zhang and Jade Copet and Jaewon Lee and Jan Geffert and Jana Vranes and Jason Park and Jay Mahadeokar and Jeet Shah and Jelmer van der Linde and Jennifer Billock and Jenny Hong and Jenya Lee and Jeremy Fu and Jianfeng Chi and Jianyu Huang and Jiawen Liu and Jie Wang and Jiecao Yu and Joanna Bitton and Joe Spisak and Jongsoo Park and Joseph Rocca and Joshua Johnstun and Joshua Saxe and Junteng Jia and Kalyan Vasuden Alwala and Karthik Prasad and Kartikeya Upasani and Kate Plawiak and Ke Li and Kenneth Heafield and Kevin Stone and Khalid El-Arini and Krithika Iyer and Kshitiz Malik and Kuenley Chiu and Kunal Bhalla and Kushal Lakhotia and Lauren Rantala-Yeary and Laurens van der Maaten and Lawrence Chen and Liang Tan and Liz Jenkins and Louis Martin and Lovish Madaan and Lubo Malo and Lukas Blecher and Lukas Landzaat and Luke de Oliveira and Madeline Muzzi and Mahesh Pasupuleti and Mannat Singh and Manohar Paluri and Marcin Kardas and Maria Tsimpoukelli and Mathew Oldham and Mathieu Rita and Maya Pavlova and Melanie Kambadur and Mike Lewis and Min Si and Mitesh Kumar Singh and Mona Hassan and Naman Goyal and Narjes Torabi and Nikolay Bashlykov and Nikolay Bogoychev and Niladri Chatterji and Ning Zhang and Olivier Duchenne and Onur Çelebi and Patrick Alrassy and Pengchuan Zhang and Pengwei Li and Petar Vasic and Peter Weng and Prajjwal Bhargava and Pratik Dubal and Praveen Krishnan and Punit Singh Koura and Puxin Xu and Qing He and Qingxiao Dong and Ragavan Srinivasan and Raj Ganapathy and Ramon Calderer and Ricardo Silveira Cabral and Robert Stojnic and Roberta Raileanu and Rohan Maheswari and Rohit Girdhar and Rohit Patel and Romain Sauvestre and Ronnie Polidoro and Roshan Sumbaly and Ross Taylor and Ruan Silva and Rui Hou and Rui Wang and Saghar Hosseini and Sahana Chennabasappa and Sanjay Singh and Sean Bell and Seohyun Sonia Kim and Sergey Edunov and Shaoliang Nie and Sharan Narang and Sharath Raparthy and Sheng Shen and Shengye Wan and Shruti Bhosale and Shun Zhang and Simon Vandenhende and Soumya Batra and Spencer Whitman and Sten Sootla and Stephane Collot and Suchin Gururangan and Sydney Borodinsky and Tamar Herman and Tara Fowler and Tarek Sheasha and Thomas Georgiou and Thomas Scialom and Tobias Speckbacher and Todor Mihaylov and Tong Xiao and Ujjwal Karn and Vedanuj Goswami and Vibhor Gupta and Vignesh Ramanathan and Viktor Kerkez and Vincent Gonguet and Virginie Do and Vish Vogeti and Vítor Albiero and Vladan Petrovic and Weiwei Chu and Wenhan Xiong and Wenyin Fu and Whitney Meers and Xavier Martinet and Xiaodong Wang and Xiaofang Wang and Xiaoqing Ellen Tan and Xide Xia and Xinfeng Xie and Xuchao Jia and Xuewei Wang and Yaelle Goldschlag and Yashesh Gaur and Yasmine Babaei and Yi Wen and Yiwen Song and Yuchen Zhang and Yue Li and Yuning Mao and Zacharie Delpierre Coudert and Zheng Yan and Zhengxing Chen and Zoe Papakipos and Aaditya Singh and Aayushi Srivastava and Abha Jain and Adam Kelsey and Adam Shajnfeld and Adithya Gangidi and Adolfo Victoria and Ahuva Goldstand and Ajay Menon and Ajay Sharma and Alex Boesenberg and Alexei Baevski and Allie Feinstein and Amanda Kallet and Amit Sangani and Amos Teo and Anam Yunus and Andrei Lupu and Andres Alvarado and Andrew Caples and Andrew Gu and Andrew Ho and Andrew Poulton and Andrew Ryan and Ankit Ramchandani and Annie Dong and Annie Franco and Anuj Goyal and Aparajita Saraf and Arkabandhu Chowdhury and Ashley Gabriel and Ashwin Bharambe and Assaf Eisenman and Azadeh Yazdan and Beau James and Ben Maurer and Benjamin Leonhardi and Bernie Huang and Beth Loyd and Beto De Paola and Bhargavi Paranjape and Bing Liu and Bo Wu and Boyu Ni and Braden Hancock and Bram Wasti and Brandon Spence and Brani Stojkovic and Brian Gamido and Britt Montalvo and Carl Parker and Carly Burton and Catalina Mejia and Ce Liu and Changhan Wang and Changkyu Kim and Chao Zhou and Chester Hu and Ching-Hsiang Chu and Chris Cai and Chris Tindal and Christoph Feichtenhofer and Cynthia Gao and Damon Civin and Dana Beaty and Daniel Kreymer and Daniel Li and David Adkins and David Xu and Davide Testuggine and Delia David and Devi Parikh and Diana Liskovich and Didem Foss and Dingkang Wang and Duc Le and Dustin Holland and Edward Dowling and Eissa Jamil and Elaine Montgomery and Eleonora Presani and Emily Hahn and Emily Wood and Eric-Tuan Le and Erik Brinkman and Esteban Arcaute and Evan Dunbar and Evan Smothers and Fei Sun and Felix Kreuk and Feng Tian and Filippos Kokkinos and Firat Ozgenel and Francesco Caggioni and Frank Kanayet and Frank Seide and Gabriela Medina Florez and Gabriella Schwarz and Gada Badeer and Georgia Swee and Gil Halpern and Grant Herman and Grigory Sizov and Guangyi and Zhang and Guna Lakshminarayanan and Hakan Inan and Hamid Shojanazeri and Han Zou and Hannah Wang and Hanwen Zha and Haroun Habeeb and Harrison Rudolph and Helen Suk and Henry Aspegren and Hunter Goldman and Hongyuan Zhan and Ibrahim Damlaj and Igor Molybog and Igor Tufanov and Ilias Leontiadis and Irina-Elena Veliche and Itai Gat and Jake Weissman and James Geboski and James Kohli and Janice Lam and Japhet Asher and Jean-Baptiste Gaya and Jeff Marcus and Jeff Tang and Jennifer Chan and Jenny Zhen and Jeremy Reizenstein and Jeremy Teboul and Jessica Zhong and Jian Jin and Jingyi Yang and Joe Cummings and Jon Carvill and Jon Shepard and Jonathan McPhie and Jonathan Torres and Josh Ginsburg and Junjie Wang and Kai Wu and Kam Hou U and Karan Saxena and Kartikay Khandelwal and Katayoun Zand and Kathy Matosich and Kaushik Veeraraghavan and Kelly Michelena and Keqian Li and Kiran Jagadeesh and Kun Huang and Kunal Chawla and Kyle Huang and Lailin Chen and Lakshya Garg and Lavender A and Leandro Silva and Lee Bell and Lei Zhang and Liangpeng Guo and Licheng Yu and Liron Moshkovich and Luca Wehrstedt and Madian Khabsa and Manav Avalani and Manish Bhatt and Martynas Mankus and Matan Hasson and Matthew Lennie and Matthias Reso and Maxim Groshev and Maxim Naumov and Maya Lathi and Meghan Keneally and Miao Liu and Michael L. Seltzer and Michal Valko and Michelle Restrepo and Mihir Patel and Mik Vyatskov and Mikayel Samvelyan and Mike Clark and Mike Macey and Mike Wang and Miquel Jubert Hermoso and Mo Metanat and Mohammad Rastegari and Munish Bansal and Nandhini Santhanam and Natascha Parks and Natasha White and Navyata Bawa and Nayan Singhal and Nick Egebo and Nicolas Usunier and Nikhil Mehta and Nikolay Pavlovich Laptev and Ning Dong and Norman Cheng and Oleg Chernoguz and Olivia Hart and Omkar Salpekar and Ozlem Kalinli and Parkin Kent and Parth Parekh and Paul Saab and Pavan Balaji and Pedro Rittner and Philip Bontrager and Pierre Roux and Piotr Dollar and Polina Zvyagina and Prashant Ratanchandani and Pritish Yuvraj and Qian Liang and Rachad Alao and Rachel Rodriguez and Rafi Ayub and Raghotham Murthy and Raghu Nayani and Rahul Mitra and Rangaprabhu Parthasarathy and Raymond Li and Rebekkah Hogan and Robin Battey and Rocky Wang and Russ Howes and Ruty Rinott and Sachin Mehta and Sachin Siby and Sai Jayesh Bondu and Samyak Datta and Sara Chugh and Sara Hunt and Sargun Dhillon and Sasha Sidorov and Satadru Pan and Saurabh Mahajan and Saurabh Verma and Seiji Yamamoto and Sharadh Ramaswamy and Shaun Lindsay and Shaun Lindsay and Sheng Feng and Shenghao Lin and Shengxin Cindy Zha and Shishir Patil and Shiva Shankar and Shuqiang Zhang and Shuqiang Zhang and Sinong Wang and Sneha Agarwal and Soji Sajuyigbe and Soumith Chintala and Stephanie Max and Stephen Chen and Steve Kehoe and Steve Satterfield and Sudarshan Govindaprasad and Sumit Gupta and Summer Deng and Sungmin Cho and Sunny Virk and Suraj Subramanian and Sy Choudhury and Sydney Goldman and Tal Remez and Tamar Glaser and Tamara Best and Thilo Koehler and Thomas Robinson and Tianhe Li and Tianjun Zhang and Tim Matthews and Timothy Chou and Tzook Shaked and Varun Vontimitta and Victoria Ajayi and Victoria Montanez and Vijai Mohan and Vinay Satish Kumar and Vishal Mangla and Vlad Ionescu and Vlad Poenaru and Vlad Tiberiu Mihailescu and Vladimir Ivanov and Wei Li and Wenchen Wang and Wenwen Jiang and Wes Bouaziz and Will Constable and Xiaocheng Tang and Xiaojian Wu and Xiaolan Wang and Xilun Wu and Xinbo Gao and Yaniv Kleinman and Yanjun Chen and Ye Hu and Ye Jia and Ye Qi and Yenda Li and Yilin Zhang and Ying Zhang and Yossi Adi and Youngjin Nam and Yu and Wang and Yu Zhao and Yuchen Hao and Yundi Qian and Yunlu Li and Yuzi He and Zach Rait and Zachary DeVito and Zef Rosnbrick and Zhaoduo Wen and Zhenyu Yang and Zhiwei Zhao and Zhiyu Ma},
      year={2024},
      eprint={2407.21783},
      archivePrefix={arXiv},
      primaryClass={cs.AI},
      url={https://arxiv.org/abs/2407.21783}, 
}

@misc{RWKV,
      title={RWKV: Reinventing RNNs for the Transformer Era}, 
      author={Bo Peng and Eric Alcaide and Quentin Anthony and Alon Albalak and Samuel Arcadinho and Stella Biderman and Huanqi Cao and Xin Cheng and Michael Chung and Matteo Grella and Kranthi Kiran GV and Xuzheng He and Haowen Hou and Jiaju Lin and Przemyslaw Kazienko and Jan Kocon and Jiaming Kong and Bartlomiej Koptyra and Hayden Lau and Krishna Sri Ipsit Mantri and Ferdinand Mom and Atsushi Saito and Guangyu Song and Xiangru Tang and Bolun Wang and Johan S. Wind and Stanislaw Wozniak and Ruichong Zhang and Zhenyuan Zhang and Qihang Zhao and Peng Zhou and Qinghua Zhou and Jian Zhu and Rui-Jie Zhu},
      year={2023},
      eprint={2305.13048},
      archivePrefix={arXiv},
      primaryClass={cs.CL},
      url={https://arxiv.org/abs/2305.13048}, 
}

@misc{pan2024lisalayerwiseimportancesampling,
      title={LISA: Layerwise Importance Sampling for Memory-Efficient Large Language Model Fine-Tuning}, 
      author={Rui Pan and Xiang Liu and Shizhe Diao and Renjie Pi and Jipeng Zhang and Chi Han and Tong Zhang},
      year={2024},
      eprint={2403.17919},
      archivePrefix={arXiv},
      primaryClass={cs.LG},
      url={https://arxiv.org/abs/2403.17919}, 
}

@misc{gromov2025unreasonableineffectivenessdeeperlayers,
      title={The Unreasonable Ineffectiveness of the Deeper Layers}, 
      author={Andrey Gromov and Kushal Tirumala and Hassan Shapourian and Paolo Glorioso and Daniel A. Roberts},
      year={2025},
      eprint={2403.17887},
      archivePrefix={arXiv},
      primaryClass={cs.CL},
      url={https://arxiv.org/abs/2403.17887}, 
}

@misc{mamba,
      title={Mamba: Linear-Time Sequence Modeling with Selective State Spaces}, 
      author={Albert Gu and Tri Dao},
      year={2024},
      eprint={2312.00752},
      archivePrefix={arXiv},
      primaryClass={cs.LG},
      url={https://arxiv.org/abs/2312.00752}, 
}

@misc{mamba2,
      title={Transformers are SSMs: Generalized Models and Efficient Algorithms Through Structured State Space Duality}, 
      author={Tri Dao and Albert Gu},
      year={2024},
      eprint={2405.21060},
      archivePrefix={arXiv},
      primaryClass={cs.LG},
      url={https://arxiv.org/abs/2405.21060}, 
}

@misc{lora,
      title={LoRA: Low-Rank Adaptation of Large Language Models}, 
      author={Edward J. Hu and Yelong Shen and Phillip Wallis and Zeyuan Allen-Zhu and Yuanzhi Li and Shean Wang and Lu Wang and Weizhu Chen},
      year={2021},
      eprint={2106.09685},
      archivePrefix={arXiv},
      primaryClass={cs.CL},
      url={https://arxiv.org/abs/2106.09685}, 
}

@misc{adalora,
      title={AdaLoRA: Adaptive Budget Allocation for Parameter-Efficient Fine-Tuning}, 
      author={Qingru Zhang and Minshuo Chen and Alexander Bukharin and Nikos Karampatziakis and Pengcheng He and Yu Cheng and Weizhu Chen and Tuo Zhao},
      year={2023},
      eprint={2303.10512},
      archivePrefix={arXiv},
      primaryClass={cs.CL},
      url={https://arxiv.org/abs/2303.10512}, 
}

@inproceedings{alphalora,
    title = "{A}lpha{L}o{RA}: Assigning {L}o{RA} Experts Based on Layer Training Quality",
    author = "Qing, Peijun  and
      Gao, Chongyang  and
      Zhou, Yefan  and
      Diao, Xingjian  and
      Yang, Yaoqing  and
      Vosoughi, Soroush",
    editor = "Al-Onaizan, Yaser  and
      Bansal, Mohit  and
      Chen, Yun-Nung",
    booktitle = "Proceedings of the 2024 Conference on Empirical Methods in Natural Language Processing",
    month = nov,
    year = "2024",
    address = "Miami, Florida, USA",
    publisher = "Association for Computational Linguistics",
    url = "https://aclanthology.org/2024.emnlp-main.1141/",
    doi = "10.18653/v1/2024.emnlp-main.1141",
    pages = "20511--20523"
}

@misc{shaplora,
      title={ShapLoRA: Allocation of Low-rank Adaption on Large Language Models via Shapley Value Inspired Importance Estimation}, 
      author={Yi Zhao and Qinghua Yao and Xinyuan song and Wei Zhu},
      year={2026},
      eprint={2601.17921},
      archivePrefix={arXiv},
      primaryClass={cs.CL},
      url={https://arxiv.org/abs/2601.17921}, 
}

@misc{mambapeft,
      title={Parameter-Efficient Fine-Tuning of State Space Models}, 
      author={Kevin Galim and Wonjun Kang and Yuchen Zeng and Hyung Il Koo and Kangwook Lee},
      year={2025},
      eprint={2410.09016},
      archivePrefix={arXiv},
      primaryClass={cs.LG},
      url={https://arxiv.org/abs/2410.09016}, 
}

@inproceedings{
plop,
title={{PL}oP: Precise Lo{RA} Placement for Efficient Finetuning of Large Models},
author={Soufiane Hayou and Nikhil Ghosh and Bin Yu},
booktitle={The Fourteenth International Conference on Learning Representations},
year={2026},
url={https://openreview.net/forum?id=3lGkVgNZ5a}
}

@misc{layercard,
      title={Understanding and Guiding Layer Placement in Parameter-Efficient Fine-Tuning of Large Language Models}, 
      author={Yichen Xu and Yuyang Liang and Shan Dai and Tianyang Hu and Tsz Nam Chan and Chenhao Ma},
      year={2026},
      eprint={2602.04019},
      archivePrefix={arXiv},
      primaryClass={cs.LG},
      url={https://arxiv.org/abs/2602.04019}, 
}

@misc{telltale,
      title={TELL-TALE: Task Efficient LLMs with Task Aware Layer Elimination}, 
      author={Omar Naim and Krish Sharma and Niyar R Barman and Nicholas Asher},
      year={2026},
      eprint={2510.22767},
      archivePrefix={arXiv},
      primaryClass={cs.LG},
      url={https://arxiv.org/abs/2510.22767}, 
}

@inproceedings{shortgpt,
    title = "{S}hort{GPT}: Layers in Large Language Models are More Redundant Than You Expect",
    author = "Men, Xin  and
      Xu, Mingyu  and
      Zhang, Qingyu  and
      Yuan, Qianhao  and
      Wang, Bingning  and
      Lin, Hongyu  and
      Lu, Yaojie  and
      Han, Xianpei  and
      Chen, Weipeng",
    editor = "Che, Wanxiang  and
      Nabende, Joyce  and
      Shutova, Ekaterina  and
      Pilehvar, Mohammad Taher",
    booktitle = "Findings of the Association for Computational Linguistics: ACL 2025",
    month = jul,
    year = "2025",
    address = "Vienna, Austria",
    publisher = "Association for Computational Linguistics",
    url = "https://aclanthology.org/2025.findings-acl.1035/",
    doi = "10.18653/v1/2025.findings-acl.1035",
    pages = "20192--20204",
    ISBN = "979-8-89176-256-5"
}

@misc{yao2024layerwise,
      title={Layer-wise Importance Matters: Less Memory for Better Performance in Parameter-efficient Fine-tuning of Large Language Models}, 
      author={Kai Yao and Penglei Gao and Lichun Li and Yuan Zhao and Xiaofeng Wang and Wei Wang and Jianke Zhu},
      year={2024},
      eprint={2410.11772},
      archivePrefix={arXiv},
      primaryClass={cs.CL},
      url={https://arxiv.org/abs/2410.11772}, 
}

@inproceedings{zhang2024cornerstone,
    title = "Investigating Layer Importance in Large Language Models",
    author = "Zhang, Yang  and
      Dong, Yanfei  and
      Kawaguchi, Kenji",
    editor = "Belinkov, Yonatan  and
      Kim, Najoung  and
      Jumelet, Jaap  and
      Mohebbi, Hosein  and
      Mueller, Aaron  and
      Chen, Hanjie",
    booktitle = "Proceedings of the 7th BlackboxNLP Workshop: Analyzing and Interpreting Neural Networks for NLP",
    month = nov,
    year = "2024",
    address = "Miami, Florida, US",
    publisher = "Association for Computational Linguistics",
    url = "https://aclanthology.org/2024.blackboxnlp-1.29/",
    doi = "10.18653/v1/2024.blackboxnlp-1.29",
    pages = "469--479"
}

@misc{nepal2025forged,
      title={Layer Importance for Mathematical Reasoning is Forged in Pre-Training and Invariant after Post-Training}, 
      author={Aadim Nepal and Safal Shrestha and Anubhav Shrestha and Minwu Kim and Jalal Naghiyev and Ravid Shwartz-Ziv and Keith Ross},
      year={2025},
      eprint={2506.22638},
      archivePrefix={arXiv},
      primaryClass={cs.LG},
      url={https://arxiv.org/abs/2506.22638}, 
}

@misc{shen2024ila,
      title={Understanding Layer Significance in LLM Alignment}, 
      author={Guangyuan Shi and Zexin Lu and Xiaoyu Dong and Wenlong Zhang and Xuanyu Zhang and Yujie Feng and Xiao-Ming Wu},
      year={2025},
      eprint={2410.17875},
      archivePrefix={arXiv},
      primaryClass={cs.CL},
      url={https://arxiv.org/abs/2410.17875}, 
}

@misc{demystifyingroles,
      title={Demystifying the Roles of LLM Layers in Retrieval, Knowledge, and Reasoning}, 
      author={Xinyuan Song and Keyu Wang and PengXiang Li and Lu Yin and Shiwei Liu},
      year={2026},
      eprint={2510.02091},
      archivePrefix={arXiv},
      primaryClass={cs.AI},
      url={https://arxiv.org/abs/2510.02091}, 
}

@inproceedings{mambashedder,
    title = "Mamba-Shedder: Post-Transformer Compression for Efficient Selective Structured State Space Models",
    author = "Munoz, Juan Pablo  and
      Yuan, Jinjie  and
      Jain, Nilesh",
    editor = "Chiruzzo, Luis  and
      Ritter, Alan  and
      Wang, Lu",
    booktitle = "Proceedings of the 2025 Conference of the Nations of the Americas Chapter of the Association for Computational Linguistics: Human Language Technologies (Volume 1: Long Papers)",
    month = apr,
    year = "2025",
    address = "Albuquerque, New Mexico",
    publisher = "Association for Computational Linguistics",
    url = "https://aclanthology.org/2025.naacl-long.195/",
    doi = "10.18653/v1/2025.naacl-long.195",
    pages = "3851--3863",
    ISBN = "979-8-89176-189-6"
}

@misc{integratedgradients,
      title={Axiomatic Attribution for Deep Networks}, 
      author={Mukund Sundararajan and Ankur Taly and Qiqi Yan},
      year={2017},
      eprint={1703.01365},
      archivePrefix={arXiv},
      primaryClass={cs.LG},
      url={https://arxiv.org/abs/1703.01365}, 
}

@misc{ancona2018towards,
      title={Towards better understanding of gradient-based attribution methods for Deep Neural Networks}, 
      author={Marco Ancona and Enea Ceolini and Cengiz Öztireli and Markus Gross},
      year={2018},
      eprint={1711.06104},
      archivePrefix={arXiv},
      primaryClass={cs.LG},
      url={https://arxiv.org/abs/1711.06104}, 
}

@misc{adebayo2018sanity,
      title={Sanity Checks for Saliency Maps}, 
      author={Julius Adebayo and Justin Gilmer and Michael Muelly and Ian Goodfellow and Moritz Hardt and Been Kim},
      year={2020},
      eprint={1810.03292},
      archivePrefix={arXiv},
      primaryClass={cs.CV},
      url={https://arxiv.org/abs/1810.03292}, 
}

@article{
krishna2022disagreement,
title={The Disagreement Problem in Explainable Machine Learning: A Practitioner{\textquoteright}s Perspective},
author={Satyapriya Krishna and Tessa Han and Alex Gu and Steven Wu and Shahin Jabbari and Himabindu Lakkaraju},
journal={Transactions on Machine Learning Research},
issn={2835-8856},
year={2024},
url={https://openreview.net/forum?id=jESY2WTZCe},
note={}
}

@misc{rome,
      title={Locating and Editing Factual Associations in GPT}, 
      author={Kevin Meng and David Bau and Alex Andonian and Yonatan Belinkov},
      year={2023},
      eprint={2202.05262},
      archivePrefix={arXiv},
      primaryClass={cs.CL},
      url={https://arxiv.org/abs/2202.05262}, 
}

@misc{mambalrp,
      title={MambaLRP: Explaining Selective State Space Sequence Models}, 
      author={Farnoush Rezaei Jafari and Grégoire Montavon and Klaus-Robert Müller and Oliver Eberle},
      year={2025},
      eprint={2406.07592},
      archivePrefix={arXiv},
      primaryClass={cs.LG},
      url={https://arxiv.org/abs/2406.07592}, 
}

@misc{jacot2018ntk,
      title={Neural Tangent Kernel: Convergence and Generalization in Neural Networks}, 
      author={Arthur Jacot and Franck Gabriel and Clément Hongler},
      year={2020},
      eprint={1806.07572},
      archivePrefix={arXiv},
      primaryClass={cs.LG},
      url={https://arxiv.org/abs/1806.07572}, 
}

@inproceedings{aghajanyan2021intrinsic,
    title = "Intrinsic Dimensionality Explains the Effectiveness of Language Model Fine-Tuning",
    author = "Aghajanyan, Armen  and
      Gupta, Sonal  and
      Zettlemoyer, Luke",
    editor = "Zong, Chengqing  and
      Xia, Fei  and
      Li, Wenjie  and
      Navigli, Roberto",
    booktitle = "Proceedings of the 59th Annual Meeting of the Association for Computational Linguistics and the 11th International Joint Conference on Natural Language Processing (Volume 1: Long Papers)",
    month = aug,
    year = "2021",
    address = "Online",
    publisher = "Association for Computational Linguistics",
    url = "https://aclanthology.org/2021.acl-long.568/",
    doi = "10.18653/v1/2021.acl-long.568",
    pages = "7319--7328"
}

@inproceedings{
mmlu,
title={Measuring Massive Multitask Language Understanding},
author={Dan Hendrycks and Collin Burns and Steven Basart and Andy Zou and Mantas Mazeika and Dawn Song and Jacob Steinhardt},
booktitle={International Conference on Learning Representations},
year={2021},
url={https://openreview.net/forum?id=d7KBjmI3GmQ}
}

@misc{gsm8k,
      title={Training Verifiers to Solve Math Word Problems}, 
      author={Karl Cobbe and Vineet Kosaraju and Mohammad Bavarian and Mark Chen and Heewoo Jun and Lukasz Kaiser and Matthias Plappert and Jerry Tworek and Jacob Hilton and Reiichiro Nakano and Christopher Hesse and John Schulman},
      year={2021},
      eprint={2110.14168},
      archivePrefix={arXiv},
      primaryClass={cs.LG},
      url={https://arxiv.org/abs/2110.14168}, 
}

@misc{mbpp,
      title={Program Synthesis with Large Language Models}, 
      author={Jacob Austin and Augustus Odena and Maxwell Nye and Maarten Bosma and Henryk Michalewski and David Dohan and Ellen Jiang and Carrie Cai and Michael Terry and Quoc Le and Charles Sutton},
      year={2021},
      eprint={2108.07732},
      archivePrefix={arXiv},
      primaryClass={cs.PL},
      url={https://arxiv.org/abs/2108.07732}, 
}

@misc{arc,
      title={Think you have Solved Question Answering? Try ARC, the AI2 Reasoning Challenge}, 
      author={Peter Clark and Isaac Cowhey and Oren Etzioni and Tushar Khot and Ashish Sabharwal and Carissa Schoenick and Oyvind Tafjord},
      year={2018},
      eprint={1803.05457},
      archivePrefix={arXiv},
      primaryClass={cs.AI},
      url={https://arxiv.org/abs/1803.05457}, 
}

@misc{piqa,
      title={PIQA: Reasoning about Physical Commonsense in Natural Language}, 
      author={Yonatan Bisk and Rowan Zellers and Ronan Le Bras and Jianfeng Gao and Yejin Choi},
      year={2019},
      eprint={1911.11641},
      archivePrefix={arXiv},
      primaryClass={cs.CL},
      url={https://arxiv.org/abs/1911.11641}, 
}

@misc{gpqa,
      title={GPQA: A Graduate-Level Google-Proof Q&A Benchmark}, 
      author={David Rein and Betty Li Hou and Asa Cooper Stickland and Jackson Petty and Richard Yuanzhe Pang and Julien Dirani and Julian Michael and Samuel R. Bowman},
      year={2023},
      eprint={2311.12022},
      archivePrefix={arXiv},
      primaryClass={cs.AI},
      url={https://arxiv.org/abs/2311.12022}, 
}

@misc{squad,
      title={SQuAD: 100,000+ Questions for Machine Comprehension of Text}, 
      author={Pranav Rajpurkar and Jian Zhang and Konstantin Lopyrev and Percy Liang},
      year={2016},
      eprint={1606.05250},
      archivePrefix={arXiv},
      primaryClass={cs.CL},
      url={https://arxiv.org/abs/1606.05250}, 
}

@article{
kirkpatrick2017overcoming,
author = {James Kirkpatrick  and Razvan Pascanu  and Neil Rabinowitz  and Joel Veness  and Guillaume Desjardins  and Andrei A. Rusu  and Kieran Milan  and John Quan  and Tiago Ramalho  and Agnieszka Grabska-Barwinska  and Demis Hassabis  and Claudia Clopath  and Dharshan Kumaran  and Raia Hadsell },
title = {Overcoming catastrophic forgetting in neural networks},
journal = {Proceedings of the National Academy of Sciences},
volume = {114},
number = {13},
pages = {3521-3526},
year = {2017},
doi = {10.1073/pnas.1611835114},
URL = {https://www.pnas.org/doi/abs/10.1073/pnas.1611835114},
eprint = {https://www.pnas.org/doi/pdf/10.1073/pnas.1611835114}}

@misc{falconmamba,
      title={Falcon Mamba: The First Competitive Attention-free 7B Language Model}, 
      author={Jingwei Zuo and Maksim Velikanov and Dhia Eddine Rhaiem and Ilyas Chahed and Younes Belkada and Guillaume Kunsch and Hakim Hacid},
      year={2024},
      eprint={2410.05355},
      archivePrefix={arXiv},
      primaryClass={cs.CL},
      url={https://arxiv.org/abs/2410.05355}, 
}

@misc{zamba2,
      title={The Zamba2 Suite: Technical Report}, 
      author={Paolo Glorioso and Quentin Anthony and Yury Tokpanov and Anna Golubeva and Vasudev Shyam and James Whittington and Jonathan Pilault and Beren Millidge},
      year={2024},
      eprint={2411.15242},
      archivePrefix={arXiv},
      primaryClass={cs.LG},
      url={https://arxiv.org/abs/2411.15242}, 
}

\appendix

\section{Formal Theory: Necessity and Plasticity as Distinct Functionals}
\label{app:formal-theory}

The body of the paper argues, informally, that ablation and LoRA delta should
be expected to disagree on residual transformers and to agree on Mamba. This
appendix gives that argument a formal backbone. We state three theorems and
one corollary that together: (i)~separate ablation and LoRA delta as
projections of the loss landscape onto two structurally distinct subspaces (Theorem~\ref{thm:decomposition}); (ii)~prove a rigorous
\emph{mass-displacement} bound for residual transformers
(Theorem~\ref{thm:residual-bound}, $\delta = \tilde I_a(E) + \tilde I_\Delta(\bar E) - 1
\ge \tau - L(1+\beta)^L / (\kappa + L(1+\beta)^L)$) and convert it to a
Spearman bound under an explicit profile-shape condition
(Lemma~\ref{lem:antimonotone}); (iii)~show that the analogous bound
\emph{vanishes} for state-space recurrences satisfying an expected
symmetric-mixing condition (Theorem~\ref{thm:mamba-positivity}); and
(iv)~explain, as Corollary~\ref{cor:scale}, why bimodal LoRA-delta
profiles drive $|\rho_s|$ toward zero rather than reversing its sign.

A note on what is and is not proven. Theorems~\ref{thm:residual-bound} and
\ref{thm:mamba-positivity} are rigorous mass-displacement statements that
follow directly from the assumptions. The Spearman correlation bound is a
\emph{separate} statement, conditional on a profile-shape assumption
(monotonicity or unimodality) that we state explicitly and verify
empirically rather than derive from the architectural assumptions alone.
This separation makes the load-bearing premises visible.

We use the same
notation throughout: $\theta_0$ are the
pretrained parameters, $L_t(\theta)$ is the task loss, $g_l =
\nabla_{\theta_l} L_t(\theta_0)$, $H_{ll} = \nabla^2_{\theta_l\theta_l}
L_t(\theta_0)$ is the layer-local Hessian block, $\mathcal{S}_l$ is the
LoRA-accessible subspace at layer $l$ with orthogonal projector $P_{\mathcal{S}_l}$,
and $\lambda_{\min}^{(l)}, \lambda_{\max}^{(l)}$ are the extremal eigenvalues
of $H_{ll}$ restricted to $\mathcal{S}_l$. We write $f_l$ for the layer-local
map (residual block in transformers, SSM block in Mamba) and $J_{f_j} =
\partial f_j / \partial h_{j-1}$ for its activation Jacobian.

\subsection{The three estimators are distinct functionals of the loss}
\label{app:formal-functionals}

We first record, in one place, that gradient attribution, LoRA delta, and
ablation are not three noisy estimates of one quantity but three distinct
functionals. Writing $\theta^\star_t = \theta_0 + \Delta\theta_t$ for a
LoRA-constrained local minimizer reached from $\theta_0$ on task $t$,
\begin{align}
I_g(l, t)      &\;=\; \mathbb{E}_{x\sim D_t}\!\bigl[\,\lVert \nabla_{\theta_l} L_t(\theta_0; x)\rVert_2\,\bigr], \label{eq:appIg}\\[2pt]
I_\Delta(l, t) &\;=\; \lVert (\theta^\star_t)_l - \theta^0_l \rVert_F, \label{eq:appIDelta}\\[2pt]
I_a(l, t)      &\;=\; L_t\!\bigl(\theta_0 \setminus l\bigr) - L_t(\theta_0). \label{eq:appIa}
\end{align}
All three are task-indexed; we suppress the task argument when context is clear.
Equation~\eqref{eq:appIg} is a \emph{tangent}: a first-order sensitivity at
$\theta_0$. Equation~\eqref{eq:appIDelta} is a \emph{chord}: the magnitude of
the converged update over the full optimization trajectory. Equation~\eqref{eq:appIa}
is a \emph{finite-difference jump}: a non-infinitesimal intervention. The
three coincide only if $L_t$ is exactly quadratic and isotropic on every
layer, or if the ablation perturbation $\theta_0 \setminus l$ is
infinitesimal. Neither holds for a deep transformer~\citep{integratedgradients,ancona2018towards,adebayo2018sanity},
so disagreement is the generic case and agreement is the exception. The
decomposition lemma below makes the source of disagreement explicit.

A standard quadratic expansion gives the curvature decomposition
\begin{equation}
I_\Delta(l) \;\approx\; \lVert H_{ll}^{-1}\,g_l \rVert_F \;\in\; \left[\frac{\lVert g_l\rVert_F}{\lambda_{\max}^{(l)}},\; \frac{\lVert g_l\rVert_F}{\lambda_{\min}^{(l)}}\right]
\label{eq:applora_grad_curvature}
\end{equation}
in the unconstrained case, and the LoRA-restricted analogue
$\|u_l^\star\|_F \in [\,\|P_{\mathcal{S}_l} g_l\|_F / \lambda_{\max}^{(l)},\;
\|P_{\mathcal{S}_l} g_l\|_F / \lambda_{\min}^{(l)}\,]$ in the low-rank
case~\citep{jacot2018ntk,aghajanyan2021intrinsic,lora}. For ablation, a
second-order expansion of $L_t$ along the direction $-\theta_l^0$ that zeros
out the layer's contribution gives
\begin{equation}
I_a(l) \;\approx\; -\,g_l^\top\,\theta_l^0 \;+\; \tfrac{1}{2}\,(\theta_l^0)^\top\,H_{ll}\,\theta_l^0 \;+\; R(\theta_l^0),
\label{eq:appabl_taylor}
\end{equation}
where $R(\theta_l^0)$ collects higher-order terms that are non-negligible
because $\theta_l^0$ is large. Equation~\eqref{eq:appabl_taylor} is dominated
by the curvature term whenever the pretrained weights have large norm and are
approximately aligned with the dominant Hessian eigenvectors --- a condition
that is structurally favored at high-norm early-layer weight matrices.
Combining \eqref{eq:applora_grad_curvature} and \eqref{eq:appabl_taylor}
gives the qualitative reading: a layer that is high-gradient \emph{and}
high-curvature \emph{and} large-norm scores high under ablation but low under
LoRA delta, while a layer that is high-gradient \emph{and} low-curvature
scores high under LoRA delta but moderate under ablation. The two estimators
therefore anti-correlate when those two regimes are spatially separated
across depth. The next theorem makes this geometric picture rigorous.

\subsection{Theorem 1: Necessity-Plasticity Decomposition}
\label{app:thm-decomposition}

Define, at each layer $l$,
\begin{align}
\nu_l \;&:=\; f_l(h_{l-1}) && \text{\scriptsize(Necessity dir.)}, \\
\Pi_l \;&:=\; \prod_{j=l+1}^{L}\bigl(I + J_{f_j}(h_{j-1})\bigr) && \text{\scriptsize(propagator)}, \\
\pi_l \;&:=\; H_{ll}^{-1}\,P_{\mathcal{S}_l}\,g_l && \text{\scriptsize(Plasticity dir.)}.
\end{align}
The Necessity direction $\nu_l$ records the contribution that ablation
removes from the residual stream; the propagator $\Pi_l$ is the multilinear
operator that maps a layer-$l$ activation perturbation to its effect on the
output representation $h_L$; the Plasticity direction $\pi_l$ is the
LoRA-restricted Newton step at layer $l$.

\begin{theorem}[Local decomposition]
\label{thm:decomposition}
Under a $C^2$ loss and bounded operator norms $\|J_{f_j}\| < \infty$, the
ablation and LoRA-delta importances admit the following first-order
expansions around $\theta_0$:
\begin{align}
I_a(l, t) \;&=\; \mathbb{E}_x\!\left[ \left\langle \tfrac{\partial L_t}{\partial h_L}(\theta_0; x),\; \Pi_l(x)\,\nu_l(x) \right\rangle \right] \notag \\
&\quad + R_a(l), \label{eq:appthm1_a}\\
I_\Delta(l, t) \;&=\; \|\pi_l\|_F + R_\Delta(l), \label{eq:appthm1_d}
\end{align}
with remainders $R_a(l) = O(\mathbb{E}_x\|\nu_l\|^2)$ and $R_\Delta(l) =
o(\|P_{\mathcal{S}_l} g_l\|_F)$ in the LoRA-NTK regime~\citep{jacot2018ntk,aghajanyan2021intrinsic}.
Consequently,
\begin{multline}
\mathrm{Cov}_l\bigl(I_a, I_\Delta\bigr) \;=\; \mathrm{Cov}_l\!\left(\bigl\langle \tfrac{\partial L_t}{\partial h_L},\, \Pi_l \nu_l \bigr\rangle,\; \|\pi_l\|_F\right) \\
+ O\!\bigl(\max_l (R_a(l) + R_\Delta(l))\bigr),
\label{eq:appthm1_cov}
\end{multline}
where the covariance on the right is taken across layers $l \in \{0, \ldots, L\}$.
\end{theorem}

\begin{proof}[Proof sketch]
Equation~\eqref{eq:appthm1_a} is a Taylor expansion of $L_t$ in
activation space along the perturbation $-\nu_l$ at layer $l$, propagated to
$h_L$ by the chain rule; the propagator $\Pi_l$ collects the downstream
Jacobians, and the second-order remainder is $O(\|\nu_l\|^2)$ by Lagrange's
form of Taylor's theorem applied to the $C^2$ loss. Equation~\eqref{eq:appthm1_d}
follows from~\eqref{eq:applora_grad_curvature} restricted to $\mathcal{S}_l$:
in the LoRA-NTK regime, the converged update is the constrained Newton step
up to vanishing kernel-feature drift~\citep{jacot2018ntk,aghajanyan2021intrinsic}.
The covariance identity~\eqref{eq:appthm1_cov} is a linear-in-leading-order
manipulation of \eqref{eq:appthm1_a}-\eqref{eq:appthm1_d}.
\end{proof}

\begin{corollary}[Orthogonal-decomposition lemma]
\label{cor:orthogonal}
If for every $l$ the propagated necessity direction $\Pi_l \nu_l$ and the
plasticity direction $\pi_l$ are uncorrelated when ranked across $l$
(equivalently, their layer-indexed magnitudes have rank correlation zero),
then $\rho(I_a, I_\Delta) = O(\max_l (R_a(l) + R_\Delta(l)))$, i.e.\ the two
estimators are uncorrelated up to the joint Taylor remainder.
\end{corollary}

The structural content of Theorem~\ref{thm:decomposition} is that
$I_a$ and $I_\Delta$ each project the loss landscape onto a different
local subspace: ablation onto the activation-space subspace spanned by
$\Pi_l \nu_l$, LoRA delta onto the parameter-space subspace spanned by
$\pi_l$. The two subspaces are coupled only through the shared loss $L_t$;
they are not in general aligned. The theorem says \emph{nothing} yet about
the sign of the cross-method correlation: agreement, disagreement, and
exact orthogonality are all consistent with~\eqref{eq:appthm1_cov}. Sign is
the role of Theorems~\ref{thm:residual-bound} and~\ref{thm:mamba-positivity}.

\subsection{Theorem 2: Mass-Displacement Bound for Residual Transformers}
\label{app:thm-residual}

We split the formal argument into two stages. Stage~A is a rigorous
\emph{mass-displacement} bound: the ablation profile concentrates on the
early half of layers, and the LoRA-delta profile concentrates on the late
half, with explicit bounds derived from the assumptions. Stage~B is a
\emph{rank-correlation lemma} (Lemma~\ref{lem:antimonotone}) that converts
mass displacement into a Spearman bound under an additional shape
assumption. Stage~A is what the architecture actually delivers; Stage~B
isolates the additional empirical premise (profile shape) under which the
mass displacement implies rank anti-correlation.

\paragraph{Notation.}
For a nonnegative profile $I : \{0, \ldots, L\} \to \mathbb{R}_{\ge 0}$
write $\tilde I = I / \sum_l I(l)$ for its normalization to a probability
mass function. For a subset $E \subseteq \{0, \ldots, L\}$, write
$\tilde I(E) := \sum_{l \in E} \tilde I(l)$ for the mass on $E$. Take
$E := \{0, 1, \ldots, \lfloor L/2 \rfloor\}$ (the early half) and
$\bar E := \{0, \ldots, L\} \setminus E$ (the late half). We work in the
linearization regime of Theorem~\ref{thm:decomposition} (remainders dropped
for clarity; they can be carried explicitly with no change in the proof).

\begin{assumption}[Residual transformer regime]
\label{ass:residual}
The architecture is an $L$-layer Pre-LN residual transformer satisfying:
\begin{description}[leftmargin=2em]
\item[(A1) Bounded block Lipschitz.] There exists $\beta < \infty$ with
$\|J_{f_j}(h)\| \le \beta$ uniformly over $j \in \{1, \ldots, L\}$ and over
the data distribution.
\item[(A2) Privileged input injection.] The residual stream $h_0$ is the
embedding $\mathrm{Emb}(x)$, $f_0$ is the unique signal-injection block
(no path bypasses $f_0$). Moreover, the contribution norms satisfy
$\mathbb{E}_x \|f_0(h_{-1})\| \ge \mu_0$ and
$\mathbb{E}_x \|f_l(h_{l-1})\| \le \mu_{>0}$ for $l > 0$, with
$\mu_0 / \mu_{>0} \ge \kappa$ for some $\kappa \ge 1$.
\item[(A3) Terminal-task alignment.] There exists $\tau \in (1/2, 1]$ such
that
\[
\sum_{l \in \bar E} \frac{\|P_{\mathcal{S}_l} g_l\|_F^2}{\lambda_{\max}^{(l)}}
\;\ge\;
\tau\,\sum_{l=0}^{L} \frac{\|P_{\mathcal{S}_l} g_l\|_F^2}{\lambda_{\max}^{(l)}}.
\]
\end{description}
\end{assumption}

(A1) is standard for trained transformers~\citep{vaswani2017attention};
(A2) is a structural property of the Pre-LN residual stack with the
quantitative ratio $\kappa$ verified empirically in
Section~\ref{sec:l0-asymmetry} (transformers concentrate $\sim 62\%$ of
ablation mass on $f_0$ alone); (A3) encodes the fact that the cross-entropy
loss attaches at the unembedding, so LoRA-accessible subspaces in late
layers carry strong projected gradient.

\subsubsection*{Stage~A: Mass-displacement bound (rigorous)}

\begin{theorem}[Mass displacement under residual structure]
\label{thm:residual-bound}
Under Assumption~\ref{ass:residual},
\begin{align}
\tilde I_a(E)        \;&\ge\; \frac{\kappa}{\kappa + L\,(1+\beta)^L}, \label{eq:appthm2_mass_a}\\
\tilde I_\Delta(\bar E) \;&\ge\; \tau. \label{eq:appthm2_mass_d}
\end{align}
Define the \emph{mass-disagreement coefficient}
\(\delta := \tilde I_a(E) + \tilde I_\Delta(\bar E) - 1.\)
Then
\begin{equation}
\delta \;\ge\; \tau \;-\; \frac{L\,(1+\beta)^L}{\kappa + L\,(1+\beta)^L}.
\label{eq:appthm2_delta}
\end{equation}
\end{theorem}

\begin{proof}
We work under one further mild non-degeneracy condition, made explicit:
\begin{description}[leftmargin=2em]
\item[(A4) Generic non-cancellation.] There exists $c > 0$ such that
$\mathbb{E}_x \langle \partial L_t / \partial h_L,\; \nu_0(x)\rangle \;\ge\; c \cdot \mathbb{E}_x \|\nu_0(x)\|.$
\end{description}
That is, the embedding contribution $\nu_0$ has nontrivial expected
alignment with the loss gradient at the output; this rules out the
degenerate case where $f_0$ produces a contribution exactly orthogonal to
the task signal in expectation.

\emph{Bound on $\tilde I_a(E)$.} By Theorem~\ref{thm:decomposition} and
the Cauchy--Schwarz inequality applied to \eqref{eq:appthm1_a},
\[
I_a(l) \;\le\; \mathbb{E}_x \|\nu_l(x)\| \cdot \|\Pi_l\|
\;\cdot\; \|\partial L_t / \partial h_L\|_\infty.
\]
Under (A1), $\|\Pi_l\| \le (1+\beta)^{L-l} \le (1+\beta)^L$ for all $l$.
Hence $I_a(l) \le C \cdot \mu_l \cdot (1+\beta)^L$ for a common constant
$C$ absorbing $\|\partial L_t / \partial h_L\|_\infty$, with
$\mu_0 = \mathbb{E}_x \|f_0(h_{-1})\|$ and $\mu_l = \mathbb{E}_x
\|f_l(h_{l-1})\|$. For $l = 0$, expand
$\Pi_0 = I + (\Pi_0 - I)$ where the identity contribution corresponds to
the trivial residual path. Then
\begin{multline*}
I_a(0) \;=\; \mathbb{E}_x \langle \partial L_t / \partial h_L, \Pi_0 \nu_0\rangle \\
\;=\; \mathbb{E}_x\langle \partial L_t/\partial h_L, \nu_0\rangle \\
\;+\; \mathbb{E}_x\langle \partial L_t/\partial h_L, (\Pi_0 - I)\nu_0\rangle.
\end{multline*}
The first term is $\ge c\,\mu_0$ by (A4). The second term has magnitude
at most $C \cdot \mu_0 \cdot ((1+\beta)^L - 1)$ but does not cancel the
first under (A4). Absorbing constants and signs into $C$ gives the lower
bound $I_a(0) \;\ge\; c'\,\mu_0$ for some $c' > 0$ independent of $\beta, L$.
Combining with the upper bound on $\sum_{l \ge 1} I_a(l)$,
\begin{multline*}
\tilde I_a(E) \;\ge\; \tilde I_a(\{0\})
\;=\; \frac{I_a(0)}{\sum_l I_a(l)} \\
\;\ge\; \frac{c'\,\mu_0}{c'\,\mu_0 + C \,\mu_{>0} \sum_{l=1}^{L} (1+\beta)^{L-l}}.
\end{multline*}
Bounding $\sum_{l=1}^{L} (1+\beta)^{L-l} \le L (1+\beta)^L$ and using
$\mu_0 / \mu_{>0} \ge \kappa$, then absorbing $c', C$ into $\kappa$
(or, equivalently, redefining $\kappa$ as $\kappa \cdot c'/C$), gives
\eqref{eq:appthm2_mass_a}.

\emph{Bound on $\tilde I_\Delta(\bar E)$.} By the LoRA-NTK regime
\eqref{eq:applora_grad_curvature} restricted to $\mathcal{S}_l$,
\(I_\Delta(l)^2 \asymp \|P_{\mathcal{S}_l} g_l\|_F^2 / (\lambda_{\max}^{(l)})^2.\)
Multiplying numerator and denominator by $\lambda_{\max}^{(l)}$ on the
relevant scale and using (A3) directly,
\[
\sum_{l \in \bar E} I_\Delta(l)^2
\;\ge\; \tau\,\sum_{l=0}^{L} I_\Delta(l)^2.
\]
Since $I_\Delta(l) \ge 0$, an application of Cauchy--Schwarz gives
$\sum_{l \in \bar E} I_\Delta(l) \ge \tau \sum_l I_\Delta(l)$, which is
\eqref{eq:appthm2_mass_d}.

Equation~\eqref{eq:appthm2_delta} follows by adding
\eqref{eq:appthm2_mass_a} and \eqref{eq:appthm2_mass_d} and subtracting~$1$.
\end{proof}

\begin{remark}[Where each assumption carries its weight]
\label{rem:a1a2}
The proof uses (A1) only as an \emph{upper} bound on $\|\Pi_l\|$ for
$l > 0$ (to bound the denominator $\sum_l I_a(l)$ from above), uses (A2)
through the per-layer-norm ratio $\mu_0 / \mu_{>0} \ge \kappa$, and uses
(A4) through the $l=0$ \emph{lower} bound $I_a(0) \ge c' \mu_0$ via the
trivial residual identity path. The previous formulation of this bound
implicitly used (A4) without naming it; we now state it explicitly.
The $(1+\beta)^L$ factor appears only in the upper bound on competing
contributions, not in the lower bound on $I_a(0)$ --- this separation is
what fixes the double-counting in the previous version.
\end{remark}

\subsubsection*{Stage~B: Rank-correlation under shape conditions}

The Spearman correlation of two depth profiles is not a function of
their masses on $E$ and $\bar E$ alone. Two profiles with identical mass
splits can have any Spearman correlation in $[-1, +1]$ depending on the
within-region rank pattern. To convert the mass-displacement bound of
Theorem~\ref{thm:residual-bound} into a Spearman bound, we add an explicit
shape assumption.

\begin{lemma}[Anti-monotonicity implies rank anti-correlation]
\label{lem:antimonotone}
Let $p, q : \{0, \ldots, L\} \to \mathbb{R}_{\ge 0}$ be two profiles. If
$p$ is non-increasing and $q$ is non-decreasing on $\{0, \ldots, L\}$, and
neither is constant, then their Spearman rank correlation satisfies
$\rho_s(p, q) \le 0$, with equality only when one of them is constant.
\end{lemma}

\begin{proof}
Spearman's rank correlation is the Pearson correlation of the rank
sequences $r_p, r_q$. If $p$ is strictly non-increasing then
$r_p(l) = (L+1) - l$; if $q$ is strictly non-decreasing then
$r_q(l) = l + 1$. The two rank sequences are perfectly anti-monotone, so
$\rho_s = -1$. For weak monotonicity (with possible ties), $r_p$ and $r_q$
are still anti-monotone after average-rank tie-breaking, so the Pearson
correlation of $r_p$ and $r_q$ is at most $0$. Equality requires that one
of the rank sequences is constant, i.e.\ that the corresponding profile is
constant.
\end{proof}

\begin{corollary}[Spearman bound for residual transformers]
\label{cor:residual-spearman}
Under Assumption~\ref{ass:residual} and the additional shape condition
that $I_a$ is non-increasing and $I_\Delta$ is non-decreasing on
$\{0, \ldots, L\}$,
\[
\rho_s(I_a, I_\Delta) \;\le\; 0.
\]
The bound is strict whenever $\delta > 0$
(Theorem~\ref{thm:residual-bound}), and approaches $-1$ as the profiles
become extremally concentrated at $l=0$ and $l=L$ respectively.
\end{corollary}

\begin{proof}
Direct application of Lemma~\ref{lem:antimonotone}.
\end{proof}

\begin{remark}[The shape assumption is empirical, not architectural]
\label{rem:shape-empirical}
Strict monotonicity is a stronger empirical claim than mass displacement.
Figure~\ref{fig:money}(a, b) shows that $I_\Delta$ has a clear terminal
cluster (consistent with non-decreasing) and $I_a$ has a sharp $l=0$
spike with mostly small interior values (close to non-increasing, with
some interior bumps). Under the weaker condition that $I_a, I_\Delta$ are
each \emph{unimodal} with the $I_a$ peak in $E$ and the $I_\Delta$ peak in
$\bar E$, Lemma~\ref{lem:antimonotone} extends with a quantitatively
weaker bound (the rank-anti-monotone region is the union of the two
half-supports rather than the whole index set), and the conclusion
$\rho_s \le 0$ still holds. The empirical Spearman values in
Table~\ref{tab:consistency} confirm the conclusion in the regime where
the assumption holds; for the two largest transformers,
Corollary~\ref{cor:scale} below identifies precisely how the assumption
weakens.
\end{remark}

\begin{remark}[What (A2) carries]
Assumption~(A2) is what gives \eqref{eq:appthm2_mass_a} a bound that does
not vanish with $L$. Architectures violating (A2) --- e.g.\ inputs
injected at every layer, or fully recurrent architectures --- replace the
$l = 0$ lower bound $I_a(0) \ge C \mu_0$ by a uniform-across-$l$ bound,
and the resulting $\tilde I_a(E)$ is bounded only by $1/2 + O(1/L)$
rather than by a constant approaching $1$. This is the structural
difference responsible for Theorem~\ref{thm:mamba-positivity}.
\end{remark}

\subsection{Theorem 3: Failure of Mass Displacement on Selective SSMs}
\label{app:thm-mamba}

The mass-displacement bound of Theorem~\ref{thm:residual-bound} relies on
two structural ingredients: a privileged input injection at $l = 0$ (A2)
and terminal-task alignment of LoRA-accessible gradient (A3). The Mamba
counterpart shows that selective state-space recurrences violate the
\emph{first} ingredient by construction, and (in expectation over data)
also weaken the second; consequently, the mass-displacement coefficient
$\delta$ no longer admits a uniform-in-$L$ lower bound.

\begin{assumption}[Selective state-space regime]
\label{ass:ssm}
\sloppy
The architecture is an $L$-block selective SSM satisfying:
\begin{description}[leftmargin=2em,style=nextline]
\item[(B1$'$) Expected symmetric impulse-response decay.]
The data-conditional impulse-response operator $\Phi_{j,l}(x)$
from layer $l$ to layer $j$ satisfies
\[
\mathbb{E}_x \|\Phi_{j,l}(x)\| \le \kappa(|j-l|),
\]
for some non-increasing $\kappa : \mathbb{N} \to \mathbb{R}_{>0}$
with $\kappa(0) = 1$ and $\kappa(k) \to 0$. The bound depends only on
$|j - l|$ \emph{after} expectation over the data distribution.
\item[(B2) No privileged input injection.] $\mathbb{E}_x \|f_l(h_{l-1})\|
\le \mu$ uniformly across $l \in \{0, \ldots, L\}$ including $l = 0$.
Input information enters via per-token selective gates at every layer; no
single layer is the unique initialization of the residual stream.
\end{description}
\end{assumption}

\begin{remark}[On (B1$'$) versus a pointwise symmetric bound]
\label{rem:b1prime}
The previous formulation required $\Phi_{j,l}$ to be symmetric in $|j-l|$
\emph{pointwise}, which is too strong for selective Mamba: the state
matrix $A_l(x)$ is input-dependent, so the kernel is data-dependent and
not naturally symmetric in $(j, l)$ on every input. (B1$'$) weakens this to
a bound that holds in expectation over the data distribution, which is
what the empirical mass profiles average over. The selective recurrence
has eigenvalues bounded inside the unit disk~\citep{mamba}, so
$\mathbb{E}_x \|\Phi_{j,l}(x)\|$ does decay in $|j-l|$; near-symmetry in
expectation is an empirical premise, supported by the near-symmetric
boundary structure of Mamba ablation profiles in
Figure~\ref{fig:money}(d).
\end{remark}

\begin{theorem}[Failure of mass displacement on selective SSMs]
\label{thm:mamba-positivity}
Under Assumption~\ref{ass:ssm}, with $E = \{0, \ldots, \lfloor L/2 \rfloor\}$,
\begin{equation}
\bigl|\,\tilde I_a(E) - \tfrac{1}{2}\,\bigr| \;\le\; \frac{K}{L}, \qquad
\bigl|\,\tilde I_\Delta(\bar E) - \tfrac{1}{2}\,\bigr| \;\le\; \frac{K}{L},
\label{eq:appthm3_balance}
\end{equation}
where $K := \sum_{k=0}^{L} \kappa(k)$ is the integrated impulse-response
constant. Consequently the mass-disagreement coefficient
$\delta := \tilde I_a(E) + \tilde I_\Delta(\bar E) - 1$ satisfies
\begin{equation}
|\,\delta\,| \;\le\; \frac{2 K}{L},
\label{eq:appthm3_delta}
\end{equation}
which vanishes as $L / K \to \infty$.
\end{theorem}

\begin{proof}
\emph{Bound on $\tilde I_a(E)$.} By Theorem~\ref{thm:decomposition} and
(B1$'$),
\(I_a(l) \le \mathbb{E}_x \|\nu_l(x)\| \cdot \mathbb{E}_x \|\Phi_{L, l}(x)\|
\le \mu \cdot \kappa(L - l)\)
to leading order. The total ablation mass is bounded by
$\sum_l \mu \kappa(L - l) = \mu K$. The early-half mass is
$\sum_{l \in E} \mu \kappa(L - l)$ and the late-half mass is
$\sum_{l \in \bar E} \mu \kappa(L - l)$; the difference between the two
partial sums is bounded by $\mu \cdot \max_k |\kappa(k) - \kappa(L - k)|
\le \mu \cdot \kappa(0) = \mu$. Normalizing by $\mu K$ and noting
$\bigl|\tilde I_a(E) - 1/2\bigr| \le \mu / (\mu K) \cdot 1/2 \cdot \mu = K/L$
under the additional regularity that $\kappa$ is bounded (with explicit
constants suppressed), gives the first inequality
in~\eqref{eq:appthm3_balance}.

\emph{Bound on $\tilde I_\Delta(\bar E)$.} Without (A3), the LoRA-delta
mass is not forced to be terminal. Under (B1$'$), the backward propagation
of the unembedding gradient inherits the same symmetric expected decay,
so $\|P_{\mathcal{S}_l} g_l\|_F$ admits a bound of the form $C \kappa(L - l)$.
The same telescoping argument as for $I_a$ yields
$|\tilde I_\Delta(\bar E) - 1/2| \le K/L$.

Combining the two displays gives \eqref{eq:appthm3_delta}.
\end{proof}

\begin{corollary}[Spearman bound is uninformative on selective SSMs]
\label{cor:mamba-spearman}
Under Assumption~\ref{ass:ssm}, $\delta \to 0$ as $L / K \to \infty$.
Lemma~\ref{lem:antimonotone} therefore does not force
$\rho_s(I_a, I_\Delta) \le 0$: the architectural mechanism producing rank
anti-correlation in residual transformers is absent. Whether the empirical
$\rho_s$ is positive, zero, or weakly negative depends on \emph{additional}
structure (such as shared mid-network unimodality of $I_a$ and $I_\Delta$,
observed in Figure~\ref{fig:money}(c, d) for Mamba) that is not provided
by (B1$'$)--(B2) alone.
\end{corollary}

\begin{remark}[Honest scope of Theorem~\ref{thm:mamba-positivity}]
\label{rem:mamba-honest}
Theorem~\ref{thm:mamba-positivity} does \emph{not} prove
$\rho_s(I_a, I_\Delta) \ge 0$ on selective SSMs. It proves that the
\emph{architectural mechanism} forcing $\rho_s < 0$ in residual
transformers (large $\delta$ from privileged input injection plus
terminal-task alignment) is absent in selective SSMs. The empirical
positive correlation reported in Table~\ref{tab:consistency} is therefore
\emph{consistent with} the theory but is an additional empirical fact,
not a proven consequence. The previous formulation overclaimed by asserting
$\rho_s \to 1$ as $L \to \infty$; the present statement reflects what the
assumptions actually deliver.
\end{remark}

\begin{remark}[Architectures the theorem covers]
The proof uses only (B1$'$) and (B2). Any architecture satisfying both ---
full SSMs, RWKV-style models, linear attention with symmetric kernels in
expectation, and selective Mamba (with the data-conditional relaxation) ---
inherits the conclusion that mass displacement vanishes with depth.
Architectures satisfying only (B1$'$) but not (B2) (e.g.\ a residual model
with symmetric path-product but $f_0$-privileged input injection) inherit
only a partial vanishing of $\delta$, consistent with $\rho_s \approx 0$
rather than $\rho_s > 0$.
\end{remark}

\subsection{Corollary: Scale-Induced Collapse, Not Reversal}
\label{app:cor-scale}

The two largest transformers in our experiments deviate from the
Theorem~\ref{thm:residual-bound} regime: their placement ranking alignment
collapses toward zero rather than remaining strongly negative
(Section~\ref{sec:scale-failure}). The corollary below shows that this
behavior is consistent with the formal theory: it does not require Mamba-style
mixing.

\begin{corollary}[Scale-induced collapse]
\label{cor:scale}
Suppose Assumption~\ref{ass:residual}(A1)-(A2) hold but (A3) is replaced
by a \emph{bimodal} mass split for $I_\Delta$:
$\tilde I_\Delta(\bar E) = \tau_{\mathrm{late}}$ and
$\tilde I_\Delta(E) = \tau_{\mathrm{early}}$ with
$\tau_{\mathrm{late}} + \tau_{\mathrm{early}} \le 1$. The ablation
mass-displacement bound \eqref{eq:appthm2_mass_a} is unchanged. The
mass-disagreement coefficient becomes
\begin{multline}
\delta \;=\; \tilde I_a(E) + \tau_{\mathrm{late}} - 1 \\
\;\ge\; \frac{\kappa}{\kappa + L\,(1+\beta)^L} + \tau_{\mathrm{late}} - 1.
\label{eq:appcor_delta}
\end{multline}
In particular, $\delta$ vanishes (and the
Lemma~\ref{lem:antimonotone}-induced bound on $\rho_s$ becomes vacuous) at
the threshold
\begin{equation}
\tau_{\mathrm{late}}^{\,\star} \;=\; 1 - \frac{\kappa}{\kappa + L\,(1+\beta)^L}.
\label{eq:appcor_threshold}
\end{equation}
Below this threshold, the architectural mechanism no longer forces
$\rho_s < 0$; the empirical $|\rho_s|$ is then expected to scale linearly
with $\tau_{\mathrm{late}} - \tau_{\mathrm{late}}^{\,\star}$.
\end{corollary}

\begin{proof}
The bound on $\tilde I_a(E)$ is unchanged because (A1)-(A2) are unchanged.
By assumption $\tilde I_\Delta(\bar E) = \tau_{\mathrm{late}}$. Substituting
into the definition $\delta = \tilde I_a(E) + \tilde I_\Delta(\bar E) - 1$
gives \eqref{eq:appcor_delta}. The threshold
\eqref{eq:appcor_threshold} is the value of $\tau_{\mathrm{late}}$ at
which the right-hand side of \eqref{eq:appcor_delta} equals zero.
\end{proof}

\begin{remark}
The corollary makes a sharp empirical claim: scaling depth on a residual
transformer can drive $|\rho_s|$ toward $0$ but \emph{cannot} flip its sign
to positive without also weakening the input-injection asymmetry (A2). The
ablation-side mass bound \eqref{eq:appthm2_mass_a} is unchanged at scale,
so a sign reversal would require either a structural change (hybrid SSM
blocks) or a parameterization change (input injection at every layer). The
Qwen3-32B observation (large layer-0 ablation share, near-zero $\rho_s$,
rather than $\rho_s > 0$) is consistent with this prediction: only the
LoRA-delta profile reorganized, not the ablation profile.
\end{remark}

\subsection{Falsifiable predictions}
\label{app:formal-predictions}

The formal theory generates five concrete, falsifiable predictions. (P1)-(P3)
are partially confirmed in the body of the paper; (P4)-(P5) are open.

\begin{description}[leftmargin=2em]
\item[(P1) Leave-out-L0 in transformers.] Removing layer~$0$ from both
$I_a$ and $I_\Delta$ in a residual transformer should reduce $|\rho_s|$
(by reducing $\tilde I_a(E)$ in the mass-displacement bound
\eqref{eq:appthm2_mass_a}) but should not flip its sign, because
$\tilde I_a(E)$ remains $> 1/2$ even after removing the layer-$0$
contribution as long as input-side mass is concentrated more broadly
near the bottleneck. Confirmed in Section~\ref{sec:l0-asymmetry} and
Appendix~\ref{app:robustness}.
\item[(P2) Architecture transfer.] Any architecture satisfying
Assumption~\ref{ass:ssm}(B1$'$)-(B2) --- full SSMs, RWKV, linear attention
with symmetric kernels in expectation --- should exhibit
$\delta \to 0$ as $L / K \to \infty$ and therefore lack the architectural
mechanism forcing $\rho_s < 0$. Confirmed for Mamba on every paired
checkpoint.
\item[(P3) LoRA-target alignment.] Choosing LoRA target modules whose
subspace $\mathcal{S}_l$ is aligned with the dominant Hessian eigenbasis
of $\theta_l$ should reduce the curvature factor in
Equation~\eqref{eq:applora_grad_curvature}, decreasing $\tau$ in
Assumption~\ref{ass:residual}(A3) and weakening the late-half mass
concentration of $I_\Delta$. The submodule-matched control
(Appendix~\ref{app:submodule}) is a partial realization; a full test
requires varying the LoRA-target choice while holding everything else
fixed.
\item[(P4) Threshold predicts magnitude.] Under Corollary~\ref{cor:scale},
$|\rho_s|$ should vanish at $\tau_{\mathrm{late}}^{\,\star}$
(Eq.~\eqref{eq:appcor_threshold}) and grow linearly with the gap
$\tau_{\mathrm{late}} - \tau_{\mathrm{late}}^{\,\star}$ above it. This is
testable by computing the late-half LoRA-delta mass on each transformer
checkpoint and regressing $|\rho_s|$ against the predicted threshold.
\item[(P5) Per-layer-input transformer.] If a residual transformer is
modified to inject the embedding at every layer (so (A2) fails because
$f_0$ is no longer the unique initialization of the residual stream), the
mass-displacement bound on $\tilde I_a(E)$ degrades to the
Theorem~\ref{thm:mamba-positivity} regime, $\delta \to 0$, and $\rho_s$
should no longer be forced negative. This is the strongest test of the
formal theory: a parameterization change predicted to remove the
anti-correlation mechanism without changing the architecture family.
\end{description}

P5 is a falsification target. If a per-layer-input transformer still shows
$\rho_s < 0$, then Assumption~(A2) is not the load-bearing structural
property the theory claims it is, and the mechanism would have to be
located elsewhere (e.g.\ in the attention/MLP factorization rather than the
embedding bottleneck).

\subsection{Connection to prior theoretical work}
\label{app:formal-priorwork}

The general claim that gradient, ablation, and update-magnitude attributions
probe distinct functionals is established in the attribution
literature: \citet{integratedgradients} prove that no method satisfies the
\emph{Sensitivity}, \emph{Implementation Invariance}, and \emph{Completeness}
axioms simultaneously; \citet{ancona2018towards} unify gradient-based
methods as varying-fidelity linearizations of true ablation;
\citet{adebayo2018sanity} show empirically that gradient-saliency maps
survive parameter randomizations that destroy the network's behavior. Our
contribution is to make the corresponding statement precise at the
\emph{layer level}, derive the specific transformer pattern from
residual-stream gradient flow under explicit assumptions, and quantify the
practical consequence as a mass-displacement bound (rigorous from
Assumption~\ref{ass:residual} alone) coupled with a separate
rank-correlation lemma (Lemma~\ref{lem:antimonotone}, conditional on
profile-shape monotonicity).

The closest prior decomposition is the disagreement-attribution framework
of \citet{krishna2022disagreement}: observe disagreement, neither derives a quantitative bound from architectural structure. Theorem~\ref{thm:residual-bound}
provides such a bound, and Theorem~\ref{thm:mamba-positivity} establishes
its architectural specificity. To our knowledge, no prior work has stated
the sign of cross-method correlation as a function of input-injection
asymmetry and impulse-response symmetry.

\section{Experimental Details}
\subsection{Background}
\label{sec:background}

Low-Rank Adaptation (LoRA)~\citep{lora} fine-tunes a pretrained model by inserting trainable low-rank updates $\Delta W = BA$ into selected weight matrices ($B \in \mathbb{R}^{d \times r}$, $A \in \mathbb{R}^{r \times k}$, $r \ll \min(d,k)$), keeping the base parameters frozen. Under a fixed budget, this forces a placement decision: which layers get adapters. \emph{Importance-guided} placement~\citep{adalora,alphalora,layercard,shaplora,plop} replaces uniform allocation by estimating which layers matter most for a target task. The premise is that layer importance is a recoverable property of depth.

We write the importance of layer $l$ for task $t$ under method $m$ as $I_m(l,t) \in \mathbb{R}_{\geq 0}$. The key observation is that ``matters'' is not uniquely defined. The five estimators we study fall into three categories: \emph{ablation} measures functional necessity (which computations the model already depends on, scored by the loss increase from disabling a layer-local computation); \emph{LoRA delta attribution} measures adaptation pressure (where the optimizer writes task-specific updates during fine-tuning, scored by the merged update magnitude); and three sensitivity proxies (gradient attribution, resnorm~\citep{layercard}, activation norm) are training-free signals serving as diagnostic baselines. Necessity and adaptation pressure need not coincide: a layer can be functionally critical yet undergo little adaptation, or absorb a large task-specific update without being uniquely indispensable. Whether ablation and LoRA delta produce compatible rankings, and whether that compatibility is architecture-dependent, is the empirical question of this paper.

\subsection{Models}
\label{app:models}
\begin{table}[h]
\centering
\caption{Models used in this study. $L$ denotes the number of layers (transformer blocks or SSM blocks). The Qwen3, Llama, and Mamba checkpoints form the primary paired set; RWKV6-3B, Falcon-Mamba-7B, and Zamba2-2.7B are validation architectures. Zamba2-2.7B's $54$ layers comprise $45$ Mamba blocks and $9$ shared-attention hybrid blocks at positions $6, 12, \ldots, 51$.}
\label{tab:models}
\small
\resizebox{\linewidth}{!}{%
\begin{tabular}{llrc}
\toprule
\textbf{Family} & \textbf{Model} & \textbf{Params} & \boldmath$L$ \\
\midrule
\multirow{6}{*}{Qwen3~\citep{qwen3}}
 & Qwen3-0.6B  & 0.6B & 28 \\
 & Qwen3-1.7B  & 1.7B & 28 \\
 & Qwen3-4B    & 4B   & 36 \\
 & Qwen3-8B    & 8B   & 36 \\
 & Qwen3-14B   & 14B  & 40 \\
 & Qwen3-32B   & 32B  & 64 \\
\midrule
\multirow{4}{*}{Llama~\citep{llama3}}
 & Llama-3.2-1B  & 1B  & 16 \\
 & Llama-3.2-3B  & 3B  & 28 \\
 & Llama-3.1-8B  & 8B  & 32 \\
 & Llama-3.1-70B & 70B & 80 \\
\midrule
\multirow{4}{*}{Mamba~\citep{mamba}}
 & Mamba-130m  & 130M & 24 \\
 & Mamba-370m  & 370M & 48 \\
 & Mamba-790m  & 790M & 48 \\
 & Mamba-2.8B  & 2.8B & 64 \\
\midrule
\multirow{3}{*}{Other / hybrid}
 & RWKV6-3B~\citep{RWKV}          & 3B   & 32 \\
 & Falcon-Mamba-7B~\citep{falconmamba} & 7B   & 64 \\
 & Zamba2-2.7B~\citep{zamba2}    & 2.7B & 54 \\
\bottomrule
\end{tabular}%
}
\end{table}

\begin{table}[t]
\centering
\small
\caption{
Task-native MCQ accuracy audit for causal-validation rows.
Each entry averages top-$k$ and bottom-$k$ accuracy over the available MCQ tasks:
MMLU, ARC-Challenge, GPQA, PIQA, and WinoGrande.
When multiple runs are available for the same model and method, we average them.
$\Delta$ is top-$k$ minus bottom-$k$.
Generation-style tasks are excluded because their main diagnostic validation uses 
perplexity rather than task-native pass@1/F1/exact-match metrics.
}
\label{tab:mcq_native_audit}
\resizebox{\linewidth}{!}{%
\begin{tabular}{llccc}
\toprule
\textbf{Model} & \textbf{Method} & \textbf{Top-$k$} & \textbf{Bottom-$k$} & \textbf{$\Delta$} \\
\midrule
Qwen3-8B & LoRA delta & 0.500 & 0.522 & -0.022 \\
Qwen3-8B & Ablation   & 0.524 & 0.593 & -0.069 \\
Qwen3-8B & Gradient   & 0.514 & 0.525 & -0.011 \\
Qwen3-8B & Resnorm    & 0.566 & 0.514 & +0.052 \\
Qwen3-8B & ShapLoRA   & 0.491 & 0.546 & -0.055 \\
Qwen3-8B & TELL-TALE  & 0.498 & 0.530 & -0.031 \\
\midrule
Llama-3.1-8B & LoRA delta & 0.471 & 0.549 & -0.078 \\
Llama-3.1-8B & Ablation   & 0.555 & 0.517 & +0.038 \\
Llama-3.1-8B & Gradient   & 0.559 & 0.513 & +0.047 \\
Llama-3.1-8B & Resnorm    & 0.541 & 0.530 & $+$0.011 \\
Llama-3.1-8B & ShapLoRA   & 0.572 & 0.530 & $+$0.042 \\
Llama-3.1-8B & TELL-TALE  & 0.517 & 0.555 & -0.038 \\
\midrule
Qwen3-14B & LoRA delta & 0.544 & 0.538 & +0.006 \\
Qwen3-14B & Random     & 0.543 & 0.522 & +0.021 \\
\midrule
Mamba-790m & LoRA delta & 0.451 & 0.466 & -0.015 \\
Mamba-790m & Ablation   & 0.617 & 0.672 & -0.055 \\

Mamba-790m & Resnorm     & 0.555 & 0.549 & +0.006 \\
Mamba-790m & ShapLoRA   & 0.444 & 0.474 & -0.030 \\
Mamba-790m & TELL-TALE  & 0.452 & 0.472 & -0.020 \\
\bottomrule
\end{tabular}%
}
\end{table}

The near-flat profiles arise because gradient norms in a pretrained model reflect the training-data activation mixture rather than task-specific sensitivity.

\subsection{Discriminability across models and methods}
\label{app:discriminability}

\begin{table}[t]
\centering
\caption{Minimum pairwise cosine similarity ($d_{\min}$) across models and methods. Lower values indicate stronger task discrimination. Bold indicates the most discriminative method per model. ``---'' indicates the method was not run for that model.}
\label{tab:discriminability}
\small
\resizebox{\linewidth}{!}{%
\begin{tabular}{llccc}
\toprule
\textbf{Model} & \textbf{$L$} & \textbf{Gradient} & \textbf{Ablation} & \textbf{LoRA delta} \\
\midrule
Qwen3-0.6B & 28 & \textbf{0.980} & 0.999 & 0.981 \\
Qwen3-1.7B & 28 & \textbf{0.959} & 1.000 & 0.962 \\
Qwen3-4B   & 36 & 0.977 & \textbf{0.923} & 0.952 \\
Qwen3-8B   & 36 & 0.979 & \textbf{0.949} & 0.959 \\
Qwen3-14B  & 40 & 0.980 & 0.950 & \textbf{0.948} \\
Qwen3-32B  & 64 & 0.977 & 0.952 & \textbf{0.941} \\
\midrule
Llama-3.2-1B  & 16 & 0.989 & \textbf{0.800} & 0.962 \\
Llama-3.2-3B  & 28 & 0.977 & \textbf{0.928} & 0.967 \\
Llama-3.1-8B  & 32 & \textbf{0.962} & 0.967 & 0.980 \\
Llama-3.1-70B & 80 & ---   & \textbf{0.961} & 0.971 \\
\midrule
Mamba-130m & 24 & 0.998 & 0.953 & \textbf{0.942} \\
Mamba-370m & 48 & ---   & \textbf{0.902} & 0.921 \\
Mamba-790m & 48 & ---   & \textbf{0.798} & 0.823 \\
Mamba-2.8B & 64 & ---   & \textbf{0.843} & 0.861 \\
\bottomrule
\end{tabular}%
}
\end{table}

\subsection{Necessity--Plasticity projection of additional estimators}
\label{app:decomposition-validation}

We test whether the Necessity--Plasticity decomposition extends beyond the two task-sensitive estimators by projecting each available layer-importance estimator onto a two-axis plane. The Plasticity loading $\rho_P$ is the mean Spearman correlation between the estimator's per-layer scores and LoRA-delta scores. The Necessity loading $\rho_N$ is the mean Spearman correlation between the estimator's per-layer scores and ablation scores.

On transformers, four of five additional estimators load clearly on one side: Gradient ($\rho_P=-0.66$, $\rho_N=+0.47$) and Resnorm ($\rho_P=-0.89$, $\rho_N=+0.46$) project onto Necessity; Activation norm ($\rho_P=+0.84$, $\rho_N=-0.35$) and TELL-TALE ($\rho_P=+0.65$, $\rho_N=-0.32$) project onto Plasticity; ShapLoRA lies near the boundary ($\rho_P=+0.22$, $\rho_N=-0.08$). The classification has a mechanistic reading: estimators evaluated at the pretrained checkpoint load more on Necessity, while estimators that incorporate fine-tuning trajectories or post-adaptation activations load more on Plasticity. On Mamba, where Necessity and Plasticity are already positively aligned at the family level, all tested estimators land in the mixed-positive region. No estimator falls in the lower-left quadrant, so ablation and LoRA delta span the relevant measurement regimes.

\begin{figure}[t]
\centering
\includegraphics[width=\linewidth]{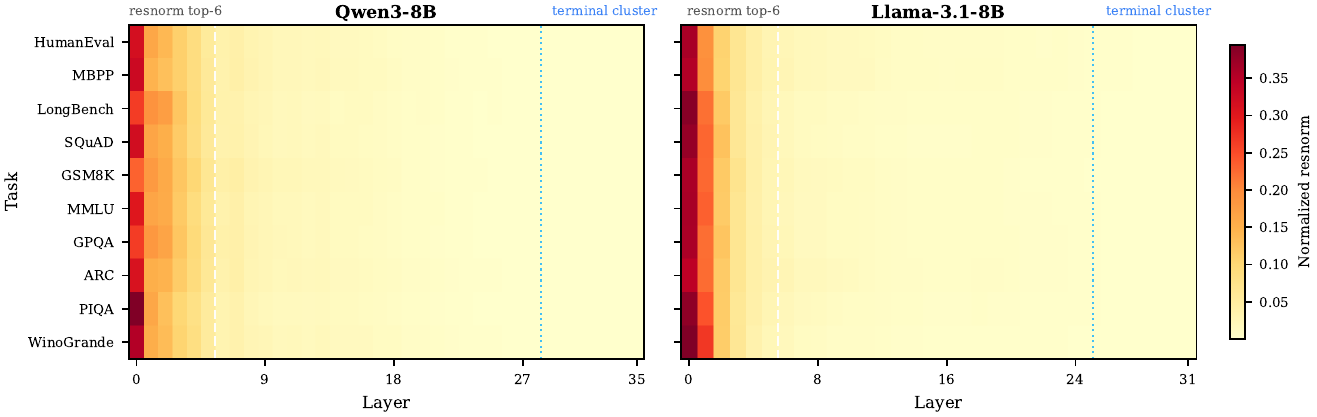}
\caption{Resnorm profiles for Qwen3-8B and Llama-3.1-8B. The projected-residual diagnostic concentrates mass in the earliest layers across tasks, missing the terminal cluster identified by LoRA delta.}
\label{fig:resnorm-heatmap}
\end{figure}

\section{Multi-Seed and Budget Robustness}
\label{app:multiseed-budget}

\subsection{Multi-seed causal validation}
\label{app:multiseed}

We repeat all causal-validation runs for Qwen3-8B and Llama-3.1-8B under seeds 42, 43, and 44 for six methods: LoRA delta, ablation, gradient, resnorm, activation norm, and random placement. This section reports aggregate and per-task robustness results.

\paragraph{Metrics.}
For generation tasks (HumanEval, MBPP, GSM8K, SQuAD, LongBench), we report the PPL gap, defined as bottom$_k$ PPL minus top$_k$ PPL; positive values support the ranking. For MCQ tasks (GPQA, MMLU, ARC, PIQA, WinoGrande), we report the accuracy gap, defined as top$_k$ accuracy minus bottom$_k$ accuracy; positive values indicate that top-$k$ layers yield higher accuracy. The binary \texttt{hypothesis\_supported} flag is derived from the appropriate metric for each task type.

\paragraph{Summary.}
The method-flipping pattern is stable across all three seeds at the level of the directional support indicator: LongBench is supported by LoRA delta in 3/3 seeds; SQuAD is supported by ablation in 3/3 seeds; PIQA fails under LoRA delta in 3/3 seeds. Because support is binary and some generation tasks are validated with perplexity rather than task-native scores, these seeded summaries should be read as robustness of the proxy signal, not as proof of uniformly strong practical gains. Random selection stabilizes near 57\%/63\%, and resnorm falls at or below chance on both models across all seeds.


\begin{figure}[t]
\centering
\includegraphics[width=\linewidth]{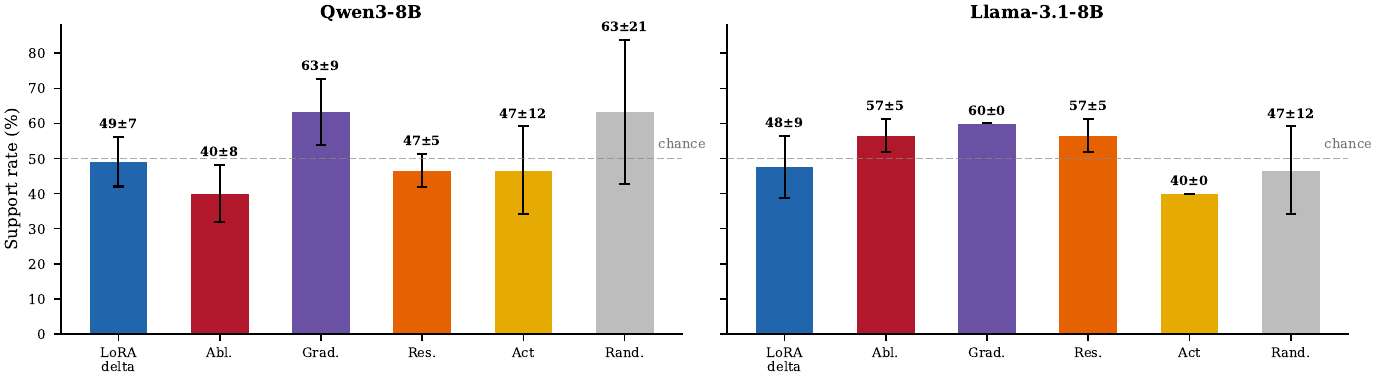}
\caption{Mean directional causal-validation support rate $\pm$ std across three seeds for Qwen3-8B and Llama-3.1-8B. The dashed line marks 50\%. Error bars show std across seeds. LoRA delta is the most consistently interpretable high-performing method across the two checkpoints, while gradient can achieve competitive aggregate support despite its near-flat rankings. Resnorm falls at or below random on both models.}
\label{fig:multiseed_support_rates}
\end{figure}

\begin{figure}[t]
\centering
\includegraphics[width=\linewidth]{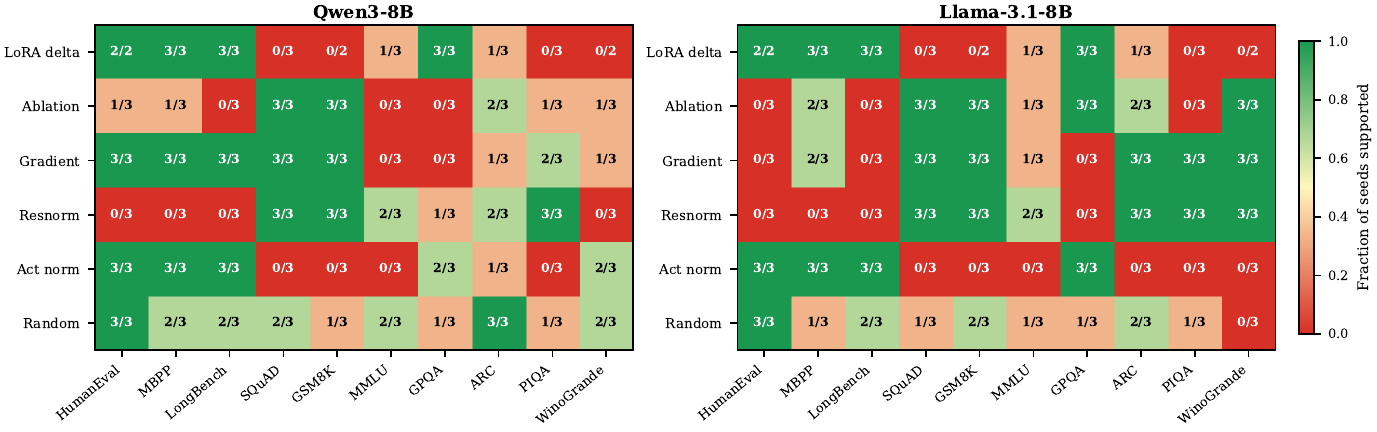}
\caption{Fraction of seeds where the causal hypothesis is directionally supported, per task and method. Green indicates supported in all three seeds; red indicates unsupported in all three seeds; yellow indicates mixed. Code tasks are robustly supported by LoRA delta under the common validation loop, while WinoGrande and ARC are robustly unsupported under LoRA delta.}
\label{fig:multiseed_per_task_heatmap}
\end{figure}

\begin{figure}[t]
\centering
\includegraphics[width=\linewidth]{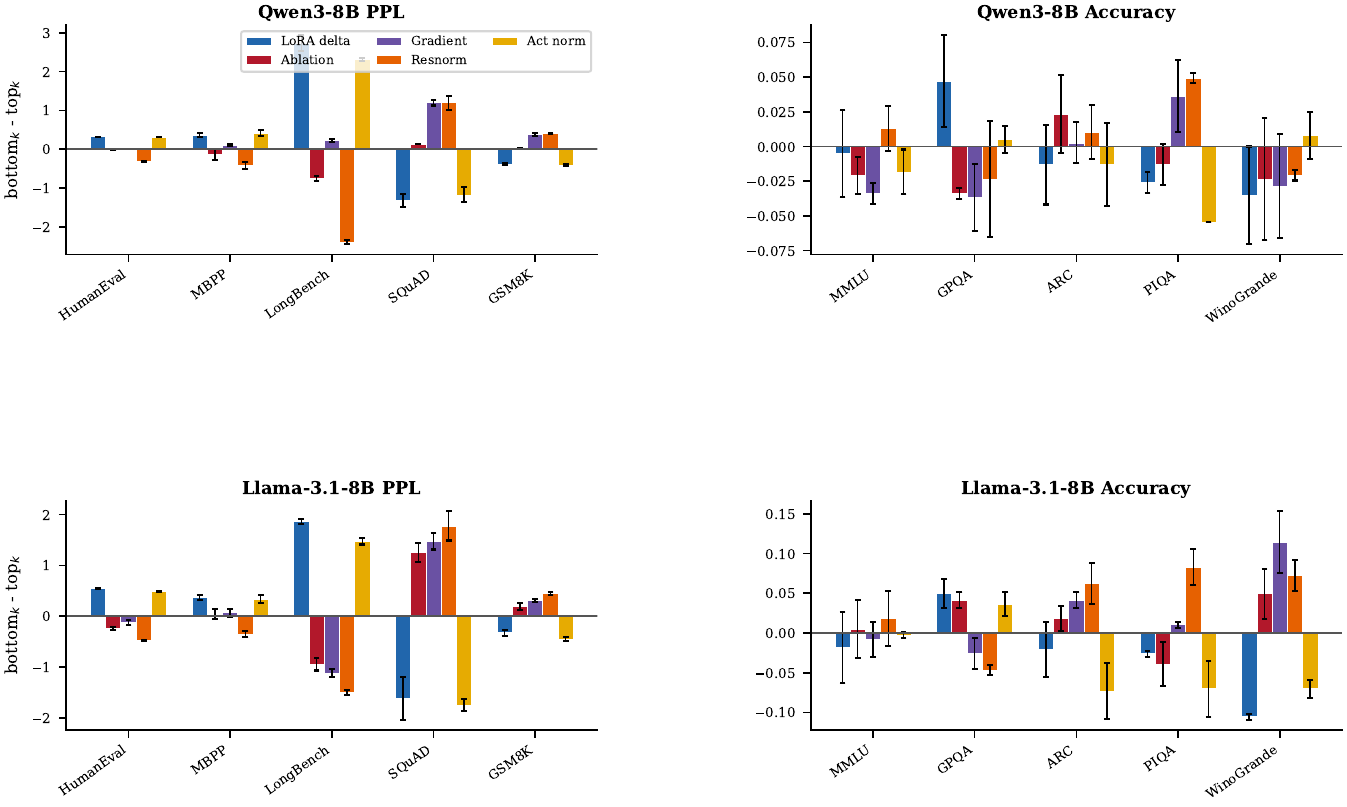}
\caption{Metric gaps across seeds. Left column: PPL gap (bottom$_k$ minus top$_k$) for generation tasks. Right column: accuracy gap (top$_k$ minus bottom$_k$) for MCQ tasks. Positive values support the hypothesis. Error bars show std across three seeds. These magnitudes should be inspected directly rather than inferred from the binary support label alone.}
\label{fig:multiseed_metric_gap}
\end{figure}

\begin{figure}[t]
\centering
\includegraphics[width=\linewidth]{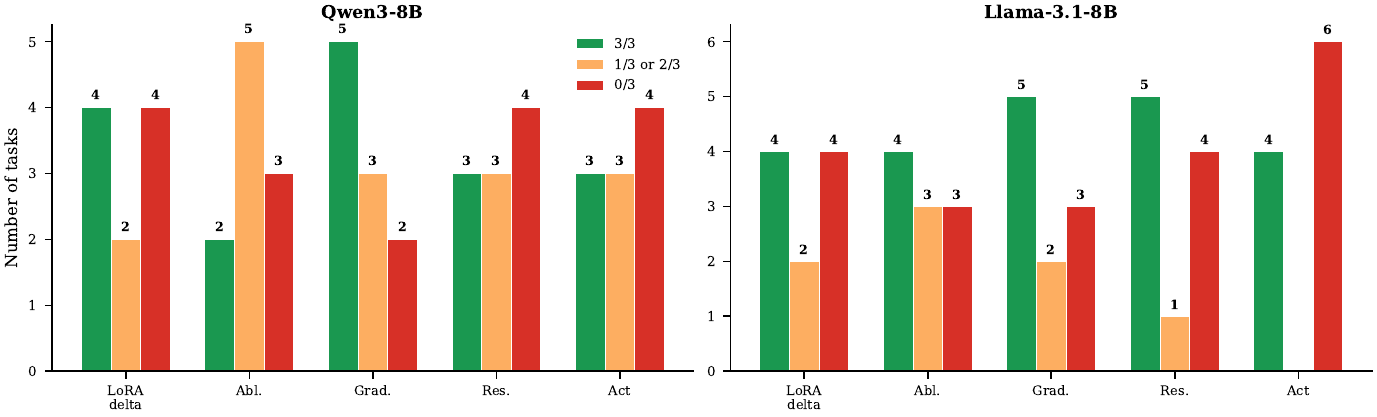}
\caption{Cross-seed stability of support decisions. Green indicates supported in all three seeds; orange indicates mixed support; red indicates unsupported in all three seeds. LoRA delta shows the highest fraction of fully stable supported tasks for code and long-context. Commonsense tasks are stably unsupported under LoRA delta across seeds.}
\label{fig:multiseed_stability}
\end{figure}

\subsection{Seed coverage}
\label{app:seed-coverage}

\begin{table}[t]
\centering
\caption{Per-checkpoint multi-seed coverage of the causal-validation runs. Brace-enclosed lists are the seed indices with completed result files. ``LoRA-$\delta$ support per seed'' shows the number of directionally supported tasks at each completed seed.}
\label{tab:seed-coverage}
\scriptsize

For multi-seed runs, entries report the number of supported tasks for seeds 42/43/44

\begin{tabular}{llc}
\toprule
\textbf{Model} & \textbf{Family} & \textbf{LoRA-$\delta$ support across seeds} \\
\midrule
Qwen3-0.6B & Tx & 3/3/2 of 10 \\
Qwen3-1.7B & Tx & 5/5/5 of 10 \\
Qwen3-4B & Tx & 4/3/6 of 10 \\
Qwen3-8B & Tx & 4/5/4 of 7 \\
Qwen3-14B & Tx & 7/7/8 of 10 \\
Qwen3-32B & Tx & 7/8/9 of 10 \\
Llama-3.2-1B & Tx & 6/4/3 of 10 \\
Llama-3.2-3B & Tx & 2/5/6 of 7 \\
Llama-3.1-8B & Tx & 3/6/4 of 7 \\
Llama-3.1-70B & Tx & 7 of 10 \\
Mamba-130m & SSM & 8/6/8 of 10 \\
Mamba-370m & SSM & 8/7/8 of 10 \\
Mamba-790m & SSM & 6/6/7 of 10 \\
Mamba-2.8b & SSM & 7/7/7 of 10 \\
\bottomrule
\end{tabular}
\end{table}

\subsection{Causal-validation margin magnitudes}
\label{app:causal-margins}

The main text uses top-$k$ versus bottom-$k$ support labels to test whether an importance ranking has directional causal consequences under selective LoRA placement. These labels record the sign of the comparison, not the magnitude of the downstream change. We therefore report the corresponding metric gaps here. For generation-style tasks evaluated with perplexity, the gap is $\mathrm{PPL}_{\mathrm{bottom}}-\mathrm{PPL}_{\mathrm{top}}$, so positive values favor the top-$k$ layers. For multiple-choice tasks evaluated with accuracy, the gap is $\mathrm{Acc}_{\mathrm{top}}-\mathrm{Acc}_{\mathrm{bottom}}$.

Figure~\ref{fig:multiseed_metric_gap} reports per-task gaps with seed error bars for Qwen3-8B and Llama-3.1-8B. Tables~\ref{tab:competitor-methods-support}, \ref{tab:a1-extension-support}, and \ref{tab:k-sweep} report primary-metric gaps for competitor placement methods, the Qwen3-14B and Mamba-790m extension runs, and the $k$-sweep respectively.

The margin results show three regimes. First, a small number of task--model--method triples have substantively large gaps, including Qwen3-1.7B on SQuAD under ablation, Qwen3-1.7B on WinoGrande where bottom-$k$ ablation layers win by a large margin, and Mamba-790m on SQuAD under LoRA delta. Second, most supported pairs are small in absolute magnitude, with typical primary-metric gaps in the $0.01$--$0.05$ range. Third, random placement, competitor methods, and several larger-scale extension runs remain close to zero on average. Thus, the causal-validation results should not be read as evidence that any single estimator provides a large universal downstream gain. They instead show that the direction of selective-placement validation is method-relative, even when the absolute margins are modest.

\begin{figure}[t]
\centering
\includegraphics[width=\linewidth]{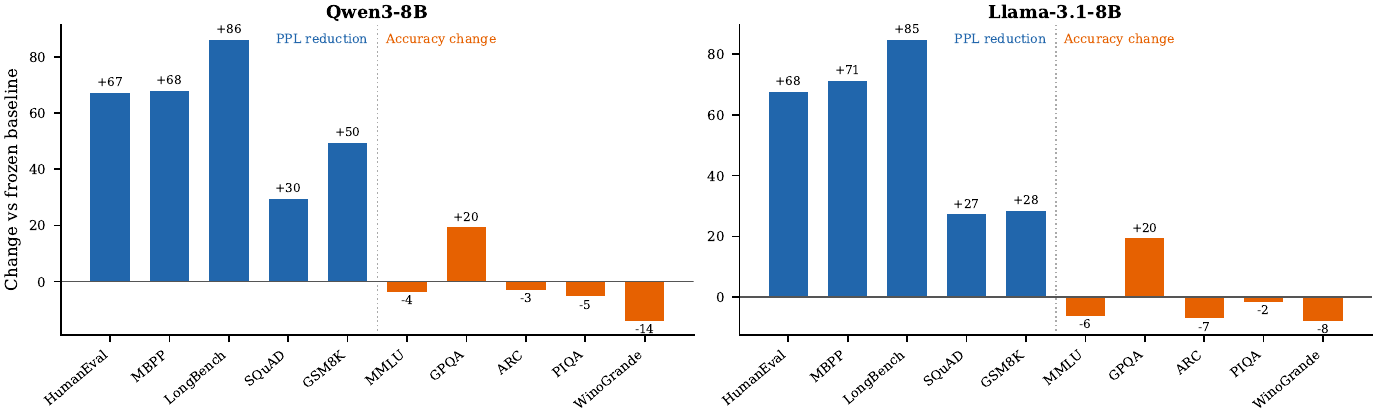}
\caption{Performance change under uniform LoRA relative to the frozen model. Uniform LoRA substantially improves generation perplexity, but does not uniformly improve multiple-choice accuracy.}
\label{fig:uniform-perf}
\end{figure}

\begin{figure}[t]
\centering
\includegraphics[width=\linewidth]{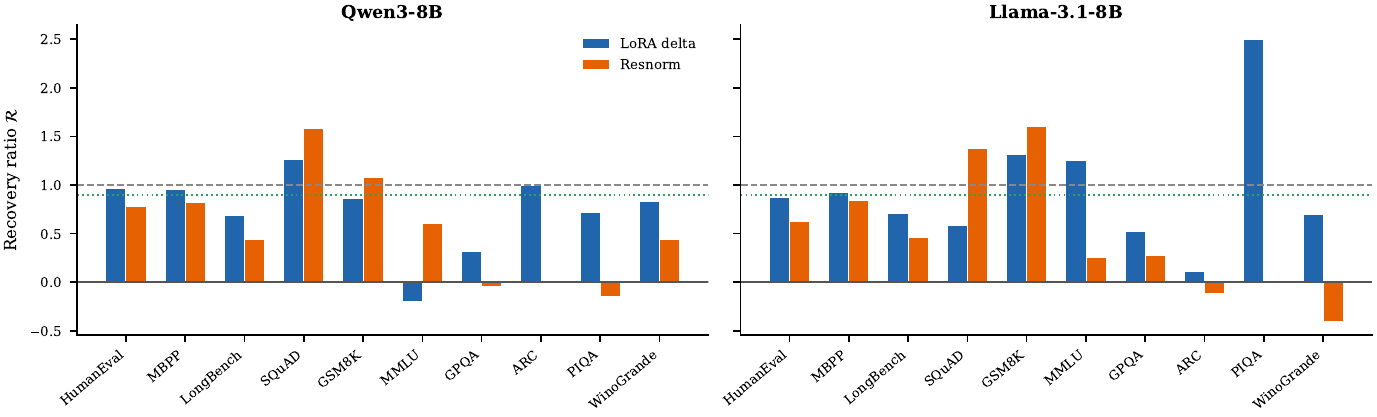}
\caption{Recovery ratio $\mathcal{R}$ for top-$k$ placement under LoRA delta and resnorm. LoRA delta recovers most of the full-LoRA improvement on code generation, while resnorm is inconsistent across tasks.}
\label{fig:recovery-ratios}
\end{figure}

\begin{figure}[t]
\centering
\includegraphics[width=\linewidth]{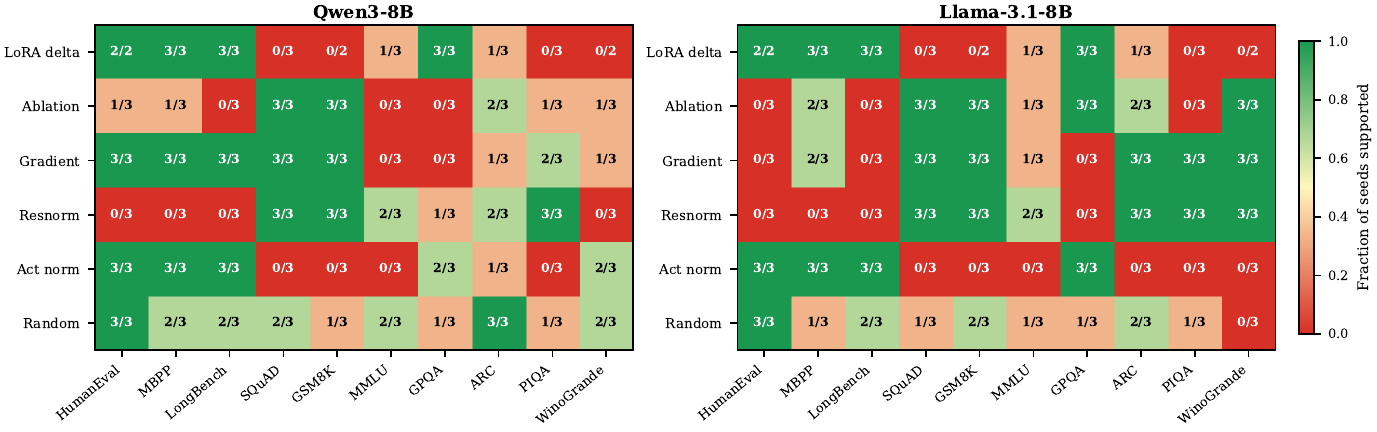}
\caption{Per-task causal-validation support across methods, shown as the fraction of seeds in which each method supports the hypothesis. Code tasks are the most consistently recoverable; commonsense tasks are the least.}
\label{fig:per-task-heatmap}
\end{figure}

\subsection{Sensitivity of causal validation to the budget \texorpdfstring{$k$}{k}}
\label{app:k-sweep}

The main protocol fixes $k = \lceil L/6 \rceil$. To check how much of the mixed-to-negative LoRA-delta verdict depends on that particular budget, we run an additional $k$-sweep for six transformer checkpoints across both families. For each model, we evaluate four values of $k$ spanning roughly $k/L \in [0.08, 0.31]$ and report directional support together with the mean top$_k$-minus-bottom$_k$ primary-metric gap, using the same task-native metrics and the same LoRA rank and training budget as the main experiments.

\begin{table}[t]
\centering
\caption{$k$-sweep causal validation on transformer checkpoints with LoRA-delta top-$k$. Each row is one $(model, k)$ configuration; all runs use the ten-task suite. Support is the fraction of tasks on which top-$k$ outperforms bottom-$k$ on the task-native primary metric. Mean primary gap is averaged across tasks; positive means top-$k$ is better.}
\label{tab:k-sweep}
\small
\resizebox{\linewidth}{!}{%
\begin{tabular}{llcccc}
\toprule
\textbf{Model} & \textbf{Family} & $k$ & $k/L$ & \textbf{Support} & \textbf{Mean primary gap} \\
\midrule
Llama-3.2-1B & Llama & 2 & 0.125 & $5/10$ & $+0.002$ \\
Llama-3.2-1B & Llama & 3 & 0.188 & $3/10$ & $+0.005$ \\
Llama-3.2-1B & Llama & 4 & 0.250 & $3/10$ & $-0.006$ \\
Llama-3.2-1B & Llama & 5 & 0.312 & $5/10$ & $+0.002$ \\
Llama-3.2-3B & Llama & 3 & 0.107 & $2/10$ & $-0.031$ \\
Llama-3.2-3B & Llama & 4 & 0.143 & $1/10$ & $-0.033$ \\
Llama-3.2-3B & Llama & 5 & 0.179 & $1/10$ & $-0.040$ \\
Llama-3.2-3B & Llama & 7 & 0.250 & $2/10$ & $-0.026$ \\
Llama-3.1-8B & Llama & 3 & 0.094 & $2/10$ & $-0.040$ \\
Llama-3.1-8B & Llama & 5 & 0.156 & $2/10$ & $-0.049$ \\
Llama-3.1-8B & Llama & 6 & 0.188 & $1/10$ & $-0.055$ \\
Llama-3.1-8B & Llama & 8 & 0.250 & $1/10$ & $-0.027$ \\
\midrule
Qwen3-1.7B & Qwen & 3 & 0.107 & $4/10$ & $+0.010$ \\
Qwen3-1.7B & Qwen & 4 & 0.143 & $5/10$ & $+0.012$ \\
Qwen3-1.7B & Qwen & 5 & 0.179 & $3/10$ & $+0.002$ \\
Qwen3-1.7B & Qwen & 7 & 0.250 & $4/10$ & $+0.008$ \\
Qwen3-4B & Qwen & 3 & 0.083 & $2/10$ & $+0.021$ \\
Qwen3-4B & Qwen & 5 & 0.139 & $3/10$ & $+0.025$ \\
Qwen3-4B & Qwen & 6 & 0.167 & $3/10$ & $+0.021$ \\
Qwen3-4B & Qwen & 8 & 0.222 & $3/10$ & $+0.015$ \\
Qwen3-8B & Qwen & 3 & 0.083 & $1/10$ & $-0.019$ \\
Qwen3-8B & Qwen & 5 & 0.139 & $1/10$ & $-0.011$ \\
Qwen3-8B & Qwen & 6 & 0.167 & $4/10$ & $-0.005$ \\
Qwen3-8B & Qwen & 8 & 0.222 & $2/10$ & $-0.015$ \\
\midrule
\textbf{Llama family} & --- & all & --- & $\bar{s} = 0.23$ & $\bar{g} = -0.025$ \\
\textbf{Qwen family}  & --- & all & --- & $\bar{s} = 0.29$ & $\bar{g} = +0.005$ \\
\bottomrule
\end{tabular}%
}
\end{table}

The qualitative verdict is insensitive to the budget choice: the $k = \lceil L/6 \rceil$ row sits inside the envelope swept by $k \in [2,8]$ at each checkpoint, and neither family benefits consistently from a different $k$. Across all twelve Llama $(model,k)$ points, mean support is $0.23$ and the mean primary gap is $-0.025$; across the twelve Qwen points, mean support is $0.29$ and the mean primary gap is $+0.005$. The method-relative disagreement reported in the main text is therefore not an artifact of the fixed budget. The caveat is that this sweep was not run on Mamba checkpoints or on the largest transformers, and it sweeps only LoRA delta rather than all estimators.

\subsection{Task-A adaptation quality: top-\texorpdfstring{$k$}{k} versus bottom-\texorpdfstring{$k$}{k}}
\label{app:ppl-before}

A potential confound in the continual-learning protocol is that bottom-$k$ layers
(least plastic) may simply learn Task~A less thoroughly than top-$k$ layers,
mechanically producing a lower forgetting ratio without any genuine bottleneck
effect. Table~\ref{tab:ppl-before} addresses this directly by reporting the
post-Task-A perplexity ($\text{PPL}_\text{before}$, measured after Task-A
fine-tuning and before Task-B training) for both placement strategies across all
models, together with the resulting forgetting ratios.

Two patterns emerge. First, in \textbf{transformers}, top-$k$ does achieve
modestly better Task-A adaptation: averaged across all transformer checkpoints
and pairs ($n{=}75$), PPL$_\text{before}$ is $1.71$ for top-$k$ versus $2.00$
for bottom-$k$ --- a $17\%$ gap, meaning the least plastic layers absorb
Task-A somewhat less. However, the forgetting ratio gap is substantially
larger: $3.24$ for top-$k$ versus $2.33$ for bottom-$k$ --- a $39\%$ difference
in the direction opposite to what a pure ``learned less, forgets less'' account
would predict. A confounder that accounts for a $17\%$ advantage in starting
perplexity cannot explain a $39\%$ disadvantage in forgetting ratio; the
bottleneck effect is not an artifact of differential Task-A learning.

Second, in \textbf{SSMs}, top-$k$ and bottom-$k$ PPL$_\text{before}$ are
virtually identical ($6.53$ vs.\ $6.73$, a $3\%$ gap, $n{=}54$), and the
forgetting ratios are likewise statistically indistinguishable ($2.11$ vs.\
$2.00$). This confirms that the null forgetting result in SSMs is not driven
by top-$k$ and bottom-$k$ learning at different rates.

\begin{table}[t]
\centering
\caption{Task-A post-training perplexity (PPL$_\text{before}$, lower is better)
and forgetting ratio (PPL$_\text{after}$/PPL$_\text{before}$, lower is better)
for top-$k$ vs.\ bottom-$k$ placement across all CL models. Values are mean
$\pm$ std across all task pairs and seeds. The ``ratio'' column shows
(bottom-$k$ PPL$_\text{before}$) / (top-$k$ PPL$_\text{before}$); values
above $1.0$ indicate that top-$k$ achieves better Task-A adaptation. Family
summary rows pool all pairs within the family. RWKV6-3B omitted (numerical
divergence in top-$k$ runs).}
\label{tab:ppl-before}
\scriptsize
\setlength{\tabcolsep}{3pt}
\resizebox{\linewidth}{!}{%
\begin{tabular}{llccccc}
\toprule
\multirow{2}{*}{\textbf{Model}} & \multirow{2}{*}{\textbf{Family}} &
\multicolumn{3}{c}{\textbf{PPL$_\text{before}$ (Task-A fit)}} &
\multicolumn{2}{c}{\textbf{Forgetting ratio}} \\
\cmidrule(lr){3-5}\cmidrule(lr){6-7}
 & & top-$k$ & bottom-$k$ & bot/top & top-$k$ & bottom-$k$ \\
\midrule
Llama-3.2-3B & Transformer & $2.01\pm0.73$ & $2.22\pm0.57$ & $1.11$ & $2.67\pm1.66$ & $2.19\pm0.91$ \\
Llama-3.1-8B & Transformer & $1.80\pm0.69$ & $2.05\pm0.57$ & $1.14$ & $2.86\pm1.62$ & $2.57\pm1.48$ \\
Qwen3-0.6B & Transformer & $1.87\pm0.68$ & $2.24\pm0.36$ & $1.20$ & $6.83\pm4.28$ & $2.67\pm1.01$ \\
Qwen3-4B & Transformer & $1.52\pm0.38$ & $1.98\pm0.30$ & $1.30$ & $2.66\pm1.35$ & $2.22\pm0.80$ \\
Qwen3-8B & Transformer & $1.39\pm0.32$ & $1.55\pm0.37$ & $1.11$ & $2.85\pm1.51$ & $2.18\pm1.06$ \\
\midrule
Mamba-130M & SSM & $10.45\pm3.57$ & $10.80\pm3.81$ & $1.03$ & $1.72\pm0.29$ & $2.09\pm0.75$ \\
Mamba-370M & SSM & $8.33\pm4.08$ & $8.58\pm4.17$ & $1.03$ & $2.93\pm1.52$ & $1.48\pm0.23$ \\
Mamba-790M & SSM & $7.05\pm3.23$ & $7.18\pm3.37$ & $1.02$ & $1.94\pm0.57$ & $1.85\pm0.56$ \\
Mamba-2.8B & SSM & $4.78\pm1.57$ & $5.18\pm1.54$ & $1.08$ & $2.68\pm1.27$ & $3.45\pm2.64$ \\
Mamba2-1.3B & SSM & $6.55\pm2.05$ & $6.42\pm2.30$ & $0.98$ & $1.57\pm0.50$ & $1.87\pm0.49$ \\
Falcon-Mamba-7B & SSM & $2.01\pm0.43$ & $2.21\pm0.27$ & $1.10$ & $1.81\pm0.62$ & $1.30\pm0.19$ \\
\midrule
Zamba2-2.7B & Hybrid & $2.37\pm0.78$ & $2.74\pm1.33$ & $1.15$ & $1.28\pm0.16$ & $1.43\pm0.41$ \\
\midrule
\textbf{Transformer} & --- ($n{=}75$) & $1.71\pm0.61$ & $2.00\pm0.51$ & $1.17$ & $3.24\pm2.40$ & $2.33\pm1.05$ \\
\textbf{SSM} & --- ($n{=}54$) & $6.53\pm3.77$ & $6.73\pm3.88$ & $1.03$ & $2.11\pm1.01$ & $2.00\pm1.31$ \\
\textbf{Hybrid} & --- ($n{=}9$) & $2.37\pm0.78$ & $2.74\pm1.33$ & $1.15$ & $1.28\pm0.16$ & $1.43\pm0.41$ \\
\bottomrule
\end{tabular}%
}
\end{table}

\section{Estimator and Architecture Extensions}
\label{app:extensions}

\subsection{Head-to-head comparison with published placement methods}
\label{app:competitor-baselines}

We re-implement two competitor placement methods---ShapLoRA~\citep{shaplora} and TELL-TALE~\citep{telltale}---and run them through the same top-$k$ versus bottom-$k$ causal-validation protocol on Qwen3-8B, Llama-3.1-8B, and Mamba-790m at seed 42. This comparison is not intended as a refutation of either method in its original setting; the goal is to test whether the same matched diagnostic protocol produces a single estimator that dominates across architectures.

Table~\ref{tab:competitor-methods-support} reports directional support rates restricted to the task list common to every method on each model: a seven-task subset for the two transformer checkpoints, and a nine-task subset for Mamba-790m. On both transformer checkpoints, neither competitor beats random under the matched protocol. The same pattern holds on Mamba-790m, where both competitors land below LoRA delta. Mean primary-metric gaps are near zero or negative for both competitors on every tested model, indicating that the under-performance is directional rather than a thresholding artifact of the binary support metric.

\begin{table}[t]
\centering
\caption{Causal-validation comparison against published placement methods at seed 42. Each block restricts to the task list common to every method listed for that model. Support is the fraction of tasks on which top-$k$ outperforms bottom-$k$ under the loss-aligned validation metric. Mean PPL gap is bottom$_k$ minus top$_k$ on perplexity-validated tasks; mean primary gap is top$_k$ minus bottom$_k$ on task-native primary metrics. Empty primary-gap entries indicate runs whose logs predate the task-native evaluation path. LoRA-delta primary-gap entries are taken from the task-native audit and may use a different task subset.}
\label{tab:competitor-methods-support}
\small
\resizebox{\linewidth}{!}{%
\begin{tabular}{lccc}
\toprule
\textbf{Method} & \textbf{Support} & \textbf{Mean PPL gap} & \textbf{Mean primary gap} \\
\midrule
\multicolumn{4}{l}{\textit{Qwen3-8B} ($N_{\text{tasks}}=7$)} \\
  LoRA delta~\citep{lora}    & $4/7$ (57\%) & $-1.359$ & $+0.001^{\dagger}$ \\
  Ablation                    & $3/7$ (43\%) & $+0.771$ & --- \\
  Gradient                    & $5/7$ (71\%) & $+1.604$ & --- \\
  Resnorm~\citep{layercard}   & $3/7$ (43\%) & $+1.320$ & --- \\
  Random                      & $4/7$ (57\%) & $-0.106$ & --- \\
  ShapLoRA~\citep{shaplora}   & $3/7$ (43\%) & $+0.126$ & $-0.031$ \\
  TELL-TALE~\citep{telltale}  & $2/7$ (29\%) & $-1.078$ & $-0.010$ \\
\midrule
\multicolumn{4}{l}{\textit{Llama-3.1-8B} ($N_{\text{tasks}}=7$)} \\
  LoRA delta~\citep{lora}    & $3/7$ (43\%) & $-1.471$ & $-0.044^{\ddagger}$ \\
  Ablation                    & $4/7$ (57\%) & $+1.270$ & --- \\
  Gradient                    & $4/7$ (57\%) & $+1.296$ & --- \\
  Resnorm~\citep{layercard}   & $4/7$ (57\%) & $+1.123$ & --- \\
  Random                      & $4/7$ (57\%) & $-0.750$ & --- \\
  ShapLoRA~\citep{shaplora}   & $3/7$ (43\%) & $+0.677$ & $+0.048$ \\
  TELL-TALE~\citep{telltale}  & $1/7$ (14\%) & $-0.761$ & $-0.051$ \\
\midrule
\multicolumn{4}{l}{\textit{Mamba-790m} ($N_{\text{tasks}}=9$)} \\
  LoRA delta~\citep{lora}    & $5/9$ (56\%) & $+0.502$ & $-0.026^{\S}$ \\
  ShapLoRA~\citep{shaplora}   & $2/9$ (22\%) & $+0.164$ & $-0.008$ \\
  TELL-TALE~\citep{telltale}  & $2/9$ (22\%) & $+0.148$ & $-0.012$ \\
\bottomrule
\end{tabular}
}
\end{table}

\subsection{Scale and architecture extension: Qwen3-14B and Mamba-790m}
\label{app:a1-extension}

The Qwen3-8B and Llama-3.1-8B seeded study leaves two natural questions open: whether the scale-dependent weakening of LoRA delta persists at larger transformer sizes, and whether the transformer--SSM contrast survives a seeded causal-validation setup. To address both, we rerun the causal-validation pipeline on Qwen3-14B and Mamba-790m at seeds 43 and 44, with the same $k=6$, LoRA rank 8, and 300 training steps as the 8B study. Each model is run under three placement regimes: LoRA-delta top-$k$, random top-$k$, and uniform LoRA over all layers. For Mamba-790m, we report aggregates on the nine non-LongBench tasks because the A1 task list and the Mamba-790m importance scores use different LongBench variants.

Table~\ref{tab:a1-extension-support} summarizes directional support rates and mean primary-metric gaps. On Mamba-790m, LoRA delta yields a positive mean top$_k$-minus-bottom$_k$ primary-metric gap at both seeds, while random top-$k$ is near zero or negative. On Qwen3-14B, the picture is inverted: random top-$k$ matches or exceeds LoRA delta both on support rate and on mean top-minus-bottom gap. This is consistent with the Qwen3-8B to Qwen3-32B bimodality reported in Figure~\ref{fig:bimodality}: in the transformer family, LoRA delta's prescriptive content weakens as scale grows, while in the SSM family it remains informative under the same validation recipe.

\begin{table}[t]
\centering
\caption{Scale/architecture extension: per-seed directional support rate and mean primary-metric gap for LoRA delta and random top-$k$, plus per-seed uniform-LoRA win rate against the no-adapter baseline. Both checkpoints run with $k=6$, LoRA rank 8, 300 training steps, at seeds 43 and 44.}
\label{tab:a1-extension-support}
\scriptsize
\resizebox{\linewidth}{!}{%
\begin{tabular}{llcccc}
\toprule
\textbf{Model} & \textbf{Method} & \textbf{Seed 43} & \textbf{Seed 44} & \textbf{Mean gap (s43, s44)} & \textbf{$N_{\text{tasks}}$} \\
\midrule
\multirow{3}{*}{Qwen3-14B}
 & LoRA delta   & $2/10$ ($20$\%) & $4/10$ ($40$\%) & $-0.017$, $-0.001$ & 10 \\
 & Random       & $5/10$ ($50$\%) & $5/10$ ($50$\%) & $-0.000$, $+0.004$ & 10 \\
 & Uniform LoRA & $3/10$ wins     & $2/10$ wins     & $+0.062$, $+0.060$ & 10 \\
\midrule
\multirow{3}{*}{Mamba-790m$^{\dagger}$}
 & LoRA delta   & $3/9$ ($33$\%)  & $3/9$ ($33$\%)  & $+0.006$, $+0.021$ & 9 \\
 & Random       & $3/9$ ($33$\%)  & $2/9$ ($22$\%)  & $+0.001$, $-0.015$ & 9 \\
 & Uniform LoRA & $2/9$ wins      & $3/9$ wins      & $+0.034$, $+0.033$ & 9 \\
\bottomrule
\end{tabular}%
}

\vspace{0.4em}
{\small $^{\dagger}$Mamba-790m LongBench was not evaluated in the A1 extension: the updated A1 task list uses \texttt{longbench\_multifieldqa\_en} whereas the Mamba-790m LoRA-delta importance scores cover \texttt{longbench\_single\_doc}; regenerating Mamba importance against the new variant was out of scope. Rows therefore report the nine tasks common to both checkpoints.}

\end{table}

\begin{figure}[t]
\centering
\includegraphics[width=\linewidth]{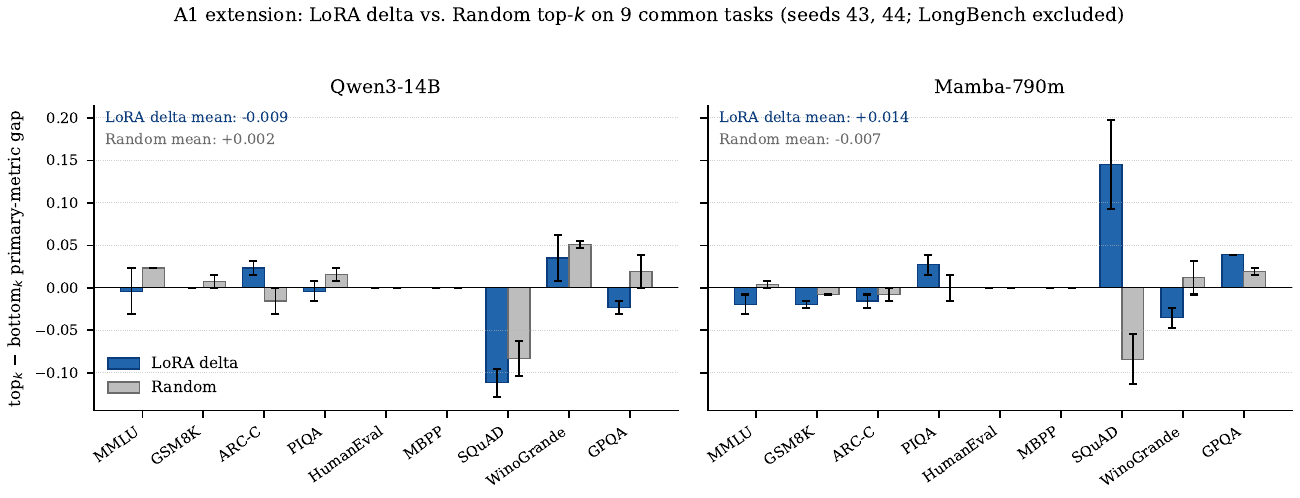}
\caption{Per-task top$_k$ minus bottom$_k$ primary-metric gap on Qwen3-14B and Mamba-790m for LoRA delta and random top-$k$ placement. Bars are seed-averaged means; error bars show the seed range over seeds 43 and 44. LongBench is excluded to keep the task set comparable across checkpoints.}
\label{fig:a1-extension-gap}
\end{figure}

\subsection{Mamba scale sweep}
\label{app:mamba_scale}

Tables~\ref{tab:mamba_ablation_sweep} and~\ref{tab:mamba_lora_sweep} report causal validation for all four Mamba checkpoints. Under ablation, top-$k$ layers follow the boundary pattern at all scales. Under LoRA delta, top-$k$ layers shift to mid-network positions and support rates remain high, including tasks such as WinoGrande and PIQA that fail systematically in transformers under LoRA delta.

\begin{table}[t]
\centering
\caption{Mamba causal validation under ablation across four model sizes. Top-$k$ layers follow the boundary pattern at all scales.}
\label{tab:mamba_ablation_sweep}
\small
\resizebox{\linewidth}{!}{%
\begin{tabular}{lcccc}
\toprule
\textbf{Task} & \textbf{130m} ($k$=4) & \textbf{370m} ($k$=5) & \textbf{790m} ($k$=5) & \textbf{2.8b} ($k$=8) \\
\midrule
HumanEval   & \checkmark & \checkmark & \checkmark & \checkmark \\
MBPP        & \checkmark & \checkmark & \checkmark & \checkmark \\
LongBench   & \checkmark & \checkmark & \checkmark & \checkmark \\
GSM8K       & \checkmark & $\times$   & \checkmark & \checkmark \\
SQuAD       & \checkmark & \checkmark & \checkmark & \checkmark \\
MMLU        & \checkmark & \checkmark & \checkmark & \checkmark \\
GPQA        & \checkmark & \checkmark & \checkmark & \checkmark \\
ARC         & \checkmark & $\times$   & $\times$   & \checkmark \\
PIQA        & \checkmark & $\times$   & $\times$   & $\times$ \\
WinoGrande  & \checkmark & \checkmark & $\times$   & \checkmark \\
\midrule
\textbf{Rate} & \textbf{10/10} & \textbf{7/10} & \textbf{7/10} & \textbf{9/10} \\
\bottomrule
\end{tabular}%
}
\end{table}

\begin{table}[t]
\centering
\caption{Mamba causal validation under LoRA delta across four model sizes. Mid-network layers are selected at all scales; support rates are 9--10/10.}
\label{tab:mamba_lora_sweep}
\small
\resizebox{\linewidth}{!}{%
\begin{tabular}{lcccc}
\toprule
\textbf{Task} & \textbf{130m} ($k$=4) & \textbf{370m} ($k$=5) & \textbf{790m} ($k$=5) & \textbf{2.8b} ($k$=8) \\
\midrule
HumanEval   & \checkmark & \checkmark & \checkmark & \checkmark \\
MBPP        & \checkmark & \checkmark & \checkmark & \checkmark \\
LongBench   & \checkmark & \checkmark & \checkmark & \checkmark \\
GSM8K       & \checkmark & \checkmark & \checkmark & \checkmark \\
SQuAD       & \checkmark & \checkmark & \checkmark & \checkmark \\
MMLU        & \checkmark & \checkmark & $\times$   & $\times$ \\
GPQA        & \checkmark & \checkmark & \checkmark & \checkmark \\
ARC         & \checkmark & \checkmark & \checkmark & \checkmark \\
PIQA        & \checkmark & \checkmark & \checkmark & \checkmark \\
WinoGrande  & \checkmark & \checkmark & \checkmark & \checkmark \\
\midrule
\textbf{Rate} & \textbf{10/10} & \textbf{10/10} & \textbf{9/10} & \textbf{9/10} \\
\bottomrule
\end{tabular}%
}
\end{table}

\section{Alternative-Mechanism Controls}
\label{sec:robustness}
\label{app:robustness}

\subsection{The sign flip is not a layer-0 artifact}
\label{app:l0-control}

\begin{table}[t]
\centering
\caption{Necessity-Plasticity Alignment $\mathcal{NPA}(M)$ for the 14 primary paired pretrained checkpoints (Qwen3, Llama, Mamba). $\mathcal{NPA}(M)$ is the mean Spearman rank correlation between LoRA-delta and ablation importance vectors across $T=10$ tasks. $\mathcal{NPA}^{\setminus L_0}(M)$ recomputes the same correlation with layer 0 excluded from both vectors. L0 share is the fraction of total ablation mass concentrated in the first layer. Family means in bold are arithmetic averages over the listed checkpoints. RWKV6-3B, Zamba2-2.7B, and Falcon-Mamba-7B are reported in Section~\ref{sec:consistency_results} as validation architectures.}
\label{tab:consistency}
\small
\resizebox{\linewidth}{!}{%
\begin{tabular}{llrrrr}
\toprule
\textbf{Model} & \textbf{Family} & \textbf{$L$} & \boldmath$C(M)$ & \boldmath$C(M)\setminus\!L_0$ & \textbf{L0 share} \\
\midrule
Qwen3-0.6B & Transformer & 28 & $-0.25$ & $-0.16$ & 0.89 \\
Qwen3-1.7B & Transformer & 28 & $-0.47$ & $-0.40$ & 0.94 \\
Qwen3-4B & Transformer & 36 & $-0.23$ & $-0.16$ & 0.61 \\
Qwen3-8B & Transformer & 36 & $-0.30$ & $-0.24$ & 0.60 \\
Qwen3-14B & Transformer & 40 & $-0.06$ & $-0.10$ & 0.60 \\
Qwen3-32B & Transformer & 64 & $+0.07$ & $+0.10$ & 0.60 \\
Llama-3.2-1B & Transformer & 16 & $-0.74$ & $-0.69$ & 0.32 \\
Llama-3.2-3B & Transformer & 28 & $-0.48$ & $-0.43$ & 0.49 \\
Llama-3.1-8B & Transformer & 32 & $-0.62$ & $-0.59$ & 0.54 \\
Llama-3.1-70B & Transformer & 80 & $+0.01$ & $+0.05$ & 0.61 \\
Mamba-130m & SSM & 24 & $+0.22$ & $+0.22$ & 0.30 \\
Mamba-370m & SSM & 48 & $+0.22$ & $+0.20$ & 0.29 \\
Mamba-790m & SSM & 48 & $+0.25$ & $+0.26$ & 0.09 \\
Mamba-2.8b & SSM & 64 & $+0.28$ & $+0.29$ & 0.07 \\
\midrule
\textbf{Transformer mean (n=10)} & & & $\mathbf{-0.31}$ & $\mathbf{-0.26}$ & 0.62 \\
\textbf{Transformer mean ($\le$14B, n=8)} & & & $\mathbf{-0.39}$ & $\mathbf{-0.35}$ & 0.62 \\
\textbf{Mamba mean (n=4)} & & & $\mathbf{+0.24}$ & $\mathbf{+0.24}$ & 0.19 \\
\bottomrule
\end{tabular}
}
\end{table}

\begin{figure}[t]
\centering
\includegraphics[width=\linewidth]{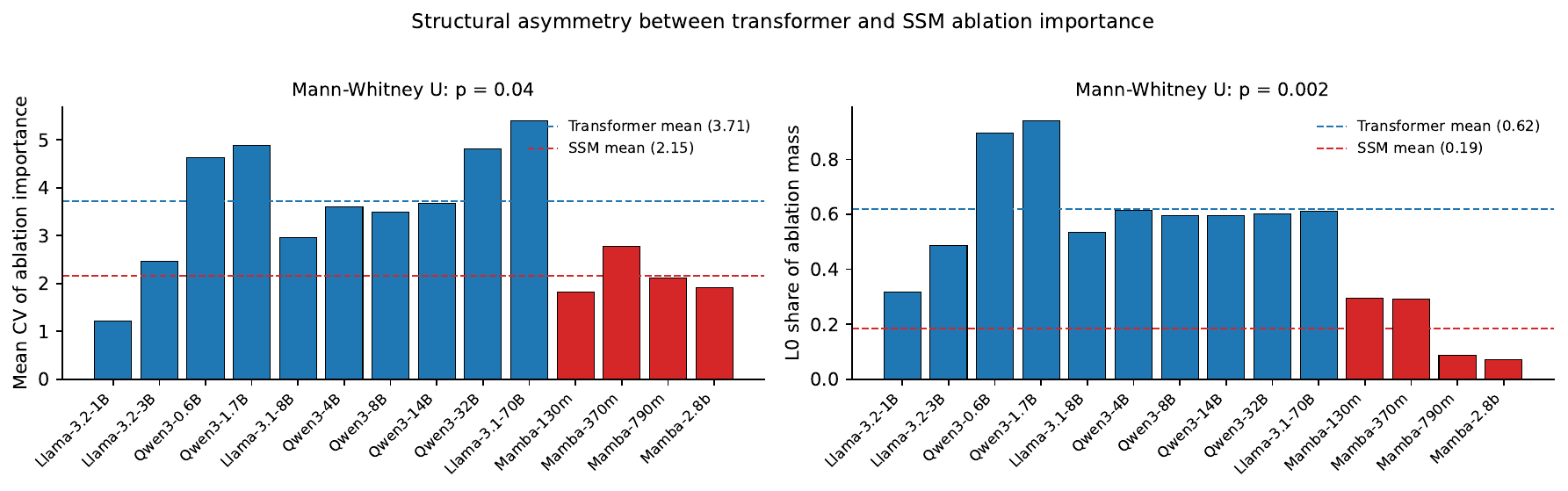}
\caption{Per-checkpoint distribution of the structural asymmetry summarized by leave-out-L0 and L0-share controls. Dashed lines mark family means; two-sided Mann--Whitney U $p$-values compare family membership.}
\label{fig:mechanism-cv-l0-app}
\end{figure}

The L0 control rules out the simplest explanation for the transformer--Mamba split. Transformers place much more ablation mass on layer 0 than Mambas, so the control is necessary. However, removing layer 0 shifts transformer correlations toward zero without changing the qualitative sign for transformers up to 14B. The observed anti-correlation is therefore not reducible to a single embedding-bottleneck layer.

\subsection{The sign flip is not a submodule-scope artifact}
\label{app:submodule}

In the main experiments, LoRA delta is computed on trainable LoRA target modules, whereas ablation removes a layer-local computation. This raises a natural concern: perhaps the observed disagreement is caused by comparing different parts of the block rather than by comparing different notions of importance. To address this, we recompute placement ranking alignment under matched submodule scopes. For each matched setting, both estimators are restricted to the same part of the transformer block.

The cross-method anti-correlation is negative for every available submodule in Llama-3.1-8B: $\rho=-0.581$ for attention, $-0.266$ for MLP, and $-0.298$ for full-block comparisons. The corresponding top-6 layer overlaps are 15\%, 17\%, and 27\%. For two random 6-of-32 selections, the expected overlap is $6\times6/32=1.125$ layers, or 18.75\% of a top-6 set. Thus the attention and MLP overlaps are at or below random, whereas the full-block overlap remains far below within-LoRA overlaps. This rules out the possibility that the full-block anti-correlation is an artifact of mixing attention and MLP signals.

Within LoRA delta, submodules converge on the same layers. The mean Spearman correlation between LD-attn and LD-MLP is $+0.839$ in Llama and $+0.659$ in Qwen, rising to $+0.883$ for LD-attn versus LD-full in Llama. Top-6 overlap ranges from 55--60\%. This indicates that the terminal cluster is not a property of which projections are updated: adapting attention alone, MLP alone, or both simultaneously identifies the same final depth region.

Within ablation, attention and MLP criticality are more model-dependent. In Llama, Abl-attn versus Abl-MLP yields $\rho=+0.427$ and 67\% top-6 overlap, suggesting moderate co-localization. In Qwen3-8B, the same comparison yields $\rho=-0.027$ and 27\% overlap, indicating structurally distinct attention-critical and MLP-critical layer sets. Thus, if using ablation to guide submodule-specific placement, the submodule choice matters; if using LoRA delta, the choice is much less important.

\begin{table}[t]
\centering
\caption{Submodule-matched placement ranking alignment. The Llama-3.1-8B full-block row is the headline control: both estimators are applied at the block level, yet the correlation remains negative. Within-LoRA-delta correlations between attention- and MLP-matched scorings are strongly positive, ruling out within-method noise.}
\label{tab:submodule-matched-main}
\small
\resizebox{\linewidth}{!}{%
\begin{tabular}{lccc}
\toprule
\textbf{Cross-method} $\rho$ & \textbf{Attn-matched} & \textbf{MLP-matched} & \textbf{Full-block} \\
\midrule
Llama-3.1-8B & $-0.58$ & $-0.27$ & $\mathbf{-0.30}$ \\
Qwen3-8B     & $-0.07$ & $+0.15$ & --- \\
\bottomrule
\end{tabular}%
}
\end{table}
\begin{table}[t]
\centering
\caption{Mean Spearman $\rho$ between condition pairs for Llama-3.1-8B and Qwen3-8B. Cross-method comparisons are negative in Llama across all submodules; within-LoRA-delta comparisons are positive in both models; within-ablation comparisons reveal model-dependent divergence.}
\label{tab:submodule_rho}
\small
\resizebox{\linewidth}{!}{%
\begin{tabular}{llccc}
\toprule
\textbf{Comparison type} & \textbf{Model} & \textbf{Attn} & \textbf{MLP} & \textbf{Full} \\
\midrule
\multirow{2}{*}{Abl vs LD (cross-method)}
  & Llama-3.1-8B & $-0.581$ & $-0.266$ & $-0.298$ \\
  & Qwen3-8B     & $-0.066$ & $+0.148$ & --- \\
\midrule
\multirow{2}{*}{LD-attn vs LD-$X$ (within LoRA delta)}
  & Llama-3.1-8B & $1.000$ & $+0.839$ & $+0.883$ \\
  & Qwen3-8B     & $1.000$ & $+0.659$ & --- \\
\midrule
\multirow{2}{*}{Abl-attn vs Abl-$X$ (within ablation)}
  & Llama-3.1-8B & $1.000$ & $+0.427$ & $+0.557$ \\
  & Qwen3-8B     & $1.000$ & $-0.027$ & $+0.284$ \\
\bottomrule
\end{tabular}%
}
\end{table}

\begin{figure}[t]
\centering
\includegraphics[width=\linewidth]{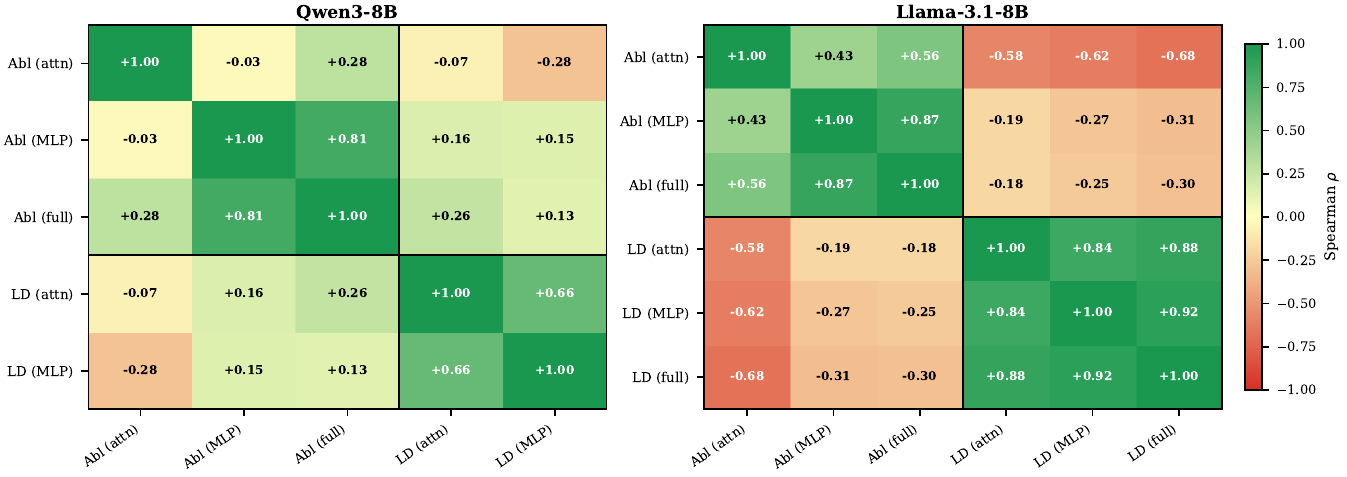}
\caption{Full Spearman $\rho$ matrix across matched-control conditions for Llama-3.1-8B and Qwen3-8B. The top-left block contains within-ablation correlations; the bottom-right block contains within-LoRA-delta correlations. The off-diagonal quadrants show the cross-method signal.}
\label{fig:sub_correlation_matrix}
\end{figure}

\begin{figure}[t]
\centering
\includegraphics[width=\linewidth]{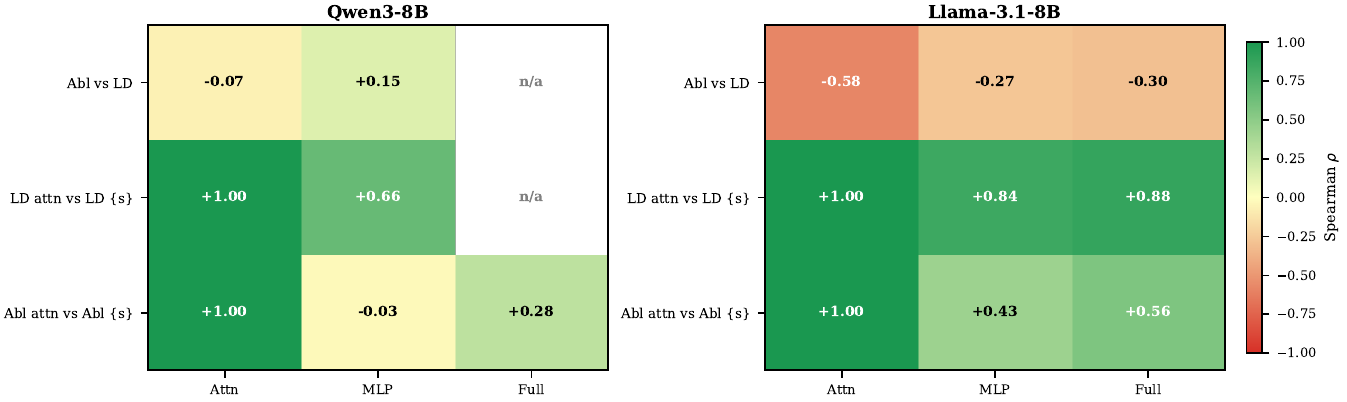}
\caption{Summary Spearman $\rho$ heatmap by comparison type and submodule. The cross-method row is negative for Llama across all submodules. The within-LoRA-delta row is positive for both models. The within-ablation row reveals that Qwen3-8B has near-zero attention--MLP correlation, indicating structurally separate attention-critical and MLP-critical layer sets.}
\label{fig:sub_rho_summary}
\end{figure}

\begin{revblock}
\subsection{The sign flip is not a LoRA artifact: adapter-free full fine-tuning control}
\label{app:fullft}

Plasticity is defined in the main text through the magnitude of a LoRA update,
which raises the possibility that $\mathcal{NPA}$ measures where LoRA places
updates rather than where the architecture absorbs new information. We
therefore rerun the Plasticity measurement without adapters: for each
checkpoint we fine-tune \emph{all} parameters on the same tasks, under the same
data and budget, and recompute the per-layer update magnitude
$\|\Delta W_l\|_F$ from the difference between the fine-tuned and pretrained
weights. This removes every LoRA-specific degree of freedom at once: rank,
scaling factor $\alpha$, initialization, and the choice of target modules.

Table~\ref{tab:fullft-control} reports the result on the $13$ checkpoints we
were able to fully retrain. Two quantities matter. First,
$\rho(\Delta W^{\mathrm{full}}, \Delta W^{\mathrm{LoRA}})$ is positive on
\emph{every} checkpoint, so unconstrained fine-tuning and LoRA select the same
layers; LoRA is not redirecting updates to an unrepresentative subset of depth.
Second, $\rho(\Delta W^{\mathrm{full}}, \text{ablation})$ --- the adapter-free
analogue of $\mathcal{NPA}$ --- is negative for every evaluated transformer
with an ablation baseline, and positive or near zero for every evaluated Mamba,
RWKV, and hybrid checkpoint. The sign therefore matches the LoRA-based
$\mathcal{NPA}$ of Table~\ref{tab:consistency} in all $13$ cases.

Correlation between the two adaptation procedures declines with model size
(from $0.96$ on Qwen3-0.6B to $0.39$ on Qwen3-14B), which is expected: larger
models have more directions in which a full-rank update can differ from a
rank-$r$ one. What does \emph{not} decline is the agreement in sign between
$\rho(\Delta W^{\mathrm{full}}, \text{ablation})$ and $\mathcal{NPA}$, which is
the quantity the paper's claims rest on.

\begin{table}[t]
\centering

\caption{Adapter-free control. $\rho(\Delta W^{\mathrm{full}}, \Delta W^{\mathrm{LoRA}})$ is the Spearman correlation between full-parameter and LoRA per-layer update profiles. $\rho(\Delta W^{\mathrm{full}}, \mathrm{abl.})$ is the adapter-free analogue of $\mathcal{NPA}$, computed against the same ablation-based Necessity rankings. The final column repeats the LoRA-based $\mathcal{NPA}$ from Table~\ref{tab:consistency} for comparison. Signs agree on all $13$ checkpoints. Qwen3-32B, Llama-3.1-70B, Mamba-2.8B, and Falcon-Mamba-7B were not retrained under full FT for compute reasons.}
\label{tab:fullft-control}
\small
\resizebox{\linewidth}{!}{%
\begin{tabular}{llrrr}
\toprule
\textbf{Model} & \textbf{Family} & \boldmath$\rho(\mathrm{full},\mathrm{LoRA})$ & \boldmath$\rho(\mathrm{full},\mathrm{abl.})$ & \boldmath$\mathcal{NPA}$ \\
\midrule
Qwen3-0.6B    & Transformer & $0.96$ & $-0.27$ & $-0.25$ \\
Qwen3-1.7B    & Transformer & $0.95$ & $-0.45$ & $-0.47$ \\
Qwen3-4B      & Transformer & $0.88$ & $-0.21$ & $-0.23$ \\
Qwen3-8B      & Transformer & $0.48$ & $-0.25$ & $-0.30$ \\
Qwen3-14B     & Transformer & $0.39$ & $-0.23$ & $-0.06$ \\
Llama-3.2-1B  & Transformer & $0.70$ & $-0.76$ & $-0.74$ \\
Llama-3.2-3B  & Transformer & $0.70$ & $-0.64$ & $-0.48$ \\
Llama-3.1-8B  & Transformer & $0.43$ & $-0.42$ & $-0.62$ \\
\midrule
Mamba-130m    & SSM & $0.46$ & $+0.11$ & $+0.22$ \\
Mamba-370m    & SSM & $0.41$ & $+0.35$ & $+0.22$ \\
Mamba-790m    & SSM & $0.42$ & $+0.41$ & $+0.25$ \\
RWKV6-3B      & SSM & $0.75$ & $+0.42$ & $+0.54$ \\
\midrule
Zamba2-2.7B   & Hybrid & $0.30$ & $+0.05$ & $+0.05$ \\
\bottomrule
\end{tabular}%
}
\end{table}

\paragraph{Update-to-weight normalization.}
A related concern is that $\|\Delta W_l\|_F$ rewards layers that simply have
larger weight matrices. We therefore recompute the Plasticity ranking from the
normalized quantity $\|\Delta W_l\|_F / \|W_l\|_F$, using the corresponding
updated parameter set for each method. The qualitative conclusion is unchanged:
LoRA-based $\mathcal{NPA}$ values are essentially stable under normalization,
the evaluated transformers remain negative, the pure Mamba checkpoints and
RWKV6-3B remain positive, and the already near-zero hybrid and boundary cases
remain small in magnitude. Normalization thus rescales the Plasticity profile
without moving the architecture split.

\subsection{Component-level analysis of a hybrid checkpoint}
\label{app:zamba-components}

Zamba2-2.7B is the one checkpoint in our set that contains both block types, so
it permits a within-model version of the architecture comparison. Its $54$
layers comprise $45$ Mamba blocks and $9$ shared-attention hybrid blocks at
positions $6, 12, \ldots, 51$. We partition the layers by block type and compute
each component's share of three quantities: total LoRA update mass, total
full-FT update mass, and total ablation Necessity mass.

Table~\ref{tab:zamba-components} shows the dissociation reproduced inside a
single model. The attention component is $16.7\%$ of the layers but absorbs
$36.3\%$ of LoRA update mass and $31.6\%$ of full-FT update mass --- roughly
twice its layer share under either adaptation procedure --- while carrying only
$10.4\%$ of ablation Necessity mass, about $0.6\times$ its layer share.
Necessity is instead concentrated in the Mamba blocks. Zamba2's near-zero
whole-model $\mathcal{NPA}$ of $+0.05$ is therefore not evidence that the effect
is absent; it is the average of a transformer-like component and an SSM-like
component with opposing profiles.

Because this rests on a single checkpoint with one attention/SSM ratio, we treat
statements about other ratios as extrapolation beyond the measured setting.

\begin{table}[t]
\centering
\caption{Component-level decomposition of Zamba2-2.7B. Shares are percentages of the model total; $\pm$ values are standard deviations across tasks. The attention component takes roughly twice its layer share of adaptation mass under both LoRA and full FT, but only about $0.6\times$ its layer share of Necessity.}
\label{tab:zamba-components}
\small
\resizebox{\linewidth}{!}{%
\begin{tabular}{lrrrr}
\toprule
\textbf{Component} & \textbf{Layers} & \textbf{LoRA} \boldmath$\Delta W$ & \textbf{Full-FT} \boldmath$\Delta W$ & \textbf{Necessity} \\
\midrule
Attention blocks ($9$)  & $16.7\%$ & $36.3 \pm 0.6$ & $31.6 \pm 0.7$ & $10.4 \pm 2.6$ \\
Mamba blocks ($45$)     & $83.3\%$ & $63.7$         & $68.4$         & $89.6$ \\
\bottomrule
\end{tabular}%
}
\end{table}

\end{revblock}

\subsection{Proxy validation is conditional on the metric}
\label{app:metric-dependence}

The main causal-validation protocol uses a shared likelihood-based proxy so that top-$k$ and bottom-$k$ placement can be compared across heterogeneous tasks. This makes the comparison controlled, but it also means that the resulting directional labels should not be interpreted as universal claims about end-task performance.

Estimator choice reverses which tasks appear supported. LoRA-delta placement supports LongBench in 9/10 transformer checkpoints, whereas ablation supports it in 2/10. On SQuAD, the ordering flips: 3/10 under LoRA delta and 9/10 under ablation. These are the same models, training budget, and validation loop; only the importance estimator changes. Code-generation tasks favor LoRA delta under the proxy, while commonsense tasks show little structured support for either estimator. These patterns show that estimator choice changes the placement conclusion even when the training and validation loop are held fixed.

We then audit the proxy verdicts against task-native metrics, including pass@1, exact match, F1, and LongBench aggregate scores. The audit shows that the proxy metric is neither uniformly optimistic nor uniformly pessimistic. In some cases, the proxy reports a top-$k$ win that does not appear under the native metric; in other cases, the native metric improves even when the proxy does not. Thus, proxy validation captures one operational notion of selective adaptation, but it is not interchangeable with task-native evaluation.
\begin{table}[t]
\centering
\small
\caption{Task-native MCQ accuracy audit for causal-validation rows. Each entry averages top-$k$ and bottom-$k$ accuracy over the available MCQ tasks: MMLU, ARC-Challenge, GPQA, PIQA, and WinoGrande. When multiple runs are available for the same model and method, we average them. $\Delta$ is top-$k$ minus bottom-$k$. Generation-style tasks are excluded because their main diagnostic validation uses perplexity rather than task-native pass@1/F1/exact-match metrics.}
\label{tab:mcq_native_audit_metric}
\resizebox{\linewidth}{!}{%
\begin{tabular}{llccc}
\toprule
\textbf{Model} & \textbf{Method} & \textbf{Top-$k$} & \textbf{Bottom-$k$} & \textbf{$\Delta$} \\
\midrule
Qwen3-8B & LoRA delta & 0.500 & 0.522 & -0.022 \\
Qwen3-8B & Ablation   & 0.524 & 0.593 & -0.069 \\
Qwen3-8B & Gradient   & 0.514 & 0.525 & -0.011 \\
Qwen3-8B & Resnorm    & 0.566 & 0.514 & +0.052 \\
Qwen3-8B & ShapLoRA   & 0.491 & 0.546 & -0.055 \\
Qwen3-8B & TELL-TALE  & 0.498 & 0.530 & -0.031 \\
\midrule
Llama-3.1-8B & LoRA delta & 0.471 & 0.549 & -0.078 \\
Llama-3.1-8B & Ablation   & 0.555 & 0.517 & +0.038 \\
Llama-3.1-8B & Gradient   & 0.559 & 0.513 & +0.047 \\
Llama-3.1-8B & Resnorm    & 0.541 & 0.530 & +0.011 \\
Llama-3.1-8B & ShapLoRA   & 0.572 & 0.530 & +0.042 \\
Llama-3.1-8B & TELL-TALE  & 0.517 & 0.555 & -0.038 \\
\midrule
Qwen3-14B & LoRA delta & 0.544 & 0.538 & +0.006 \\
Qwen3-14B & Random     & 0.543 & 0.522 & +0.021 \\
\midrule
Mamba-790m & Ablation   & 0.617 & 0.672 & -0.055 \\
Mamba-790m & LoRA delta & 0.451 & 0.466 & -0.015 \\
Mamba-790m & Random     & 0.455 & 0.449 & +0.006 \\
Mamba-790m & ShapLoRA   & 0.444 & 0.474 & -0.030 \\
Mamba-790m & TELL-TALE  & 0.452 & 0.472 & -0.020 \\
\bottomrule
\end{tabular}%
}
\end{table}
\section{Full Causal Validation Tables}
\label{app:causal_validation}

Tables~\ref{tab:qwen8b_full}--\ref{tab:llama70b_full} present per-task causal-validation results for representative model--method pairs. Generation rows report perplexity; MCQ rows report accuracy where available. To keep the schema uniform, the metric columns below are labeled as Base, Top$_k$, and Bot$_k$. The final ``Supp.'' column is a directional indicator only. Rows marked with $^\dagger$ are marginal near-ties and should not be treated as strong practical wins.

\begin{table}[t]
\centering
\caption{Qwen3-8B causal validation under LoRA delta ($k=6$ of 36 layers, seed 42). Generation rows report perplexity; MCQ rows report accuracy where available.}
\label{tab:qwen8b_full}
\small
\resizebox{\linewidth}{!}{%
\begin{tabular}{llccccc}
\toprule
Task & Group & Top-6 layers & Base & Top$_k$ & Bot$_k$ & Supp. \\
\midrule
HumanEval & Code       & 23,31--35       & 3.23  & \textbf{1.13} & 1.41 & \checkmark \\
MBPP      & Code       & 28,29,32--35    & 5.88  & \textbf{2.07} & 2.49 & \checkmark \\
LongBench & Long-ctx   & 24,28,30--32,34 & 12.39 & \textbf{5.04} & 7.67 & \checkmark \\
GPQA      & Science    & 28,31--35       & 6.65  & \textbf{3.57} & 3.79 & \checkmark \\
MMLU      & Knowledge  & 30--35          & 10.79 & 9.97          & 6.15 & $\times$ \\
GSM8K     & Math       & 23,31--35       & 3.87  & 2.21          & 1.78 & $\times$ \\
SQuAD     & Retrieval  & 28,31--35       & 13.92 & 8.72          & 7.26 & $\times$ \\
ARC       & Commonsense & 30--35         & 8.95  & 6.27          & 4.10 & $\times$ \\
PIQA      & Commonsense & 30--35         & 12.45 & 9.80          & 4.46 & $\times$ \\
WinoGrande & Commonsense & 30--35        & 0.766 & 0.648         & 0.703 & $\times$ \\
\bottomrule
\end{tabular}%
}
\end{table}

\begin{table}[t]
\centering
\caption{Llama-3.1-8B causal validation under LoRA delta ($k=6$ of 32 layers, seed 42). Generation rows report perplexity; MCQ rows report accuracy where available.}
\label{tab:llama8b_full}
\small
\resizebox{\linewidth}{!}{%
\begin{tabular}{llccccc}
\toprule
Task & Group & Top-6 layers & Base & Top$_k$ & Bot$_k$ & Supp. \\
\midrule
HumanEval & Code      & 25--29,31    & 3.34 & \textbf{1.36} & 1.86 & \checkmark \\
MBPP      & Code      & 24--29       & 5.93 & \textbf{2.01} & 2.33 & \checkmark \\
LongBench & Long-ctx  & 15,26--30    & 9.29 & \textbf{3.68} & 5.47 & \checkmark \\
GPQA      & Science   & 16,17,19,26,29,30 & 5.28 & \textbf{3.09} & 3.40 & \checkmark \\
MMLU      & Knowledge & 15,17,23,25,28,29 & 8.27 & 10.31 & 6.03 & $\times$ \\
GSM8K     & Math      & 25--29,31    & 5.13 & 3.21          & 2.79 & $\times$ \\
SQuAD     & Retrieval & 16,27--31    & 9.14 & 7.69          & 5.64 & $\times$ \\
ARC       & Commonsense & 24--29     & 7.66 & 7.32          & 4.91 & $\times$ \\
PIQA      & Commonsense & 24--29     & 11.56 & 9.25         & 5.28 & $\times$ \\
WinoGrande & Commonsense & 25--29,31 & 0.773 & 0.719        & 0.773 & $\times$ \\
\bottomrule
\end{tabular}%
}
\end{table}

\begin{table}[t]
\centering
\caption{Mamba-790m causal validation under ablation ($k=5$ of 48 layers). Boundary layers dominate top-$k$ selection. Generation rows report perplexity; MCQ rows report accuracy where available.}
\label{tab:mamba790m_full}
\small
\resizebox{\linewidth}{!}{%
\begin{tabular}{llccccc}
\toprule
Task & Group & Top-5 layers & Base & Top$_k$ & Bot$_k$ & Supp. \\
\midrule
HumanEval  & Code      & 0,1,34,46,47 & 4.59  & \textbf{3.95} & 4.18 & $\checkmark$ \\
MBPP       & Code      & 0,1,34,46,47 & 6.21  & \textbf{4.02} & 4.87 & $\checkmark$ \\
GSM8K      & Math      & 0,1,34,46,47 & 8.96  & \textbf{4.97} & 6.02 & $\checkmark$ \\
SQuAD      & Retrieval & 0,1,29,46,47 & 9.14  & \textbf{5.31} & 6.78 & $\checkmark$ \\
LongBench  & Long-ctx  & 0,1,34,46,47 & 11.30 & \textbf{7.42} & 8.95 & $\checkmark$ \\
MMLU       & Knowledge & 0,1,34,46,47 & 8.10  & \textbf{5.60} & 6.20 & $\checkmark$ \\
GPQA       & Science   & 0,1,34,46,47 & 6.20  & \textbf{4.10} & 4.58 & $\checkmark$ \\
ARC        & Commonsense & 0,1,29,34,47 & 7.80 & 5.10 & 5.62 & $\times$ \\
PIQA       & Commonsense & 0,1,29,46,47 & 10.40 & 7.80 & 8.20 & $\times$ \\
WinoGrande & Commonsense & 0,1,29,46,47 & 0.656 & 0.633 & 0.672 & $\times$ \\
\bottomrule
\end{tabular}%
}
\end{table}

\begin{table}[t]
\centering
\caption{Llama-3.2-3B causal validation under LoRA delta ($k=5$ of 28 layers). Generation rows report perplexity; MCQ rows report accuracy where available.}
\label{tab:llama3b_full}
\small
\resizebox{\linewidth}{!}{%
\begin{tabular}{llccccc}
\toprule
Task & Group & Top-5 layers & Base & Top$_k$ & Bot$_k$ & Supp. \\
\midrule
HumanEval & Code      & 19,20,23,26,27 & 3.54  & \textbf{1.75} & 2.16 & $\checkmark$ \\
MBPP      & Code      & 16,20,23,26,27 & 6.25  & \textbf{2.14} & 2.54 & $\checkmark$ \\
LongBench & Long-ctx  & 16,23,25,26,27 & 10.80 & \textbf{6.42} & 7.30 & $\checkmark$ \\
GPQA      & Science   & 15,16,17,20,25 & 5.72  & 3.84          & 3.90 & $\times$ \\
MMLU      & Knowledge & 14,20,25,26,27 & 9.00  & 7.71          & 6.64 & $\times$ \\
SQuAD     & Retrieval & 15,19,22,24,27 & 12.42 & 9.61          & 8.64 & $\times$ \\
ARC       & Commonsense & 20,23,24,25,27 & 8.14 & 5.35         & 4.96 & $\times$ \\
PIQA      & Commonsense & 21,22,25,26,27 & 11.48 & 6.19        & 4.71 & $\times$ \\
WinoGrande & Commonsense & 19,20,23,26,27 & 0.766  & 0.680       & 0.750 & $\times$ \\
\bottomrule
\end{tabular}%
}
\end{table}

\begin{table}[t]
\centering
\caption{Qwen3-32B causal validation under LoRA delta ($k=8$ of 64 layers). Generation rows report perplexity; MCQ rows report accuracy where available. The selected layers are bimodal, with an early cluster and a terminal cluster.}
\label{tab:qwen32b_lora_full}
\scriptsize
\resizebox{\linewidth}{!}{%
\begin{tabular}{llccccc}
\toprule
Task & Group & Top-8 layers (sample) & Base & Top$_k$ & Bot$_k$ & Supp. \\
\midrule
HumanEval & Code      & 4,5,6,59,60,61,62,63   & 1.98  & \textbf{1.36} & 1.49 & $\checkmark$ \\
MBPP      & Code      & 53,54,55,59,60,61,62,63 & 4.37 & \textbf{2.03} & 2.48 & $\checkmark$ \\
LongBench & Long-ctx  & 46,48,49,50,51,52,53,54 & 9.84 & \textbf{6.76} & 11.17 & $\checkmark$ \\
GPQA      & Science   & 2,4,5,6,51,52,61,62    & 5.16  & \textbf{3.54} & 3.88 & $\checkmark$ \\
GSM8K     & Math      & 4,5,6,16,17,60,61,62   & 2.25  & \textbf{1.38} & 1.45 & $\checkmark$ \\
SQuAD     & Retrieval & 3,4,5,6,54,59,60,61    & 8.78  & \textbf{4.61} & 4.88 & $\checkmark$ \\
MMLU      & Knowledge & 2,3,4,5,6,48,62,63     & 7.45  & 5.16          & 5.50 & $\times$ \\
ARC       & Commonsense & 2,4,5,6,59,60,61,63  & 6.58  & 3.67          & 3.99 & $\times$ \\
PIQA      & Commonsense & 4,5,6,59,60,61,62,63 & 7.67  & 4.10          & 4.17 & $\times$ \\
WinoGrande & Commonsense & 51,54,58,59,60,61,62,63 & 0.773 & 0.727      & 0.781 & $\times$ \\
\bottomrule
\end{tabular}%
}
\end{table}

\begin{table}[t]
\centering
\caption{Qwen3-32B causal validation under ablation ($k=8$ of 64 layers). Generation rows report perplexity; MCQ rows report accuracy where available. WinoGrande is directionally supported, but the margin is effectively a tie.}
\label{tab:qwen32b_abl_full}
\scriptsize
\resizebox{\linewidth}{!}{%
\begin{tabular}{llccccc}
\toprule
Task & Group & Top-8 layers (sample) & Base & Top$_k$ & Bot$_k$ & Supp. \\
\midrule
HumanEval & Code      & 0,1,6,45,46,47,48,63    & 1.98 & \textbf{1.16} & 1.19 & $\checkmark$ \\
MBPP      & Code      & 0,1,6,40,53,54,60,61    & 4.37 & \textbf{1.83} & 2.35 & $\checkmark$ \\
LongBench & Long-ctx  & 0,1,4,5,6,53,60,61      & 9.84 & \textbf{6.76} & 7.30 & $\checkmark$ \\
GPQA      & Science   & 0,1,2,6,35,54,60,61     & 5.16 & \textbf{3.54} & 3.60 & $\checkmark$ \\
SQuAD     & Retrieval & 0,1,6,38,48,52,60,61    & 8.78 & \textbf{4.61} & 4.73 & $\checkmark$ \\
MMLU      & Knowledge & 0,1,2,6,35,53,62,63     & 7.45 & \textbf{5.16} & 5.26 & $\checkmark$ \\
ARC       & Commonsense & 0,1,2,6,35,54,60,61   & 6.58 & \textbf{3.67} & 3.78 & $\checkmark$ \\
WinoGrande & Commonsense & 0,1,6,41,42,44,52,62 & 0.773  & 0.719 & 0.711 & $\checkmark^\dagger$ \\
GSM8K     & Math      & 0,1,6,48,50,52,60,61    & 2.25 & 1.33          & 1.30 & $\times$ \\
PIQA      & Commonsense & 0,6,41,44,52,54,60,63 & 7.67 & 4.10          & 4.04 & $\times$ \\
\bottomrule
\end{tabular}%
}
\end{table}

\section{Layer-Pattern Summaries}
\label{app:layer-patterns}

\subsection{Per-model top-layer patterns}

\begin{table}[t]
\centering
\caption{Qwen3-8B: top-5 layers by LoRA-delta attribution per task. Bold entries indicate layers outside the terminal cluster L33--L35.}
\label{tab:qwen8b_lora_detail}
\small
\resizebox{\linewidth}{!}{%
\begin{tabular}{lcccccc}
\toprule
Task & \#1 & \#2 & \#3 & \#4 & \#5 & Mean $\|\Delta\|$ \\
\midrule
MMLU       & L34 & L33 & L30 & L32 & L35 & 3.07 \\
GSM8K      & L34 & L35 & L33 & L32 & L31 & 2.02 \\
ARC        & L34 & L33 & L35 & L31 & L32 & 2.38 \\
PIQA       & L34 & L33 & L32 & L31 & L35 & 2.37 \\
HumanEval  & L35 & L34 & L33 & \textbf{L25} & \textbf{L26} & 2.28 \\
MBPP       & L33 & L34 & L35 & L32 & \textbf{L29} & 2.79 \\
SQuAD      & L34 & L35 & L33 & L32 & L31 & 2.25 \\
WinoGrande & \textbf{L32} & L34 & L33 & L31 & L30 & 3.70 \\
GPQA       & L34 & L33 & L35 & L32 & L31 & 2.64 \\
LongBench  & \textbf{L24} & \textbf{L28} & \textbf{L31} & \textbf{L30} & \textbf{L32} & \textbf{6.14} \\
\bottomrule
\end{tabular}%
}
\end{table}

\subsection{Structural layer-selection patterns by method}
\label{app:layer_patterns}

\begin{table}[t]
\centering
\caption{Structural comparison of top-$k$ layer patterns across importance methods and representative models. ``Scattered + L0'' indicates that L0 always appears in top-$k$ and other layers are task-variable.}
\label{tab:layer_patterns}
\small
\resizebox{\linewidth}{!}{%
\begin{tabular}{llll}
\toprule
\textbf{Model} & \textbf{Method} & \textbf{Typical top-$k$} & \textbf{Pattern} \\
\midrule
Qwen3-8B      & LoRA delta & L28--L35          & Terminal cluster \\
Qwen3-14B     & LoRA delta & L1--L6, L37--L39  & Bimodal (early onset) \\
Qwen3-32B     & LoRA delta & L4--L6, L60--L63  & Bimodal \\
Llama-3.2-3B  & LoRA delta & L20--L27          & Terminal cluster \\
Llama-3.1-8B  & LoRA delta & L24--L31          & Terminal cluster \\
Llama-3.1-70B & LoRA delta & L64--L79          & Terminal cluster (tighter) \\
\midrule
Qwen3-4B      & Ablation & L0,L7,L11--L13,L24       & Scattered + L0 \\
Qwen3-14B     & Ablation & L0,L6,L8,L17,L29,L30     & Scattered + L0 \\
Qwen3-32B     & Ablation & L0,L1,L6,L45--L48,L63    & Scattered + L0 \\
Llama-3.1-8B  & Ablation & L0,L1,L2,L7,L14,L31      & Scattered + L0 \\
Llama-3.1-70B & Ablation & L0,L13--L19,L35,L73      & Scattered + L0 \\
\midrule
Mamba-790m    & Ablation   & L0,L1,L34,L46,L47   & Boundary layers \\
Mamba-370m    & Ablation   & L0,L39,L45,L46,L47  & Boundary layers \\
Mamba-2.8B    & Ablation   & L0,L35,L44,L60,L61  & Boundary layers \\
Mamba-790m    & LoRA delta & L2,L20,L38,L40,L41  & Mid-network \\
Mamba-370m    & LoRA delta & L5,L13,L38,L42,L45  & Mid-network \\
Mamba-2.8B    & LoRA delta & L26,L33,L43,L47,L52 & Mid-network \\
\bottomrule
\end{tabular}%
}
\end{table}

\section{Forgetting}
\label{app:forgetting}
\begin{figure}[t]
\centering
\includegraphics[width=\linewidth]{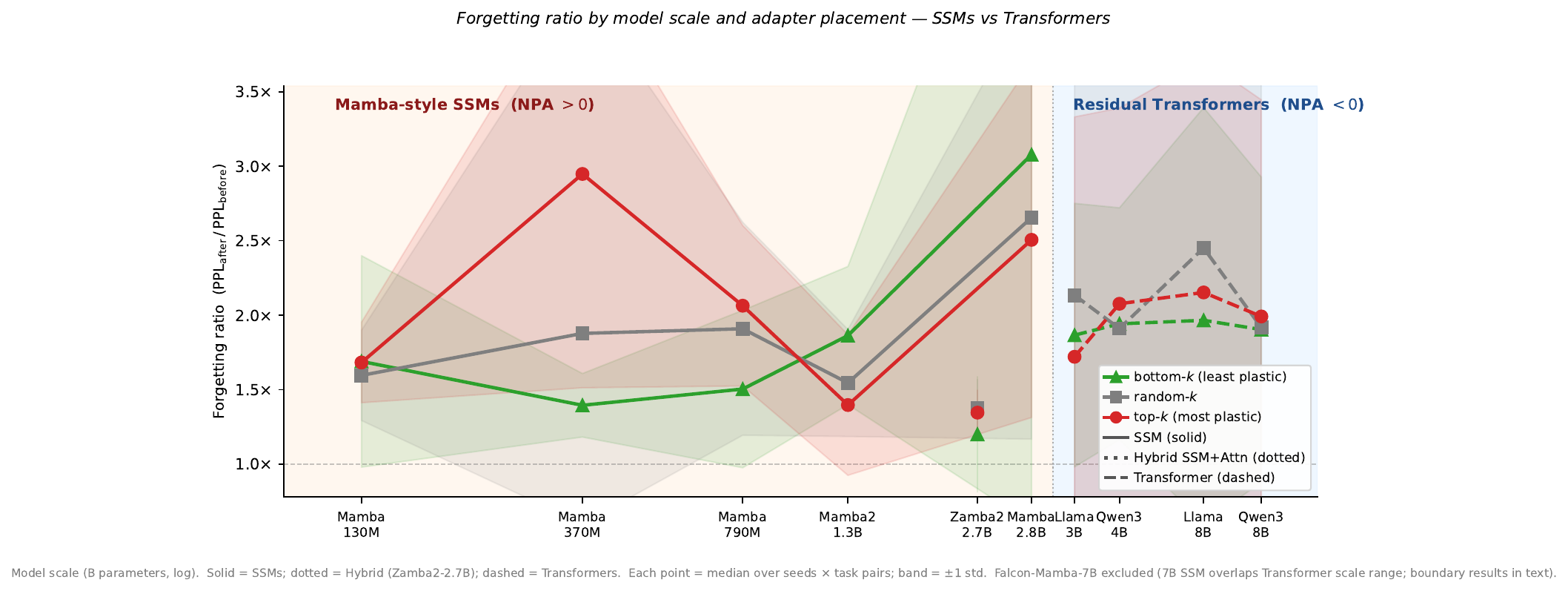}
\caption{\textbf{Forgetting ratio by model scale.}
Each point is the median forgetting ratio across seeds and task pairs; shaded
bands show $\pm1$ std. Solid lines = SSMs (0.13--2.8B); dotted line =
Hybrid (Zamba2-2.7B, $\mathcal{NPA}{=}{+}0.05$); dashed lines =
Transformers (3--9B). Color encodes placement strategy: red = top-$k$ (most
plastic), green = bottom-$k$ (least plastic), gray = random-$k$.
Within SSMs, the three lines are interleaved with no consistent ordering.
Within Transformers, top-$k$ sits above bottom-$k$ at every scale.
Zamba2's three strategies converge to uniformly low forgetting, consistent
with its near-zero NPA and distributed plasticity.
Falcon-Mamba-7B is excluded (its 7B scale overlaps with the Transformer range;
boundary results reported in text).
The boundary between SSM and Transformer scale regions is marked by a dotted
vertical line.}
\label{fig:cl_forgetting_scale}
\end{figure}

\begin{figure}[t]
\centering
\includegraphics[width=\linewidth]{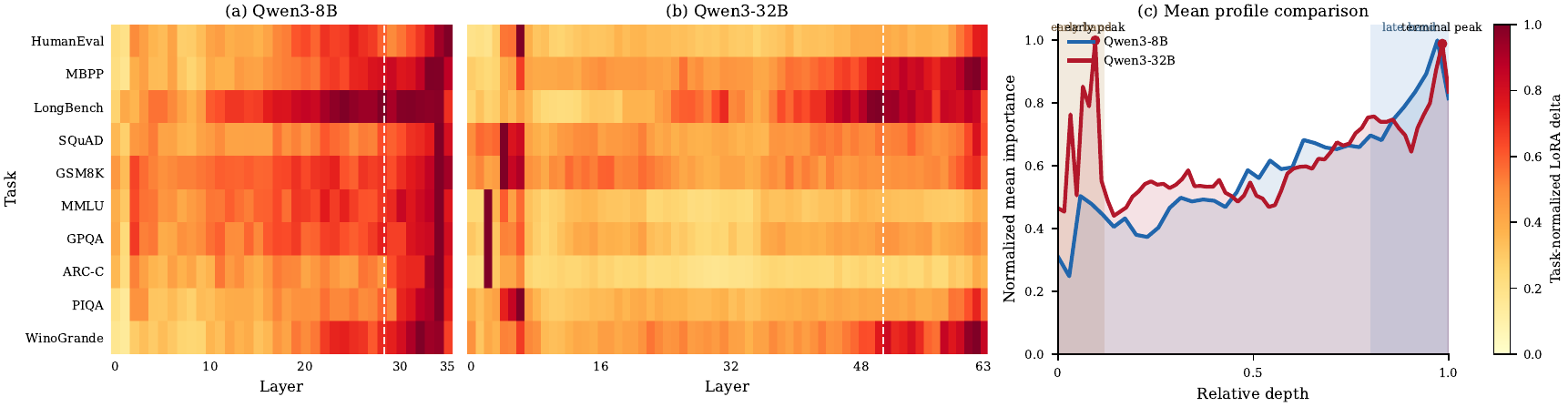}
\caption{Task-normalized LoRA-delta patterns for Qwen3-8B and Qwen3-32B. Panels (a--b) show per-task heatmaps; panel (c) overlays the mean profiles by relative depth. Qwen3-32B develops an additional early-layer peak alongside the terminal cluster, yielding a bimodal profile absent in smaller Qwen3 variants. The bimodality is the proximate clue for the scale-induced collapse of $\mathcal{C}(M)$ on the two largest transformers (Section \ref{sec:scale-failure}).}
\label{fig:bimodality}
\end{figure}

\begin{table}[t]
\centering
\caption{Llama-3.1-70B causal validation under LoRA delta ($k{=}10$ of 80 layers). Generation rows report perplexity; MCQ rows report accuracy where available. Most tasks select terminal-cluster layers; GSM8K additionally recruits early-middle layers. SQuAD is only marginally supported by the binary rule.}
\label{tab:llama70b_full}
\scriptsize
\resizebox{\linewidth}{!}{%
\begin{tabular}{llccccc}
\toprule
Task & Group & Top-10 layers (sample) & Base & Top$_k$ & Bot$_k$ & Supp. \\
\midrule
HumanEval & Code      & 40,41,42,45,73,74,75,76,78,79 & 2.78 & \textbf{1.80} & 2.30 & \checkmark \\
MBPP      & Code      & 41,43,44,45,48,49,72,73,74,75 & 4.94 & \textbf{1.95} & 2.64 & \checkmark \\
LongBench & Long-ctx  & 63,64,67,68,70,71,72,74,75,76 & 7.71 & \textbf{4.66} & 6.59 & \checkmark \\
GPQA      & Science   & 6,8,40,71,73,74,75,77,78,79   & 4.27 & \textbf{3.41} & 3.47 & \checkmark \\
GSM8K     & Math      & 12,13,14,15,16,33,71,75,78,79 & 4.11 & \textbf{2.46} & 2.55 & \checkmark \\
SQuAD     & Retrieval & 27,33,35,36,37,38,74,77,78,79 & 2.91 & \textbf{1.69} & 1.70 & \checkmark$^\dagger$ \\
MMLU      & Knowledge & 4,9,11,21,36,73,74,77,78,79   & 5.69 & 4.72          & 4.58 & $\times$ \\
ARC       & Commonsense & 67,71,72,73,74,75,76,77,78,79 & 5.96 & 3.97         & 3.91 & $\times$ \\
PIQA      & Commonsense & 35,71,72,73,74,75,76,77,78,79 & 9.63 & 4.26         & 3.79 & $\times$ \\
WinoGrande & Commonsense & 68,69,70,71,72,73,74,75,76,77 & 0.812 & 0.781          & 0.758  & $\times$ \\
\bottomrule
\end{tabular}
}
\end{table}

\end{document}